\documentclass[10pt]{article}
\usepackage[hmargin=2cm,vmargin=2cm]{geometry}
\usepackage{palatino,epsfig,latexsym,xfrac} 
\usepackage[linesnumbered,boxed]{algorithm2e}
\usepackage{amssymb,amsmath,amsthm,enumerate,etoolbox,paralist,url}
\usepackage{color}
\usepackage{lineno}
\DontPrintSemicolon

\newtheorem{prp}{Proposition}

\newtheorem{thm}{Theorem}

\def\E{\mathop{{\mathbb{E}}}}
\def\R{{\mathbb{R}}}
\def\1{{\mathbf{1}}}

\def\N{{\mathbb{N}}}

\def\_{\,\,\,\,\,}
\def\prob{{\mathbf{Pr}}}

\newcommand{\eps}{\epsilon}

\usepackage{authblk}
\title{On the Covariance-Hessian Relation in Evolution Strategies}
\author{Ofer M. Shir}
        \affil{Computer Science Department, Tel-Hai College, and Migal Institute, Upper Galilee, Israel }
\author{Amir Yehudayoff}
        \affil{Department of Mathematics, Technion - Israel Institute of Technology, Haifa, Israel}
    
\date{}
\begin{document}
\maketitle

\begin{abstract}
We consider Evolution Strategies (ESs) operating only with isotropic Gaussian mutations on positive quadratic objective functions, and investigate the covariance matrix when constructed out of selected individuals by truncation. 
We prove that the covariance matrix over $(1,\lambda)$-selected decision vectors becomes proportional to the inverse of the landscape Hessian as the population-size $\lambda$ increases. 
This confirms a classical hypothesis that statistical learning of the landscape is an inherent characteristic of standard ESs, and that this distinguishing capability stems only from the usage of isotropic Gaussian mutations and rank-based selection.
Even though the model under consideration does not precisely conform with practically encountered scenarios, it plays a role of a theoretical foundation for learning capabilities within ESs.
We also provide broad numerical validation for the proven results, and present empirical evidence for its generalization to $(\mu,\lambda)$-selection.
\end{abstract}


\section{Introduction}
Evolution Strategies (ESs) are popular randomized search heuristics that excel in global optimization of continuous landscapes; for recent reviews see \cite{RudolphHandbookNACO,hansen2015,ESchapter2018}.
Their mechanism is primarily characterized by the mutation operator, whose step is drawn from a multivariate normal distribution using an evolving covariance matrix.
Since their early development in the 1960's \cite{Schwefel,Baeck-book,Beyer} and up until recently, it has been hypothesized that this adapted matrix approximates the inverse Hessian of the search landscape. 
This hypothesis was intuitively supported by the rationale that the optimal covariance distribution can offer mutation steps whose \emph{equidensity probability contours} match the \emph{level sets} of the landscape, and so they maximize the progress rate~\cite{Rudolph92}. 
Altogether, the motivation to hold a covariance matrix reflective of the eigen-directions of the landscape Hessian is well-justified from the algorithmic perspective \cite{RudolphHandbookNACO}. It additionally carries \textit{practical implications} for the application perspective | knowledge of second-order Hessian information is desirable in practice as (i) a measure of system robustness to noise in the decision variables, (ii) a means for dimensionality reduction, and (iii) assisting in landscape characterization \cite{Shir-FOCAL}.

Historically speaking, Rudolph's study on correlated mutations \cite{Rudolph92} constituted one of the motivations for ESs' scholars to design strategies that accumulate search information of selected individuals. Such information is to be iteratively constructed by means of statistically-learned algebraic structures -- memory vectors, sliding-window matrices, or proper covariance matrices -- which in turn hold the capacity to \textit{derandomize} the search process. 
Derandomized ESs \cite{HansenDR1,HansenDR2,HansenDR3,Hansen01completely} have become a successful family of search heuristics, whose operation heavily relies on statistical learning of past search information; for a comprehensive overview see \cite{Baeck2013contemporary}.  
Lastly but importantly, from an empirical perspective, it must be noted that the volume of supportive evidence for the hypothesis that ESs can learn such a covariance matrix is overwhelming, including application-driven studies (see, e.g., \cite{Shir-FOCAL}, where the Hessian's eigenvectors were successfully recovered following statistical learning of a covariance matrix).

Recent developments in randomized search heuristics for continuous optimization succeeded in making a link between certain modern ESs (e.g., the renowned CMA-ES \cite{hansen2015}) to Information Geometry \cite{Amari}; the linkage is due to the strategy adaptation by means of the mutation distribution update. 
This line of research originated in the release of the so-called Natural ESs \cite{Wierstra2014NES}, and the consequent compilation of the Information Geometry Optimization (IGO) philosophy \cite{IGO} in parallel to the formulation of the Natural Gradient (NGD) algorithm \cite{Akimoto2012_NGD}.
In short, the building block of this class of algorithms is the \textit{natural gradient method}~\cite{Amari}, which features steepest ascent moves on the Riemannian manifold of the strategy distribution function using Fisher information.
Importantly, in the context of landscape learning, modern ESs were proven to achieve such learning under certain adaptation mechanisms.
Akimoto proved that the NGD algorithm adapts its covariance matrix so it becomes proportional to the inverse landscape Hessian of any monotonic convex-quadratic composite function in the limit of a large population-size \cite{Akimoto2012_NGD}.
In a broader perspective, Beyer showed that upon implementing the IGO philosophy in the same population-size limit, the self-adapted covariance matrix necessarily becomes proportional to the inverse Hessian \cite{Beyer2014_ConvergenceIGO}.

For the generic ESs' context, a preliminary study \cite{Shir-Theory-foga17} assumed a quadratic model in the vicinity of the optimum, and proved that learning the landscape was an inherent property of classical ESs. Notably, this kind of learning does not require Derandomization (for adaptation) nor IGO (as a proof tool). In short, it showed that the statistically-constructed covariance matrix over selected decision vectors had the same eigenvectors as the Hessian matrix, and that when the population-size was increased, the covariance became proportional to the inverse of the Hessian. Its main results are furthermore outlined in Section \ref{sec:previousresults}.

The current study generalizes this model beyond the near-optimum-assumption and obtains results that are valid for learning anywhere on positive quadratic objective functions. 
Here, we prove that the statistically-constructed covariance matrix over \textit{single-winning} decision vectors converges to the inverse of the landscape Hessian, up to a scalar factor, when the population-size tends to infinity (around any point for any quadratic objective function). 
This shows that for large populations, ESs indeed unveil effective information on the landscape, requiring only usage of isotropic Gaussian mutations and rank-based selection.

The main contribution of this work lies within its theoretical results, being a mathematical rigor to a long-standing hypothesis in ESs' foundations. 
However, their direct applicability is limited, mainly due to the facts that (i) the model does not necessarily reflect effective search, and (ii) the required population-sizes seem to be impractical for pragmatic scenarios.

The remainder of this paper is organized as follows.
The framework is formally stated in Section \ref{sec:problem}, where the assumed model is described in detail, and the covariance matrices for $(1,\lambda)$- and $(\mu,\lambda)$-truncation selection are explicitly derived. We also outline the previous results concerning the near-optimum-assumption.
We address in Section \ref{sec:CinvH} the relation between the covariance matrix and the landscape Hessian, subject to a large population-size. A simulation study encompassing various landscape scenarios for $(1,\lambda)$-selection is presented in Section~\ref{sec:simulation}, constituting a numerical corroboration for all the theoretical outcomes in this work. 
Finally, the results are summarized in Section~\ref{sec:discussion}, and a future direction is discussed, constituting a generalization of the current results into $(\mu,\lambda)$-selection. We formally hypothesize such a generalizing claim and present empirical evidence to support it.

\section{Statistical Learning of the Landscape}\label{sec:problem}
We target the following \textit{research question}: 
\begin{quote}
What is the relation between the statistically-constructed covariance matrix over ESs' $(1,\lambda)$-\textit{winners} to the landscape Hessian under the quadratic approximation when search-points are taken anywhere?
\end{quote}
We focus on the \textit{a posteriori} statistical construction of the covariance matrix of the decision variables.
Next, we formulate the problem, assume a model and present our notation.
\subsection{The Model}
Let ${J}: \mathbb{R}^n \to \mathbb{R}$ denote the objective function subject to minimization, and let it be minimized at the location $\vec{x}^{*}$.
We assume that $J$ is quadratic; 
we model the $n$-dimensional basin of attraction about $\vec{x}^{*}$ by means of a quadratic function:\footnote{This representation of $J$ includes the gradient (linear) term.}
\begin{equation}
\label{eq:Jfunc}
\displaystyle  
J\left(\vec{x}\right)
= J_{{\cal H},\vec{x}^*} \left(\vec{x}\right)
= (\vec{x}-\vec{x}^{*})^T \cdot \mathcal{H} \cdot (\vec{x}-\vec{x}^{*}),
\end{equation}
with $\mathcal{H}$ being the landscape Hessian about the optimum, which is assumed to be full-rank.

The classical non-elitist single-parent ES operates in the following manner:
$\lambda$ search-points $\vec{x}_1,\ldots,\vec{x}_\lambda$ are generated in each iteration, based upon Gaussian sampling with respect to the given search-point $\vec{x}_{0}$.
We are especially concerned with the canonical ES variation operator, which adds a normally distributed \textit{mutation}
$\vec{z} \sim \mathcal{N} (\vec{0},\mathbf{I})$.
That is, $\vec{x}_1,\ldots,\vec{x}_\lambda$ are independent and each is $\mathcal{N} (\vec{x}_0,\mathbf{I})$.
Upon evaluating those $\lambda$ search-points with respect to $J$, the best (minimal) individual is selected and recorded as
\begin{equation}
\vec{y} = \arg\min \left\{ J(\vec{x}_1),~ J(\vec{x}_2),~ \ldots ,~ J(\vec{x}_\lambda) \right\}.
\end{equation}
Finally, let $\omega$ denote the \textit{winning} objective function value,
\begin{equation}
\omega = J(\vec{y}) = \min \left\{ J(\vec{x}_1),~ J(\vec{x}_2),~ \ldots ,~ J(\vec{x}_\lambda) \right\}.
\end{equation}
We will also consider the case of $(\mu,\lambda)$-selection, where the truncated subset of $\mu$ winners is selected.

We mention the difference between the optimization phase, which aims to arrive at the optimum and is not discussed here, to the statistical learning of the basin, which lies in the focus of this study. 
In other words, we consider a \textbf{static model where sampling takes place around a fixed point and the selected points are statistically accumulated}.
The sampling procedure is summarized as Algorithm 1, wherein the routine \texttt{statCovariance} refers to \textit{statistically} constructing a covariance matrix from raw observations. \\
Our notation is summarized in Table~\ref{tab:nomenclature}.
\IncMargin{1em}
\RestyleAlgo{boxed} 
\begin{algorithm}[ht]
\SetKwInOut{Output}{output}
\caption{Statistical sampling by $(1,\lambda)$-selection}
$t \leftarrow 0$\;
$\mathcal{S} \leftarrow$ $\emptyset$\;
\Repeat{$t \geq N_{\texttt{iter}}$} {
 \For{$k\leftarrow 1$ \KwTo $\lambda$}{
 $\vec{x}^{(t+1)}_{k} \leftarrow \vec{x}_{0} + \vec{z}_k,~~~\vec{z}_k \sim \mathcal{N} ( \vec{0},\mathbf{I} )$\;
 $J^{(t+1)}_{k}\leftarrow$ \texttt{evaluate} $\left(\vec{x}^{(t+1)}_{k}\right)$\;
 } 
 $m_{t+1} \leftarrow\arg\min \left(\left\{J^{(t+1)}_{\imath}\right\}_{\imath=1}^{\lambda} \right)$\;
 $\mathcal{S} \leftarrow \mathcal{S} \cup \left\{ \vec{x}^{(t+1)}_{m_{t+1}} \right\} $\;
 $t \leftarrow t+1$\;
 }
\Output{$\mathcal{C}^{\texttt{stat}}=$\texttt{statCovariance}$\left(\mathcal{S}\right)$}
\label{algo:ES_sampling}
\end{algorithm}
\DecMargin{1em}

\begin{table*}[ht]
\centering 
\begin{tabular}{l l l}
\hline
Term & Description & Notation\\
\hline
landscape Hessian & positive definite matrix defining the landscape structure & $\mathcal{H}$ \\
Hessian eigenvalue & by eigendecomposition: $\mathcal{H}=\mathcal{U}\mathcal{D}\mathcal{U}^{-1},~\mathcal{D}=\textrm{diag}\left[\Delta_1,\ldots,\Delta_n \right]$ & $\Delta_i$\\ 
objective function &  subject to minimization, assumed to be minimized at $\vec{x}^{*}$ (Eq.~\ref{eq:Jfunc}) & %
$J\left(\vec{x}\right)$\\
quadratic function & a function $\hat{J}: \mathbb{R}^n \to \mathbb{R}$ of the form $\hat{J}\left(\vec{r}\right)=\vec{r}^T \mathbf{B} \vec{r}+\vec{c}^T\vec{r}+\vec{d}$ & $\hat{J}\left(\vec{r}\right)$\\
random vector & a normal Gaussian mutation & $\vec{z}$ \\
random vector's function value & representing the objective function's value of $\vec{z}$ & $\psi=J(\vec{z})$ \\
population-size & number of generated search-points per iteration & $\lambda$ \\
offspring & $\lambda$ independent copies of $\vec{z}$ & $\vec{x}_1,\ldots,\vec{x}_\lambda$ \\
parental population-size & number of selected search-points per iteration & $\mu$ \\
winner & the recorded best (minimal) individual by selection & $\vec{y}$ \\
winning value & the winning objective function value & $\omega = J(\vec{y})$ \\
$\ell^{th}$ winner & the recorded $\ell^{th}$-best individual by selection & $\vec{y}_{\ell:\lambda}$ \\
$\ell^{th}$ winning value & the $\ell^{th}$-best objective function value & $\omega_{\ell:\lambda}$ \\
expectation vector & expectation vector over winning decision vectors (Eq.~\ref{eq:Exp0}) & $\mathcal{E}$\\
covariance matrix & covariance matrix over winning decision vectors (Eq.~\ref{eq:Cov0}) & $\mathcal{C}$\\
empirical covariance matrix & statistically-constructed matrix (Algorithm 1) & $\mathcal{C}^{\texttt{stat}}$\\
\hline
\end{tabular}
\caption{Nomenclature.\label{tab:nomenclature}}
\end{table*}

\subsection{Previous Results: Optimum's Vicinity}\label{sec:previousresults}
Here are the main results obtained for the special case of near-optimum sampling (i.e.,~$\vec{x}_0 = \vec{x}^*$ ) \cite{Shir-Theory-foga17}:
\begin{compactenum}[(i)]
\item $\mathcal{C}$ and $\mathcal{H}$ commute for any $\lambda$.
This learning capability stems only from two components: (1) isotropic Gaussian mutations, and (2) rank-based selection. 
\item There is $\alpha \left(\lambda,\mathcal{H} \right) >0$ so that $\lim_{\lambda \to \infty}  \alpha \mathcal{C}  \mathcal{H} = \mathbf{I}$.
For this we need to guarantee that $\mathcal C$ is pointwise $\eps$-close to $\alpha {\mathcal H}^{-1}$.
An upper bound on the $\lambda$ needed for this part depends on $\eps$ and on the spectrum of $\mathcal H$.

\item Guaranteeing that $\mathcal{C}^{\texttt{stat}}$ is pointwise $\eps$-close to $\mathcal C$ with confidence $1-\delta$.
The number of samples required for this part is polynomial in $\lambda,1/\eps,\ln(n)$ and $\ln(1/\delta)$.

\item In order to calculate $\mathcal{C}_{ij}$ when $\lambda$ tends to infinity, it is possible to approximate the appropriate density function by considering the Generalized Extreme Value Distribution ($\textrm{GEVD}$) \cite{Castillo2004}, which belongs to the only non-degenerate family of distributions satisfying the limit $\lambda \rightarrow \infty$.
\end{compactenum}

\subsection{The Covariance Matrix for $\left(1,\lambda\right)$-Selection}
For simplicity, we address $\left(1,\lambda\right)$-selection, and accordingly, we analytically formulate the covariance matrix when constructed by consecutive single winners.
Importantly, the following results hold for sampling about any search-point $\vec{x}_0$. 
However, since the Hessian and the covariance matrices in the current modeling are invariant with respect to translations, the nature of $\vec{x}_0$ has no effect whatsoever. 
Thus, for convenience and without loss of generality, we assume that the search-point resides in the origin, $\vec{x}_0:=\vec{0}$.
Yet, $\vec{x}_0$ and $\vec{x}^*$ can be different.
The expectation vector of the winner is defined by its $i^{th}$ element:
\begin{equation}
\label{eq:Exp0}
\displaystyle \mathcal{E}_{i} = \int x_i \texttt{PDF}_{\vec{y}}\left( \vec{x}\right) \textrm{d}\vec{x}~,
\end{equation}
where $\texttt{PDF}_{\vec{y}}\left( \vec{x}\right)$ is an $n$-dimensional density function characterizing the \textit{winning} decision variables in this mutation and selection processes. 
The covariance elements are defined as
\begin{equation}
\label{eq:Cov0}
\boxed{
\displaystyle \mathcal{C}_{ij} = \int (x_i-\mathcal{E}_i)(x_j-\mathcal{E}_j) \texttt{PDF}_{\vec{y}}\left( \vec{x}\right) \textrm{d}\vec{x} }~ .
\end{equation}

The density function of a single winning vector $\vec{y}$ is related to the density of the 
winning value $\omega$ via the following relation \cite{Shir-Theory-foga17}:
\begin{equation}\label{eq:x_pdf_2}
\displaystyle \texttt{PDF}_{\vec{y}}\left(\vec{x}\right) =   \texttt{PDF}_{\omega}\left( J\left(\vec{x} \right) \right) \cdot \frac{\texttt{PDF}_{\vec{z}}\left(\vec{x}\right)}{\texttt{PDF}_{\psi}\left( J\left(\vec{x} \right)\right)} ~,
\end{equation}
with $\texttt{PDF}_{\vec{z}}$ denoting the density function for generating an individual by \textit{mutation} (i.e.,\ Gaussian), and $\texttt{PDF}_{\psi}$ denoting the density function of the objective function values for an individual mutation.
For completeness, we explain~\eqref{eq:x_pdf_2} in~\ref{sec:DE}.


We note the distribution function of the winning value,
\begin{equation}\label{eq:y_cdf}
\displaystyle \texttt{CDF}_{\omega}\left( v \right) =  \prob\left\{ \omega \leq v \right\} = 1 - \left( 1-\texttt{CDF}_{\psi}\left( v \right)\right)^{\lambda},
\end{equation}
of which the density function is differentiated:
\begin{equation}\label{eq:y_pdf}
\displaystyle \texttt{PDF}_{\omega}\left( v \right) =   \lambda \cdot \left( 1-\texttt{CDF}_{\psi}\left(v \right)\right)^{\lambda-1}\cdot \texttt{PDF}_{\psi}\left(v\right).
\end{equation}
Overall, we get the following pleasant representation:
\begin{equation}\label{eq:CDFy}
\boxed{
\displaystyle \texttt{PDF}_{\vec{y}}\left(\vec{x}\right)
 =   \lambda \cdot \left( 1-\texttt{CDF}_{\psi}\left(J\left(\vec{x} \right) \right)\right)^{\lambda-1}
 \cdot \texttt{PDF}_{\vec{z}}\left(\vec{x}\right) } ~.
\end{equation}

%
\subsection{$\left(\mu,\lambda\right)$-Truncation Selection}
In a more general case, $\mu$ winners are selected out of the population of size $\lambda$ in each iteration.
We denote by
$J_{1:\lambda}\leq J_{2:\lambda}\leq \ldots \leq J_{\lambda:\lambda}$ the order statistics obtained by sorting the objective function values, and furthermore denote by $\omega_{1:\lambda},\ldots,\omega_{\mu:\lambda}$ the first $\mu$ values from this list, and by $\vec{y}_{1:\lambda},\ldots,\vec{y}_{\mu:\lambda}$ their corresponding vectors.

Here, there are $\mu$ vectors, rather than a single vector, so the expectation vector and covariance matrix can be defined in several ways.
We choose to focus on the average of these $\mu$ vectors.

The expectation vector reads:
\begin{equation}
\label{eq:Exp_mu}
\displaystyle \mathcal{E}_{i} = 
\sum_{k=1}^{\mu} \frac{1}{\mu}
\int x_{k,i} \texttt{PDF}_{\vec{y}_{k:\lambda}}\left(\vec{x}_k\right) \textrm{d}\vec{x}_k ~ ,
\end{equation}
for the appropriate density of $\vec{y}_{k:\lambda}$.
The covariance element reads:
\begin{small}
\begin{equation}
\label{eq:Cov_mu}
\displaystyle \mathcal{C}_{ij} = 
\frac{1}{\mu^2}
\sum_{k,\ell=1}^{\mu} \int \left( x_{k,i} - \mathcal{E}_{i} \right) \left( x_{\ell,j} - \mathcal{E}_{j} \right) \texttt{PDF}_{\vec{y}_{k:\lambda},\vec{y}_{\ell:\lambda}}\left(\vec{x}_k,\vec{x}_{\ell}\right) \textrm{d}\vec{x}_k \textrm{d}\vec{x}_{\ell} ,
\end{equation}
\end{small}
for the appropriate joint density.


\section{The Inverse Relation}
\label{sec:CinvH}

The following propositions show that for a large population-size $\lambda$, the covariance matrix in \eqref{eq:Cov0} is close to being proportional to the inverse of the landscape Hessian. 

\textbf{We begin by considering the diagonal Hessian case:}

\begin{prp} \label{prop:inverseRelation}
For every invertible diagonal Hessian matrix $\mathcal{H}=\textrm{diag}\left[\Delta_1,\ldots,\Delta_n \right]$ and $\lambda \in \N$, there exists a constant $\alpha = \alpha({\cal H},\lambda) > 0$ such that
$$\lim_{\lambda \to \infty}  \alpha \mathcal{C}  \mathcal{H} = \mathbf{I}.$$
\end{prp}

\begin{proof}
In the following, $\eps_1,\eps_2,\ldots$ tend to zero as $\lambda$ tends to infinity
and $c_0,c_1,c_2,\ldots$ are large positive constants (that may depend on $\mathcal{H}$ and $\vec{x}^*$).
In the proof, we explain why each $\eps_i$ tends to zero (either explicitly or by comparing it to $\eps_j$ for $j<i$).


We start by studying the expectation vector.
We already know from Eqs.\ \ref{eq:Exp0} and \ref{eq:CDFy} that
\begin{equation}\label{eq:proofExp_i}
\displaystyle \mathcal{E}_{i} = \int x_i 
\lambda (1-\texttt{CDF}_\psi(J(\vec{x})))^{\lambda-1} f(\|\vec{x}\|)
 \textrm{d}\vec{x} ,
\end{equation}
where $f(v) = \texttt{PDF}_{\vec{z}}((v,0,0,\ldots,0))$.
By changing variables from $\vec{x}$ to $\vec{r}$, defined as
$r_i = \sqrt{\Delta_i} \cdot ( x_i - x_i^*)$ for all $i$, 
one gets
\begin{equation}\label{eq:proofExp_i-X_*}
 \mathcal{E}_{i} - x^*_i =
\frac{c_{\cal H}}{\sqrt{\Delta_i}}  \int r_i G(\vec{r}) \textrm{d}\vec{r} ,
\end{equation}
where $c_{\cal H} >0$ is a constant that depends on $\mathcal{H}$ and
$$G(\vec{r}) = \lambda (1-\texttt{CDF}_\psi(\|\vec{r}\|^2))^{\lambda-1} \exp \left(-\hat{J}(\vec{r})\right),$$
{where $\hat{J}$ denotes some quadratic function of $\vec{r}$}, having the general form $\hat{J}\left(\vec{r}\right)=\vec{r}^T \mathbf{B} \vec{r}+\vec{c}^T\vec{r}+\vec{d}$.

\medskip

We partition the integration to two parts: firstly on 
$$A = \{ \vec{r}: \texttt{CDF}_\psi(\|\vec{r}\|^{2}) > 1/\sqrt{\lambda}\}$$
and secondly on its complement $\bar A$.
 
The integral on $A$ is at most $c_1 \lambda e^{- \sqrt{\lambda}}$: 
\begin{align}
I_A& := \frac{c_{\cal H}}{\sqrt{\Delta_i}} \int_A r_i G(\vec{r}) \textrm{d}\vec{r} \notag \\
& \leq \lambda (1-1/\sqrt{\lambda})^{\lambda-1} \frac{c_{\cal H}}{\sqrt{\Delta_i}} \int  |r_i| \exp \left(-\hat{J}(\vec{r})\right) \textrm{d}\vec{r} \notag \\
&  \leq c_1 \lambda (1-1/\sqrt{\lambda})^{\lambda-1} \leq c_1 \lambda e^{-\sqrt{\lambda}}
\label{eqn:IofA}.
\end{align}

We now move to the integral over $\bar A$.
The only way to make $\texttt{CDF}_\psi(\|\vec{r}\|^2)$ small
is to have $\|\vec{r}\|$ sufficiently small; that is, we claim that
\begin{align}
\bar A \subset \{ \vec{r} : \|\vec{r}\| < \eps_1 \}.
\label{eqn:barAin}
\end{align}
Indeed,
recall that $\psi$ is the objective function value for an individual mutation.
The function $\texttt{CDF}_\psi$ is $0$ at $0$,
and is strictly increasing on $[0,\infty)$.
So for every $\delta >0$,
there is $\lambda_0 > 0$ so that 
$\texttt{CDF}_\psi(\delta) > 1/\sqrt{\lambda}$
for all $\lambda > \lambda_0$.

We now claim that there is a constant $\phi >0$ so that
for all $\vec{r} \in \bar A$,
\begin{align}
\left|\exp \left(- \hat{J}(\vec{r}) \right) - \phi \right| \leq \eps_2
\leq \eps_3 \exp \left(- \hat{J}(\vec{r}) \right).
\label{eqn:phiDef}
\end{align}
Indeed, all points in $\bar A$ are at most $2 \eps_1$-far apart,
by \eqref{eqn:barAin}.
The function $\exp (- \hat{J}(\cdot) )$ is continuous.
So, when $\eps_1 \to 0$,
its values on $\bar A$ are $\eps_2$-close to some 
fixed value $\phi$,
where $\eps_2$ tends to $0$ as $\eps_1$ tends to $0$.
There is also
some constant $c_0  >0$
so that $\sup \{ |\hat{J}(\vec{r})|
: \vec{r} \in \bar A\} < c_0$
as long as $\eps_1 \leq 1$.
So we can take $\eps_3 = c_0 \eps_2$.

In addition,
\begin{align}
 \int_{\bar A} 
 \phi \cdot r_i \lambda (1-\texttt{CDF}_\psi(\|\vec{r}\|^2))^{\lambda-1} ~ \textrm{d}\vec{r} = 0,
 \label{eqn:barAphi0}
\end{align}
because this is an integral of an odd function over an even domain.
Therefore,
\begin{small}
\begin{align}
I_{\bar A}& := 
\frac{c_{\cal H}}{\sqrt{\Delta_i}} \int_{\bar A}  r_i G(\vec{r}) \textrm{d}\vec{r} \notag \\
& = \frac{c_{\cal H}}{\sqrt{\Delta_i}} \int_{\bar A}  r_i 
\lambda (1-\texttt{CDF}_\psi(\|\vec{r}\|^2))^{\lambda-1} 
\left( \exp \left(- \hat{J}(\vec{r}) \right) - \phi \right) \textrm{d}\vec{r} \notag \\
& \leq 
\eps_3
\frac{c_{\cal H}}{\sqrt{\Delta_i}} \int |r_i| G(\vec{r}) \textrm{d}\vec{r}  \notag \\
& \leq 
\eps_4 \sqrt{ 
 \int r_i^2 G(\vec{r}) \textrm{d}\vec{r}  } \label{eqn:I_0boudn} ,
\end{align}
\end{small}where the transition into the last inequality 
uses convexity, and $\eps_4$ is some constant times $\eps_3$.

Next, we target the covariance diagonal term. Fix $i$ for now.
We claim that 
\begin{equation} \label{eq:ubCii}
\displaystyle \mathcal{C}_{ii} \geq \frac{1}{c_2 \lambda^2} .
\end{equation}
The reason being that for every $\eps_5 > 0$,
a normally distributed $z_i$ takes values in any interval of length $L$,
with a probability at most $2 L$.
So, by the union bound over the $\lambda$ choices of mutations,
$\displaystyle \prob\left[ |y_i - {\cal E}_i|< \sfrac{1}{4 \lambda} \right] \leq \sfrac{1}{2}$.
Hence, ${\cal C}_{ii} = \E \left[(y_i-{\cal E}_i)^2\right] \geq \sfrac{1}{32 \lambda^2}$.

Using the same change of variables as above:
\begin{align*}
 \mathcal{C}_{ii}
& = \int (x_i - \mathcal{E}_i)^2
\lambda (1-\texttt{CDF}_\psi(J(\vec{x})))^{\lambda-1} f(\|\vec{x}\|)
 \textrm{d}\vec{x} \\
& = c_{\cal H} \int \left(\frac{r_i}{\sqrt{\Delta_i}} + x_i^* - \mathcal{E}_i \right)^2 G(\vec{r}) \textrm{d}\vec{r} \\
& = \left(  c_{\cal H} \int \frac{r^2_i}{\Delta_i}  G(\vec{r}) \textrm{d}\vec{r}   \right)
- 2 (x_i^*-\mathcal{E}_i)^2 + (x_i^*-\mathcal{E}_i)^2   \tag{using \eqref{eq:proofExp_i-X_*}} \\
& = S - (x_i^*-\mathcal{E}_i)^2 , \notag
\end{align*}
where
$$S =c_{\cal H} \int \frac{r^2_i}{\Delta_i}  G(\vec{r}) \textrm{d}\vec{r}.$$
By \eqref{eqn:IofA}, \eqref{eq:ubCii} and \eqref{eqn:I_0boudn},

\begin{align*}
(x_i^*-\mathcal{E}_i)^2
& =  (I_A + I_{\bar A})^2 \\
& \leq 2 (I^2_A + I^2_{\bar A}) \\
& \leq 2 \eps_6 ( {\cal C}_{ii} +   S) ,
\end{align*}
where
$$\eps_6 = \max \left \{c_1 c_2 \lambda^3 e^{-\sqrt{\lambda}},
\frac{\eps_4^2 \Delta_i}{c_{\cal H}} 
\right\}.$$
We see that $\eps_6 \to 0$ as $\lambda \to \infty$.
Hence, 
${\cal C}_{ii}
\geq S - 2 \eps_6 ( {\cal C}_{ii} +   S)$ 
or
\begin{align*}
{\cal C}_{ii}
\geq S \frac{1 - 2 \eps_6}{1+2\eps_6} .
\end{align*}
It follows that
\begin{align}
(1-\eps_7) S \leq  {\cal C}_{ii} \leq S. \label{eqn:CiiVsIi}
\end{align}
Consequently,
\begin{align}
(x_i^*-\mathcal{E}_i)^2  
& \leq 2 \eps_6 ( {\cal C}_{ii} +   S) \leq
6 \eps_6 {\cal C}_{ii}
= \eps_8 {\cal C}_{ii} ,
\label{eqn:xIEiVsCii}
\end{align}
as long as $\eps_7 \leq 1/2$.
We see that $\eps_8 \to 0$
when $\lambda \to \infty$.

We now claim that $\Delta_i S$
hardly depends on $i$.
Let
$$S_{A} : = c_{\cal H} \int_A \frac{r^2_i}{\Delta_i} G(\vec{r}) \textrm{d}\vec{r}$$
and $S_{\bar A} = S - S_{A}$.
Let 
\begin{equation}\label{eq:alphadefined}
\displaystyle 
\alpha = \frac{1}{c_{\cal H} \int_{\bar A} \phi \cdot r^2_i \lambda (1-\texttt{CDF}_\psi(\|\vec{r}\|^2))^{\lambda-1} ~ \textrm{d}\vec{r}},
\end{equation}
where $\phi > 0$ is defined in \eqref{eqn:phiDef};
note that $\alpha > 0$ does not depend on $i$.
Bound \begin{align*}
\left| \frac{1}{\alpha} - \Delta_i S_{\bar A} \right|
& \leq \eps_2 c_{\cal H} \int_{\bar A} r^2_i \lambda (1-\texttt{CDF}_\psi(\|\vec{r}\|^2))^{\lambda-1}  ~ \textrm{d}\vec{r}\\
& = \eps_2 \frac{1}{\phi \alpha} ,
\end{align*}
which implies 
$$| 1 - \alpha \Delta_i S_{\bar A}| \leq \eps_2 \frac{1}{\phi} =
\eps_9 .$$
We see that $\eps_9 \to 0$ as $\lambda \to \infty$,
since $\phi$ does not depend on~$\lambda$.
Similarly to \eqref{eqn:IofA} and by \eqref{eq:ubCii},
we know that $S_A \leq \sfrac{{\cal C}_{ii}}{2}$ for large $\lambda$.
Since $S_{\bar A} + S_A \geq {\cal C}_{ii}$, we get
$$S_{\bar A} \geq \frac{{\cal C}_{ii}}{2} \geq \frac{1}{c_3 \lambda^2} .$$
Hence,
\begin{align}
\alpha \leq c_4 \lambda^2.
\label{eqn:alphaIslarge}
\end{align}
Similarly to \eqref{eqn:IofA} again,
$$|\alpha \Delta_i S_A |\leq \eps_{10}.$$
We see that $\eps_{10} \to 0$
as $\lambda \to \infty$.
because $e^{-\sqrt{\lambda}}$ tends
to zero faster than any polynomial in $\lambda$,

Now, bound
\begin{small}
\begin{align}
& \Delta_i I^2_{\bar A} & \notag \\
& =  c^2_{\cal H} 
\left[ \int_{\bar A}   r_i
\lambda (1-\texttt{CDF}_\psi(\|\vec{r}\|^2))^{\lambda-1} 
\left( \exp \left(- \hat{J}(\vec{r}) \right) - \phi \right)
\textrm{d}\vec{r} \right]^2 \tag{using \eqref{eqn:barAphi0}} \\
& \leq  c^2_{\cal H} 
\left[ \eps_3 \int_{\bar A}  |r_i| \lambda (1-\texttt{CDF}_\psi(\|\vec{r}\|^2))^{\lambda-1} \exp \left(- \hat{J}(\vec{r}) \right)
 \textrm{d}\vec{r} \right]^2 \tag{using \eqref{eqn:phiDef}} \\
& \leq \eps_{11} \int_{\bar A}  r_i^2 \lambda (1-\texttt{CDF}_\psi(\|\vec{r}\|^2))^{\lambda-1} 
 \exp \left(- \hat{J}(\vec{r}) \right) \textrm{d}\vec{r} \tag{convexity} \\
& \leq \eps_{12} \int_{\bar A}  r_i^2 \lambda (1-\texttt{CDF}_\psi(\|\vec{r}\|^2))^{\lambda-1} 
  \textrm{d}\vec{r} \tag{$\bar A$ is bounded} \\
 & \leq \eps_{13} \frac{1}{\alpha}, \label{eqn:I2barA}
 \end{align}
\end{small}
where $\eps_{13}$ is some constant times $\eps_3$.

Hence, using \eqref{eqn:alphaIslarge} and \eqref{eqn:IofA},
\begin{align}
\label{eqn:xiEileq}
(x_i^*-\mathcal{E}_i)^2
 \leq  2 (I_A^2 + I_{\bar A}^2 )
 \leq \eps_{14} \frac{1}{\alpha \Delta_i} .
\end{align}
Similarly to that $\eps_8 \to 0$
as $\lambda \to \infty$ in~\eqref{eqn:xIEiVsCii},
we see that $\eps_{14} \to 0$ as $\lambda \to 0$.
Finally,
\begin{align*}
| 1 - \alpha \Delta_i {\cal C}_{ii} |
= |1 - \alpha \Delta_i (S_A + S_{\bar A} - (x_i^*-\mathcal{E}_i)^2 )| \\
\leq \eps_{9}+\eps_{10}+\eps_{14} = \eps_{15} .
\end{align*}
\textbf{This completes the treatment of the diagonal of $\alpha {\cal C} {\cal H}$}.

\medskip

Let us move to the off-diagonal term ${\cal C}_{ij}$ for $i \neq j$.
With the same substitution, using \eqref{eq:proofExp_i-X_*},
\begin{small}
\begin{align*}
\mathcal{C}_{ij}  
& = c_{\cal H} \int \left(\frac{r_i}{\sqrt{\Delta_i}} + x_i^* - \mathcal{E}_i \right) \left(\frac{r_j}{\sqrt{\Delta_j}} + x_j^* - \mathcal{E}_j \right)
G(\vec{r}) \textrm{d}\vec{r} \\
& = \left( c_{\cal H} \int \frac{r_i r_j}{\sqrt{\Delta_i \Delta_j}}
G(\vec{r}) \textrm{d}\vec{r} \right)
  - (x_j^* - \mathcal{E}_j ) (x_i^* - \mathcal{E}_i ) .
\end{align*}
\end{small}

We need to show that the two summands are small, even when
multiplied by $\alpha$.
The second summand is small by~\eqref{eqn:xiEileq}.
Bound the first summand as follows.
By symmetry,
\begin{align*}
 \int_{\bar A} 
 \phi \cdot r_i r_j \lambda (1-\texttt{CDF}_\psi(\|\vec{r}\|^2))^{\lambda-1} ~ \textrm{d}\vec{r} = 0.
\end{align*}
Write
\begin{small}
\begin{align*}
\Bigg| \int  & r_i r_jG(\vec{r}) \textrm{d}\vec{r} \Bigg| \\
 & = 
\Bigg| \int  r_i r_j  \lambda (1-\texttt{CDF}_\psi(\|\vec{r}\|^2))^{\lambda-1} 
\left( \exp \left(-\hat{J}(\vec{r})\right) - \phi \right)
 \textrm{d}\vec{r} \Bigg| \\
& \leq \eps_3 \int  |r_i| |r_j|  \lambda (1-\texttt{CDF}_\psi(\|\vec{r}\|^2))^{\lambda-1} 
 \exp \left(-\hat{J}(\vec{r})\right)  \textrm{d}\vec{r} 
 \tag{using \eqref{eqn:phiDef}} \\
 & \leq \eps_{16} 
 \sqrt{\int  r_i^2   \lambda (1-\texttt{CDF}_\psi(\|\vec{r}\|^2))^{\lambda-1} 
 \exp \left(-\hat{J}(\vec{r})\right)  \textrm{d}\vec{r}}  \\ 
 & \qquad \times  
 \sqrt{\int  r_j^2   \lambda (1-\texttt{CDF}_\psi(\|\vec{r}\|^2))^{\lambda-1} 
 \exp \left(-\hat{J}(\vec{r})\right)  \textrm{d}\vec{r}}  \tag{Cauchy-Schwartz}
\\
 & \leq \eps_{17} \frac{1}{\alpha}. \tag{see end of \eqref{eqn:I2barA}}
\end{align*}
\end{small}
Here $\eps_{17}$ is some constant times $\eps_3$.
Finally,
$$ \big| (\alpha {\cal C} {\cal H})_{ij} \big|
= \big| \alpha {\cal C}_{ij} \Delta_j \big| \leq 
c_5 \eps_{17} +  c_6 \eps_{14} = \eps_{18} .$$
\end{proof}

\medskip

\textbf{Next, we show how to handle non-diagonal ${\cal H}$.}

\begin{prp} \label{prop:commuting}
Let the orthogonal matrix $\mathcal{U}$ diagonalize $\mathcal{H}$; that is, 
 $\mathcal{U} \mathcal{H} \mathcal{U}^T = \mathcal{D} = \textrm{diag}\left[\Delta_1,\Delta_2,\ldots,\Delta_n\right]$. Let $\mathcal{C}_{\mathcal{H}}$ and $\mathcal{C}_{\mathcal{D}}$ denote the covariance matrices over winning decision vectors 
for $J_{{\cal H},\vec{x}^*}$ and $J_{{\cal D},{\cal U}\vec{x}^*}$, respectively. Then, 
$$\mathcal{C}_{\mathcal{D}} = \mathcal{U} \mathcal{C}_{\mathcal{H}} \mathcal{U}^T.$$ 
\end{prp}
\begin{proof}
The relation between the two objective functions reads:
\begin{align}
J_{\mathcal{H},\vec{x}^*}(\vec{x})& =  (\vec{x}-\vec{x}^{*})^T \cdot \mathcal{H} \cdot (\vec{x}-\vec{x}^{*}) \notag\\
& = (\vec{x}-\vec{x}^{*})^T \cdot \mathcal{U}^T \mathcal{D} \mathcal{U} \cdot (\vec{x}-\vec{x}^{*}) \notag\\
& = (\mathcal{U} (\vec{x}-\vec{x}^{*}))^T \cdot \mathcal{D} \cdot \mathcal{U}(\vec{x}-\vec{x}^{*}) \notag\\
& = J_{\mathcal{D}, {\cal U} \vec{x}^*}(\mathcal{U}\vec{x}).
\end{align}
Now, 
consider the following experiment. 
Sample $\lambda$ Gaussian mutations
$\vec{x}_1,\ldots,\vec{x}_\lambda$. 
Let $\vec{y}_{{\cal H}}$ be the winner among $\vec{x}_1,\ldots,\vec{x}_\lambda$
with respect to  $J_{\mathcal{H}, \vec{x}}$.\\
The winner among $\mathcal{U}\vec{x}_1,\ldots,\mathcal{U}\vec{x}_\lambda$ 
with respect to $J_{\mathcal{D}, {\cal U} \vec{x}^*}$
is 
$$\vec{y}_{{\cal D}} = {\cal U} \vec{y}_{{\cal H}}.$$
Since $\mathcal{N} (\vec{0},\mathbf{I})$ is invariant under rotations, 
the vectors\\ 
$\mathcal{U}\vec{x}_1,\ldots,\mathcal{U}\vec{x}_\lambda$ 
are also distributed as $\lambda$ Gaussian mutations.
The expectation vectors over such winners thus satisfy:
$$\mathcal{E}_{\mathcal{H}} = \E \left[ \vec{y}_{\cal H}\right] 
=  \E \left[\mathcal{U}^T \vec{y}_{\cal D} \right] = \mathcal{U}^T \E \left[\vec{y}_{\cal D} \right] =
 \mathcal{U}^T \mathcal{E}_{\mathcal{D}} .$$
The covariance matrices read:
\begin{align*}
\mathcal{C}_{\mathcal{H}} 
& = \E \left[(\vec{y}_{\cal H}-\mathcal{E}_{\mathcal{H}})  (\vec{y}_{\cal H}-\mathcal{E}_{\mathcal{H}})^T \right] \\
& = \E \left[(\mathcal{U}^T\vec{y}_{\cal D} -\mathcal{U}^T\mathcal{E}_{\mathcal{D}})  (\mathcal{U}^T\vec{y}_{\cal D} -\mathcal{U}^T \mathcal{E}_{\mathcal{D}})^T \right]\\
& = \mathcal{U}^T \E \left[(\vec{y}_{\cal D}-\mathcal{E}_{\mathcal{D}})  (\vec{y}_{\cal D}-\mathcal{E}_{\mathcal{D}})^T \right]\mathcal{U} \\
& = \mathcal{U}^T \mathcal{C}_{\mathcal{D}}\mathcal{U}.
\end{align*}
\end{proof}

\textbf{The general result is thus obtained:}

\medskip

\noindent\fbox{
\parbox{0.45\textwidth}{%
\begin{thm} \label{thm:generalInverseRelation}
For every invertible $\cal H$ and $\lambda \in \N$, there exists a constant $\alpha({\cal H},\lambda) > 0$ such that
$$\lim_{\lambda \to \infty}  \alpha \mathcal{C}  \mathcal{H} = \mathbf{I}.$$
\end{thm}
}
}

\medskip

Before proving the theorem, we note that (\textit{a posteriori}) for a given $\lambda$,
one can choose $\alpha$ as $\sfrac{1}{\beta}$
for 
$$\beta = \max \{ ({\cal C}{\cal H})_{ij} : i,j\}.$$

\begin{proof}

Proposition \ref{prop:commuting} shows for a general (non-diagonal) Hessian matrix
${\cal H}$ with diagonalizing matrix ${\cal U}$ and covariance matrix ${\cal C}$,
the matrix $\mathcal{U} \mathcal{C} \mathcal{U}^T$ is 
the covariance matrix for the diagonal Hessian 
${\cal D} = \mathcal{U} \mathcal{H} \mathcal{U}^T$.
Applying Proposition~\ref{prop:inverseRelation} to ${\cal D}$ yields
that
\begin{align*}
\mathbf{0} 
& = \lim_{\lambda \to \infty}  \mathbf{I} - \alpha \ \mathcal{U} \mathcal{C} \mathcal{U}^T \
\mathcal{U} \mathcal{H} \mathcal{U}^T  = \lim_{\lambda \to \infty} {\cal U} (\mathbf{I} - \alpha  \mathcal{C}  \mathcal{H}) \mathcal{U}^T ,
\end{align*}
as needed.
\end{proof}

\section{Numerical Validation}\label{sec:simulation}
We present numerical validation for two aspects: 
\begin{itemize}
\item Section \ref{sec:numPDF} assesses the nature of the winners' distribution. 
\item Section \ref{sec:simCH} validates Propositions \ref{prop:inverseRelation} and \ref{prop:commuting} by accounting for the deviation of $\mathcal{H}\mathcal{C}^{\texttt{stat}}$ away from the identity as a function of increasing $\lambda$ and/or translating the sampling point farther away from the optimum.
It also presents a systematic evaluation of the inverse relation when the Hessian's conditioning varies.
\end{itemize}
We consider three separable and two non-separable Hessian matrices:
\begin{enumerate}[(H-1)]
\item Discus:~$\left( \mathcal{H}_{\textrm{disc}}\right)_{11} = c,~\left( \mathcal{H}_{\textrm{disc}}\right)_{ii} = 1~~~i=2,\ldots, n$
\item Cigar:~$\left( \mathcal{H}_{\textrm{cigar}}\right)_{11} = 1,~\left( \mathcal{H}_{\textrm{cigar}}\right)_{ii} = c~~~i=2,\ldots,n$
\item Ellipse:~$\left( \mathcal{H}_{\textrm{ellipse}}\right)_{ii} = c^{\frac{i-1}{n-1}}$
\item Rotated Ellipse:~$\mathcal{H}_{\textrm{RE}} = \mathcal{R}\mathcal{H}_{\textrm{ellipse}} \mathcal{R}^{-1}$ where $\mathcal{R}$ is rotation by $\approx \frac{\pi}{4}$ radians in the plane spanned by $(1,0,1,0,\ldots)^T$ and $(0,1,0,1,\ldots)^T$;
\item Hadamard Ellipse:~$\mathcal{H}_{\textrm{HE}} = \mathcal{S}\mathcal{H}_{\textrm{ellipse}} \mathcal{S}^{-1}$ where $\mathcal{S}:=\textrm{Hadamard}(n)/\sqrt{n}$,
\end{enumerate}
with $c$ denoting a parametric condition number. \\
In practice, sampling is carried out as follows ($\mathcal{H}_0$ refers to 
one of the above matrices):
\begin{equation}\label{eqn:QSampling}
\displaystyle J\left(\vec{z} \right)= \vec{z}^{T}\mathcal{H}_0 \vec{z} + \vec{a}^{T}\vec{z},
\end{equation}
where $\vec{z}$ is normally distributed, and $\vec{a}$ is the translation vector.
Unless specified otherwise, the translation is set to a vector of ones, $\vec{a}:=\vec{1}$.

Finally, we note that although $\mathcal{C}$ and $\mathcal{C}^{\texttt{stat}}$ tend to zero, no regularization was needed \textit{de facto} in our experiments.
Nevertheless, due to this tendency to zero, regularization may be needed when attempting to extract landscape information
(considering the fact that $\alpha$ is unknown).
For instance, the study in \cite{Shir-FOCAL} reported the usage of Tikhonov filtering \cite{regtutorial} to this end.

All calculations were implemented using \texttt{python3}.\footnote{Source code: \url{https://github.com/ofersh/invHessian}}

\subsection{Probability Density Functions}\label{sec:numPDF}
We are interested in providing estimates for the distributions that are relevant to Theorem~\ref{thm:generalInverseRelation}. 
The $n$-dimensional density $\texttt{PDF}_{\vec{y}}$, which lies in the center of proof, is difficult to visualize.
We therefore numerically assess only the following two distributions:
\begin{enumerate}
\item The objective function values for a single sample | $\texttt{PDF}_{\psi}$. 
\item The winning objective function values per competitions over $\lambda$ individuals | $\texttt{PDF}_{\omega}$.
\end{enumerate}
These underlying probability functions for the quadratic model were extensively explored in \cite{Shir-Theory-foga17} for the near-optimum case.
We refer the reader to \cite{Shir-Theory-foga17} for the complete details, and at the same time summarize these functions in Table \ref{table:probfunc} below.
In practice, $\texttt{PDF}_{\omega}$ requires both $\texttt{CDF}_{\psi}$ and $\texttt{PDF}_{\psi}$ (Eq.~\ref{eq:y_pdf}), which in the quadratic case adhere to the $\chi^2$-distribution.
The generalized $\chi^2$-distribution for any proper Hessian \cite{MOSCHOPOULOS1984383} $F_{\mathcal{H}\chi^2}$ is \textbf{approximated} \cite{FD68_NASA} by $F_{\tau\chi^2}$.
In our calculations, we used $F_{\tau\chi^2}$ instead of $F_{\mathcal{H}\chi^2}$ to approximate the density $\texttt{PDF}_{\omega}$;
see Eq.~\ref{eq:pdf_omega_practice} below.
The density $\texttt{PDF}_{\omega}$ may alternatively be approximated by GEVD, which is not used herein.\footnote{As already mentioned in Section \ref{sec:previousresults}, the GEVD approximately describes the winning values when the population-size increases, as was shown in~\cite{Shir-Theory-foga17}.}
\begin{table*}[ht]
\centering
\begin{small}
\begin{tabular}{|l p{6.5cm}|}
\hline
exact generalized $\chi^2$-distribution &
\begin{equation}\label{eq:cdf_chi}
\begin{array}{l}
\medskip
\displaystyle 
F_{\mathcal{H}\chi^2}(\psi) = \int_0^{\infty} \frac{2}{\pi} \frac{\sin\frac{t\psi}{2}}{t}\\
\displaystyle \times \cos \left(-t\psi + \frac{1}{2} \sum_{j=1}^{n}\tan^{-1}2\Delta_j t \right) \\
\displaystyle \times \prod_{j=1}^{n}\left(1+\Delta_j^2 t^2 \right)^{-\frac{1}{4}} ~ \textrm{d}t
\end{array}
\end{equation}\\
\hline
approx.\ generalized $\chi^2$-distribution &
\begin{equation}\label{eq:Nchi}
\displaystyle 
F_{\tau\chi^2}\left(\psi \right) = \frac{\Upsilon^{\eta}}{\Gamma\left(\eta\right)} \int_0^{\psi}t^{\eta-1}\exp\left(-\Upsilon t \right) ~ \textrm{d}t
\end{equation}\\
& \begin{equation}
\Upsilon=\frac{1}{2}\frac{\sum_{i=1}^{n}\Delta_i}{\sum_{i=1}^{n}\Delta_i^2},~~~\eta=\frac{1}{2}\frac{\left(\sum_{i=1}^{n}\Delta_i\right)^2}{\sum_{i=1}^{n}\Delta_i^2}
\end{equation}\\
approx.\ generalized $\chi^2$-density &
\begin{equation}\label{eq:Nchi_density}
\displaystyle f_{\tau\chi^2}\left(\psi \right) = \frac{\Upsilon^{\eta}}{\Gamma\left(\eta\right)} \psi^{\eta-1}\exp\left(-\Upsilon \psi \right)
\end{equation}\\
\hline
\hline
approx.~density of $(1,\lambda)$-winning values &
\begin{equation}\label{eq:pdf_omega_practice}
\displaystyle \texttt{PDF}_{\omega}\left( v \right) =   \lambda \cdot \left( 1-F_{\tau\chi^2}\left(v \right)\right)^{\lambda-1}\cdot f_{\tau\chi^2}\left(v\right)
\end{equation}\\
\hline
\end{tabular}
\end{small}
\caption{Summary of probability functions describing quadratic functions' values near the optimum, as was explored in \cite{Shir-Theory-foga17}.
The approximate $(1,\lambda)$-winning values' density function refers to the calculated form of the winning values in practice. \label{table:probfunc}}
\end{table*}

We consider the (H-1)-(H-4) landscapes with conditioning $c=10$ at dimensionality $n=64$.
The objective function values are generated over a sample of $10^5$ individuals when constructing $\texttt{PDF}_{\psi}$. 
The winners are selected out of a population-size of $\lambda=10^3$ over $N_{\texttt{iter}}=10^6$ competitions when constructing $\texttt{PDF}_{\omega}$. 
The empirical densities are constructed by histograms and depicted side-by-side with the analytical forms in Figures \ref{fig:densities1} and \ref{fig:densities2}.

Firstly, it is evident that $f_{\tau\chi^2}$ constitutes a sound approximation for the generalized $\chi^2$-distribution over the various Hessian forms. Secondly, $\texttt{PDF}_{\omega}$, whose analytical form uses $F_{\tau\chi^2}$, exhibits high accuracy on all cases.
\begin{figure}
\begin{tabular}{c c}
\hline
(H-1) & \\
\epsfig{file=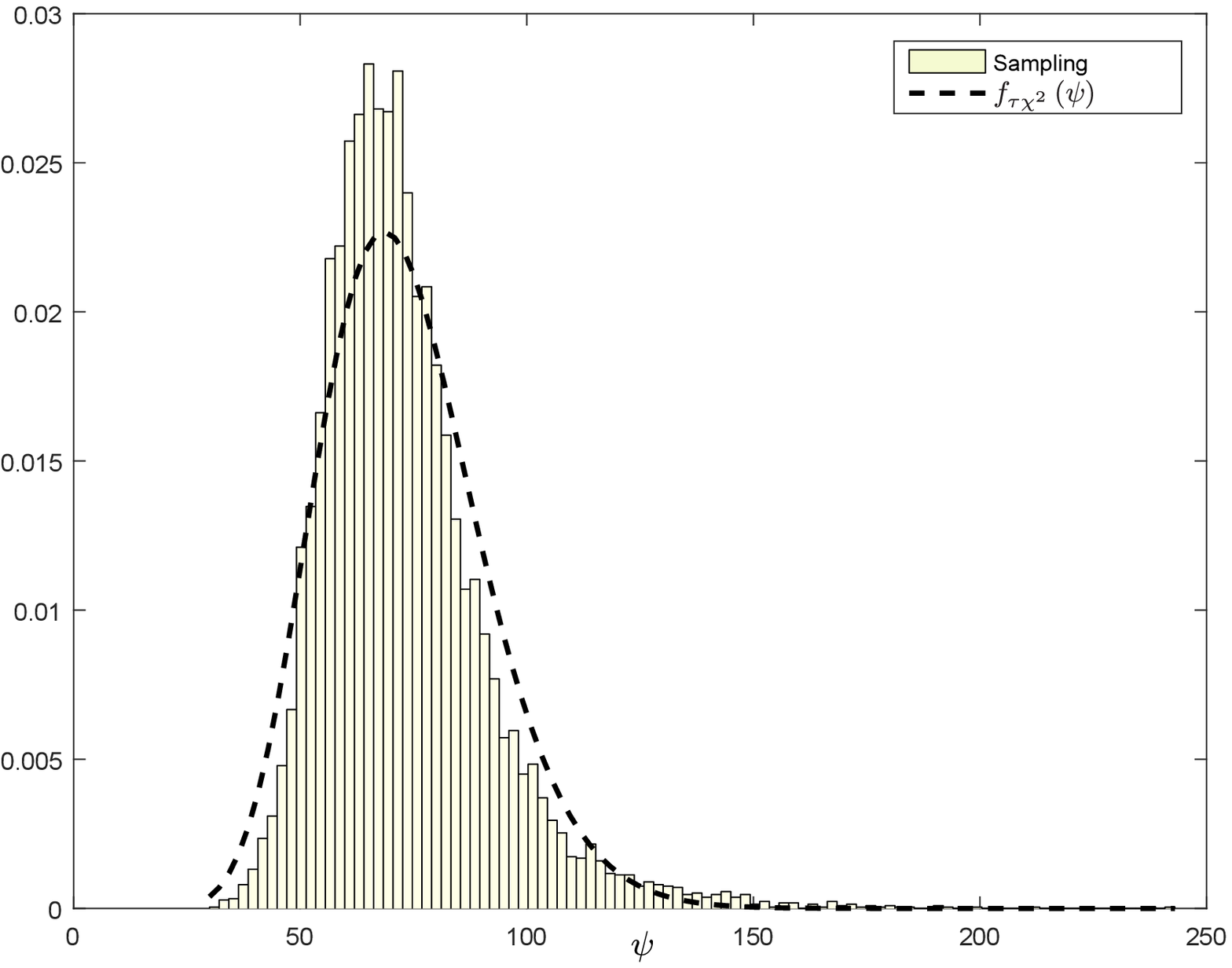, width=0.45\columnwidth} & \epsfig{file=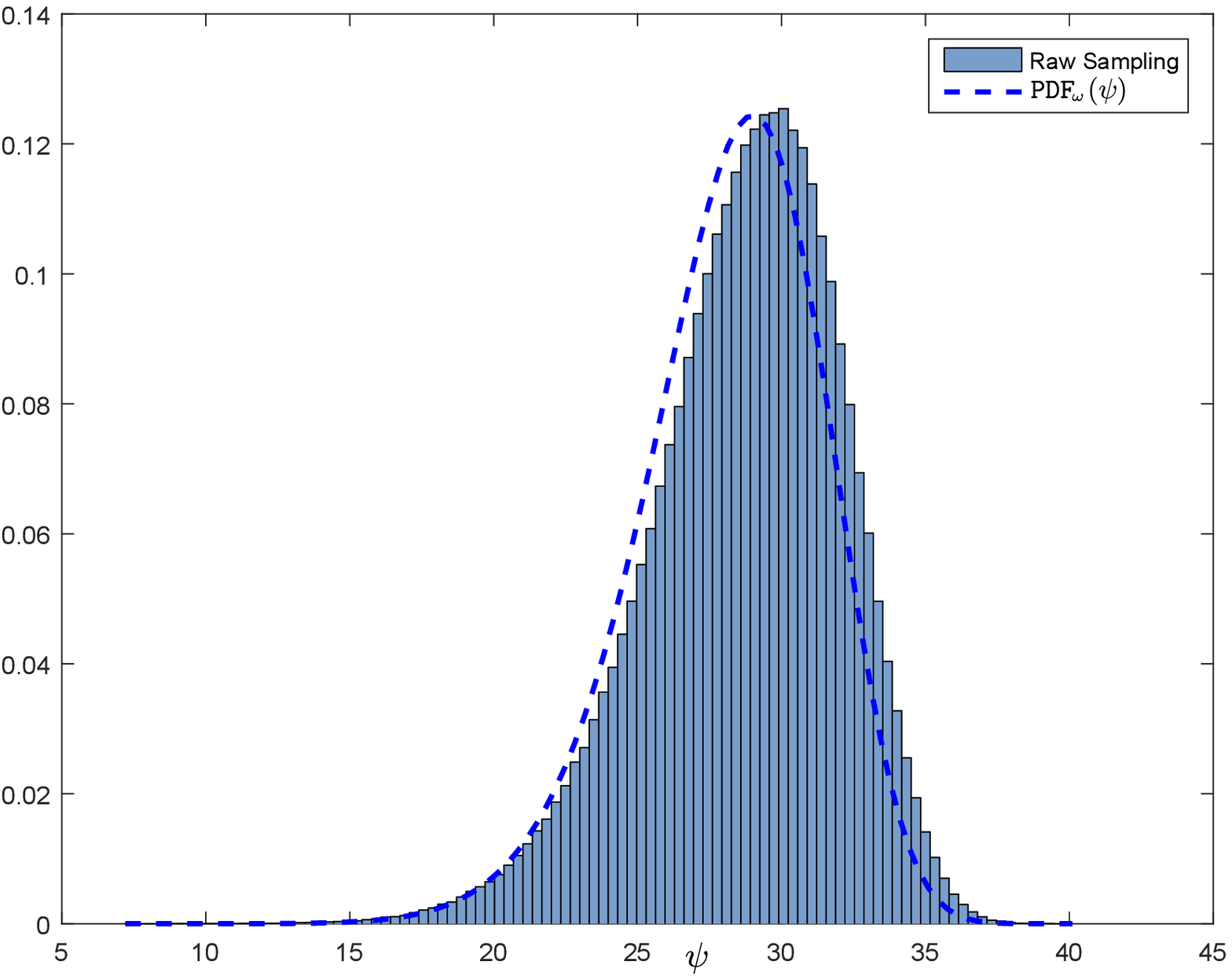, width=0.45\columnwidth} \\
\hline
(H-2) & \\
\epsfig{file=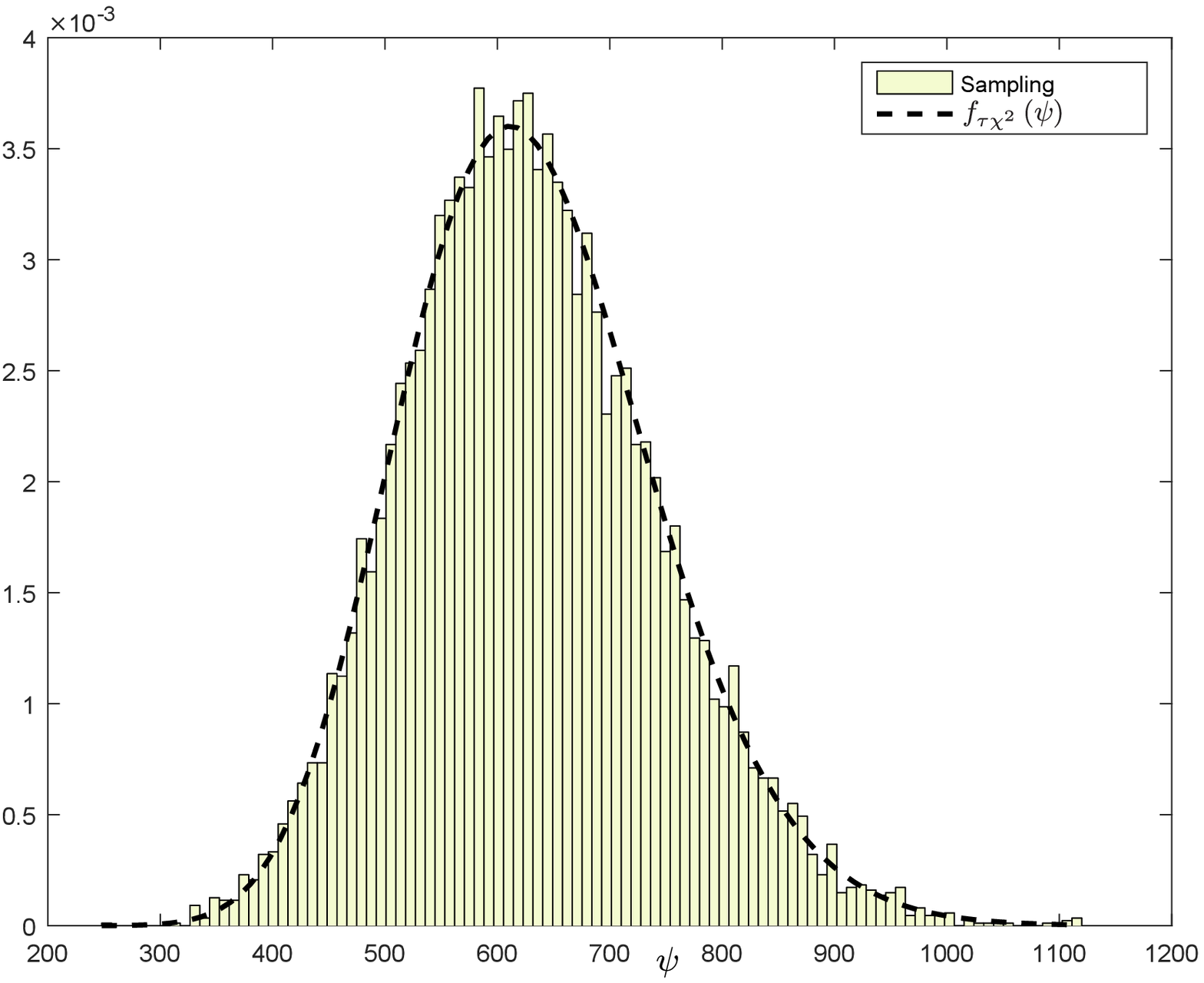, width=0.45\columnwidth} & \epsfig{file=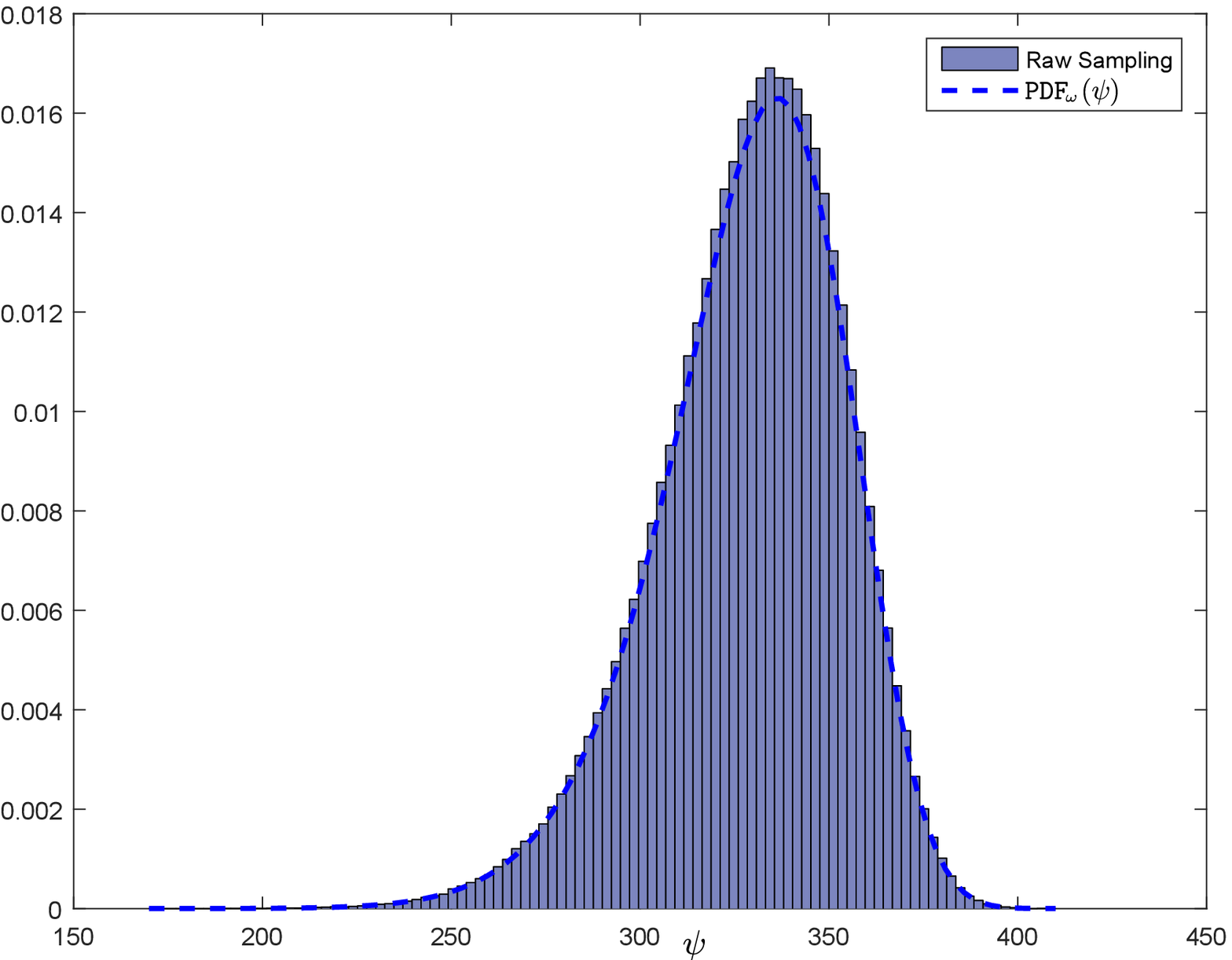, width=0.45\columnwidth} \\
\hline
\end{tabular}
\caption{Statistical demonstration of approximated densities for (H-1) and (H-2).
[LEFT]: The statistical density for a single mutation versus 
the approximation $f_{\tau\chi^2}$; see Eq.~\ref{eq:Nchi_density}.
[RIGHT]: 
The statistical density for the winner versus its approximation;
see Eq.~\ref{eq:pdf_omega_practice}.\label{fig:densities1}}
\end{figure}
\begin{figure}
\begin{tabular}{c c}
\hline
(H-3) & \\
\epsfig{file=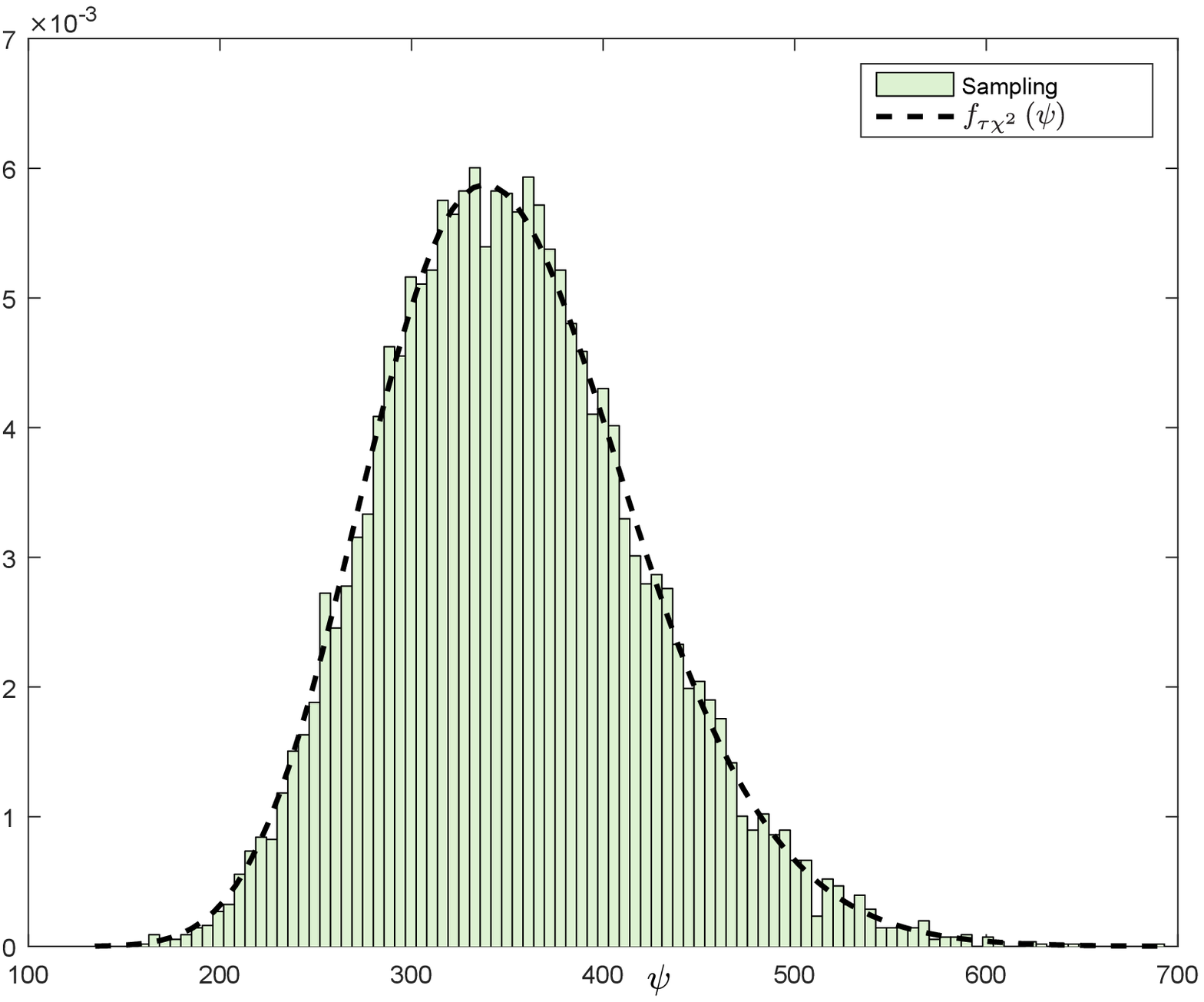, width=0.45\columnwidth} & \epsfig{file=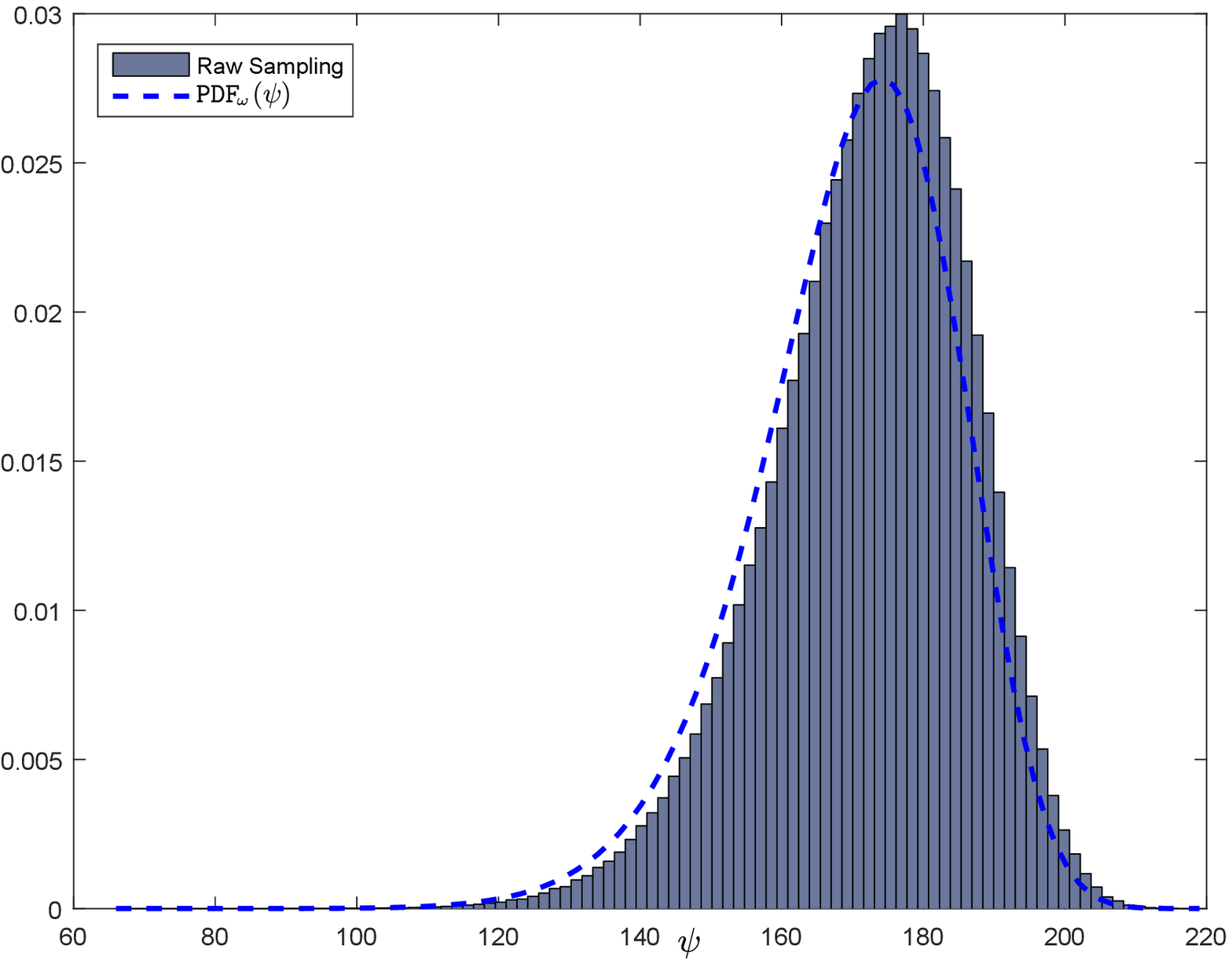, width=0.45\columnwidth} \\
\hline
(H-4) & \\
\epsfig{file=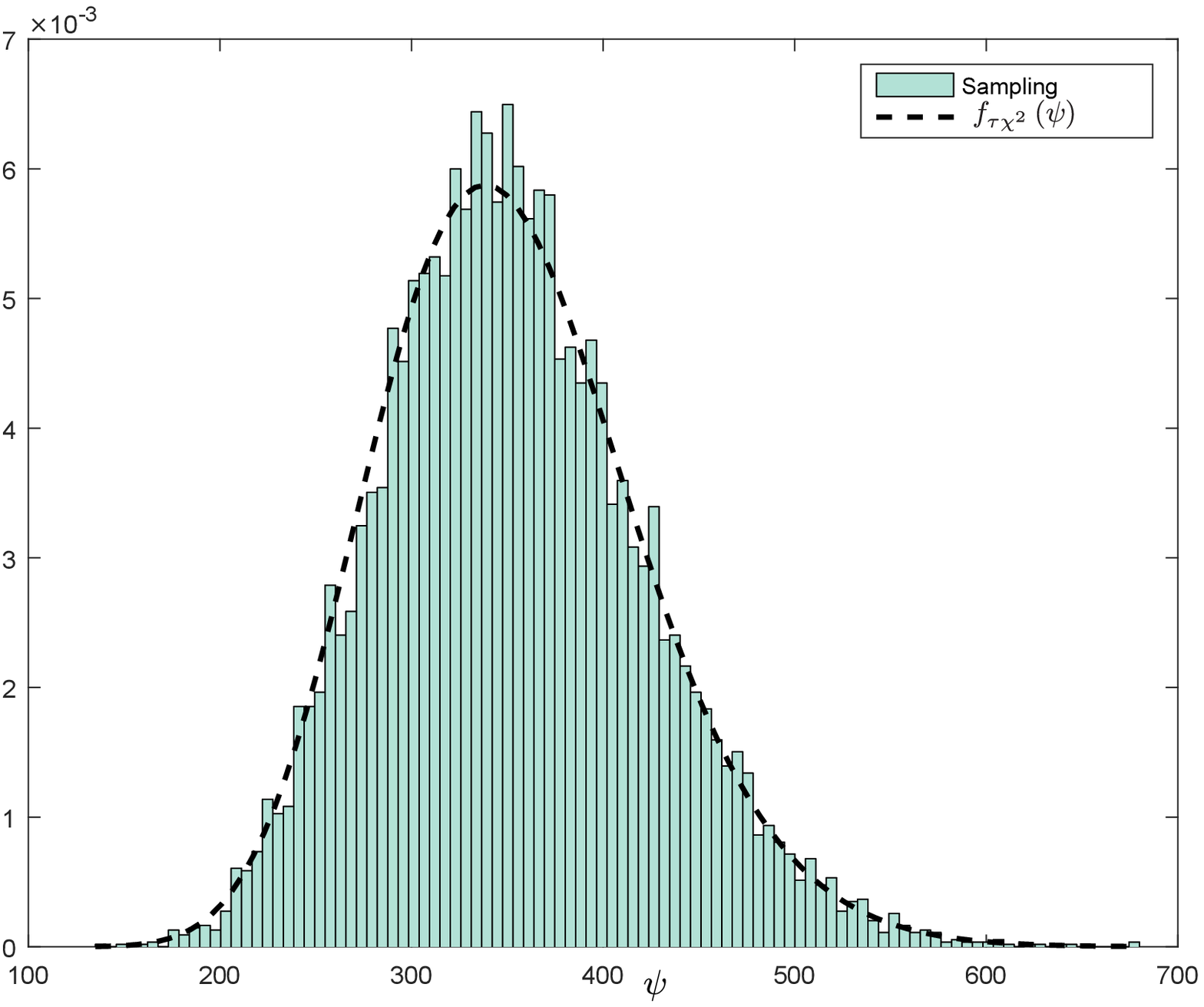, width=0.45\columnwidth} & \epsfig{file=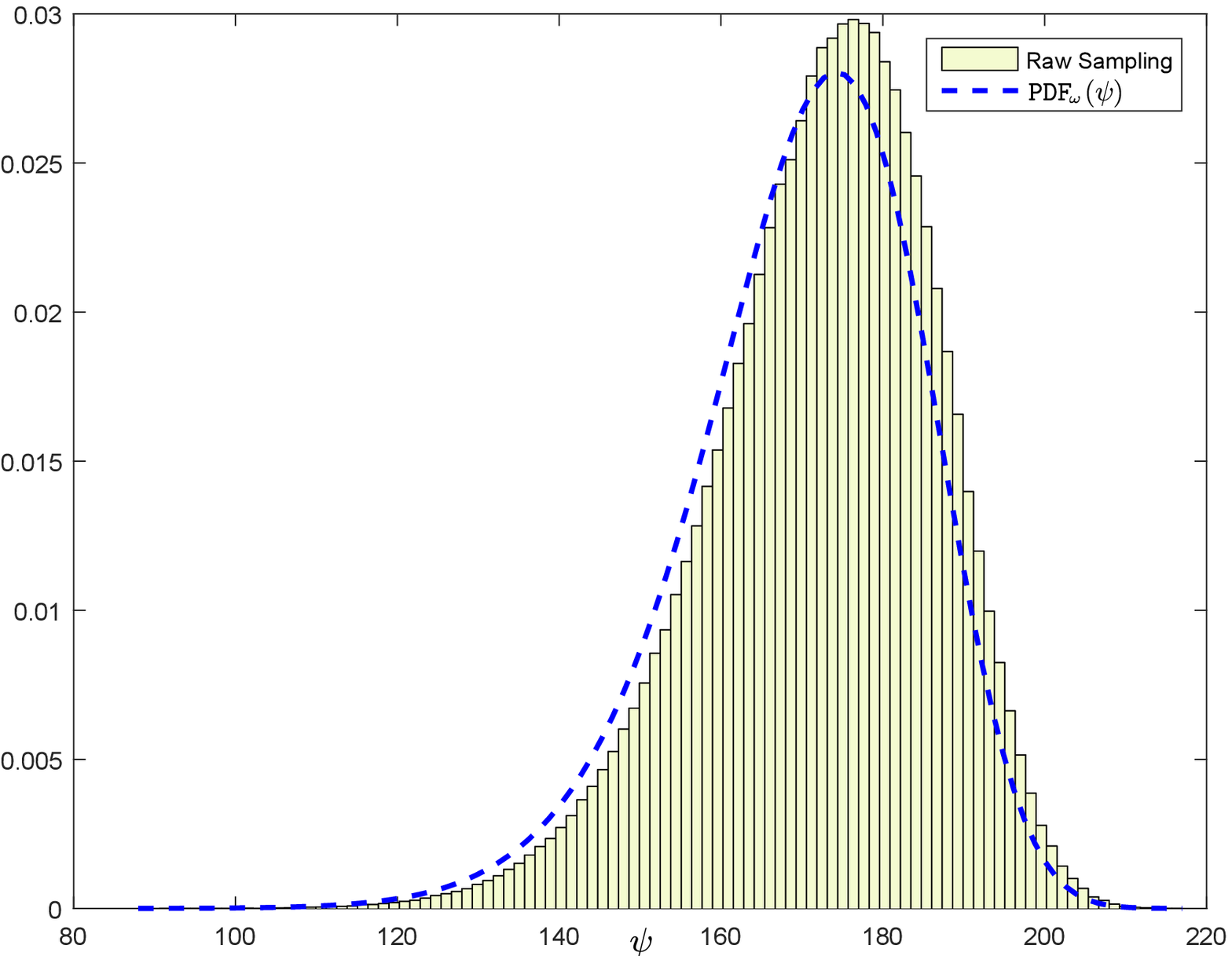, width=0.45\columnwidth} \\
\hline
\end{tabular}
\caption{Statistical demonstration of approximated densities for (H-3) and (H-4).
[LEFT]: The statistical density for a single mutation versus 
the approximation $f_{\tau\chi^2}$; see Eq.~\ref{eq:Nchi_density}.
[RIGHT]: The statistical density for the winner versus its approximation;
see Eq.~\ref{eq:pdf_omega_practice}.%
\label{fig:densities2}}
\end{figure}

\subsection{The Inverse Relation | Theorem \ref{thm:generalInverseRelation}}\label{sec:simCH}
We numerically examine the inverse relation in light of increasing the population-size. 
The sampling was done using \eqref{eqn:QSampling} with increasing population-sizes $\lambda$ non-uniformly in $5,\ldots,10^5$,
for dimensions $n:=\left\{4,8,16,32,64\right\}$, and over $N_{\texttt{iter}}=10^6$ iterations.
The (H-1)-(H-5) landscapes are considered with conditioning $c=10$.
The winner in each iteration, $t =1,\ldots,N_{\texttt{iter}}$, is denoted by $\vec{y}_t$.
The empirical covariance matrix $\mathcal{C}^{\texttt{stat}}$ is constructed out of the winners.

Firstly, the average distance of the winners from the optimum $\vec{x}^{*}$ is assessed using the following (e-0) measure:
\begin{enumerate}[(e-1)]
\setcounter{enumi}{-1}
\item the mean distance from the optimum:
\begin{equation}\label{eq:e0}
\displaystyle \texttt{e0:=}~~~ \frac{1}{N_{\texttt{iter}}}\sum_{t=1}^{N_{\texttt{iter}}}\left[ \| \vec{y}_t - \vec{x}^{*} \| \right] ~.
\end{equation}
\end{enumerate}
As expected, our experiments show that (e-0) systematically decreases as $\lambda$ increases, with tendency to vanish.

Secondly, the empirical covariance matrix was multiplied by the Hessian matrix
and normalized by the largest element:
\begin{equation}\label{eq:HCnormalization}
  \widetilde{\mathcal{H}_{0} \mathcal{C}}:=\frac{\mathcal{H}_{0} \mathcal{C}^{\texttt{stat}}}{\max_{i,j} \left\{ |\left(\mathcal{H}_{0} \mathcal{C}^{\texttt{stat}}\right)_{ij}| \right\} }.  
\end{equation}
The deviations from the identity matrix are assessed by two error measures (e-1),(e-2):
\begin{enumerate}[(e-1)]
\item the largest deviation within the diagonal:
\begin{equation}\label{eq:e1}
\displaystyle \texttt{e1:=}~~~ \max_{i}\left\{ \left| \left(\widetilde{\mathcal{H}_{0} \mathcal{C}}\right)_{ii}- 1.0 \right| \right\} ~.
\end{equation}
\item the largest off-diagonal deviation:
\begin{equation}\label{eq:e2}
\displaystyle \texttt{e2:=}~~~ \max_{i \neq j}\left\{ \left| \left(\widetilde{\mathcal{H}_{0} \mathcal{C}}\right)_{ij}\right| \right\} .
\end{equation}
\end{enumerate}
Our experiments show that both (e-1) and (e-2) tend to 0 as $\lambda$ increases.
Note that (e-1) and (e-2) also implicitly account for the commutator error
$$\| \mathcal{H}_{0} \mathcal{C}^{\texttt{stat}} - \mathcal{C}^{\texttt{stat}} \mathcal{H}_{0}\|_{ \textrm{frob} },$$
which was directly investigated for the near-optimum special case in \cite{Shir-Theory-foga17}. 

\subsubsection*{Elaboration}
Figures \ref{fig:H1H2}-\ref{fig:H3H4} present the calculations of the error measures for the (H-1)-(H-2)-(H-3) and (H-4)-(H-5) test-cases, respectively. Evidently, all error measures tend to decrease when the population-size increases. An elaboration follows:
\begin{enumerate}[e-1]
\setcounter{enumi}{-1}
\item The mean winners' distance to the optimum shrinks with the growing population-size. 
Experiments at lower dimensions yield lower absolute distance, as expected.
\item This error measure consistently decreases as $\lambda$ grows, reflecting a consistently increasing alignment of the diagonal of $\widetilde{\mathcal{H}_{0} \mathcal{C}}$ to uniform as $\lambda\rightarrow\infty$.

(H-5) exhibits exceptional behavior. Its (e-1) value is practically zero for all $\lambda$. 
This effect occurs regardless of the sampling/learning scheme;
see additional numerical analysis below.
We mention the inherent properties of the Hadamard transformation $\mathcal{S}$ that are responsible for this effect: 
(i) it is orthogonal and (ii) the absolute values of its entries $|{\cal S}_{ij}|$ are identical constants. These properties imply that for any diagonal matrix ${\cal D}$, the diagonal of $\mathcal{S} {\cal D} {\cal S}^{-1}$ is constant.

\item This error measure is practically zero for (H-1)-(H-3);
it lies within the noise regime due to the separable (non-rotated) nature of those landscapes. 
This follows from the fact that when the Hessian is diagonal the covariance matrix is also diagonal~\cite{Shir-Theory-foga17}.
It is therefore relevant only to the non-separable landscapes (H-4) and (H-5), where it follows the same trend of decrease that occurs for (e-1). 
Notably, the ordering of (e-2) values per dimensions are counter-intuitive for low $\lambda$ on (H-4), where $n=4$ exhibits the higher error rates, followed by $n=8$, and so on. This effect is a consequence of the matrix normalization, described in Eq.~\ref{eq:HCnormalization}. 
The absolute value of the largest off-diagonal term of $\mathcal{H}_{0} \mathcal{C}^{\texttt{stat}}$ increases with the dimensionality $n$ prior to normalization, but the normalization factor (largest diagonal term) increases faster over $n$ and thus shadows this trend. This effect is not evident for (H-5), where the diagonal of $\mathcal{H}_{0} \mathcal{C}^{\texttt{stat}}$ behaves differently.
\end{enumerate}
Overall, given the observations of (e-1) and (e-2), we conclude that the normalized multiplication $\widetilde{\mathcal{H}_{0} \mathcal{C}}$ indeed approaches the identity matrix when the population-size increases, as Theorem \ref{thm:generalInverseRelation} predicts.

\medskip

We numerically investigated the \textit{somewhat surprising} effect of (e-1) vanishing on (H-5), which we discussed and explained above.
We replaced the empirical covariance matrix $\mathcal{C}^{\texttt{stat}}$ by a normally-perturbed identity matrix $$\mathbf{C}^{\texttt{pert}}:=\mathbf{I}+\mathbf{E},$$ with $\mathbf{E}$ being a symmetric matrix whose elements are independently normally distributed $\mathcal{N}(0,\varepsilon^2)$ with $\varepsilon=0.05$.
The goal was to isolate the properties of the Hadamard matrix, by removing the ``cumulation'' aspect of the covariance matrix and using a truly random matrix instead (a ``no-cumulation'' \textit{reference}).
We generated $10^6$ normally-perturbed identity matrices, and examined (e-1) and (e-2) with respect to all five Hessians. 
Figure~\ref{fig:noLearning} depicts the statistical tests evaluating both (e-1) and (e-2) on all the quadratic landscapes listed in (H-1)-(H-5) for dimensions $\left\{4,8,32\right\}$ and over conditioning in the exponential range $c = \left\{2^2,2^3,\ldots, 2^{10} \right\}$.
It is apparent that the error measure (e-1) on (H-5) for this ``no-cumulation'' reference $\mathbf{C}^{\texttt{pert}}$ is significantly low in comparison to the (H-1)-H-4) landscapes for all tested dimensions and along all condition numbers.
This behavior is consistent with the aforementioned observation when constructing covariance matrices over winning vectors (Figure \ref{fig:H3H4}). 
The behavior on the other landscapes is as expected, with the separable cases (H-1)-(H-3) asymptotically approaching together $\texttt{e1}=1.0$ as the conditioning grows. At the same time, the (e-2) measure also behaves as expected, with increasingly growing values as the conditioning grows, exhibiting changing behavior for the various cases.

\begin{figure*}
\begin{tabular}{ c  c  c}
{\Large (e-0)} & {\Large (e-1)} & {\Large (e-2)}\\
\hline
(H-1) & & \\
\epsfig{file=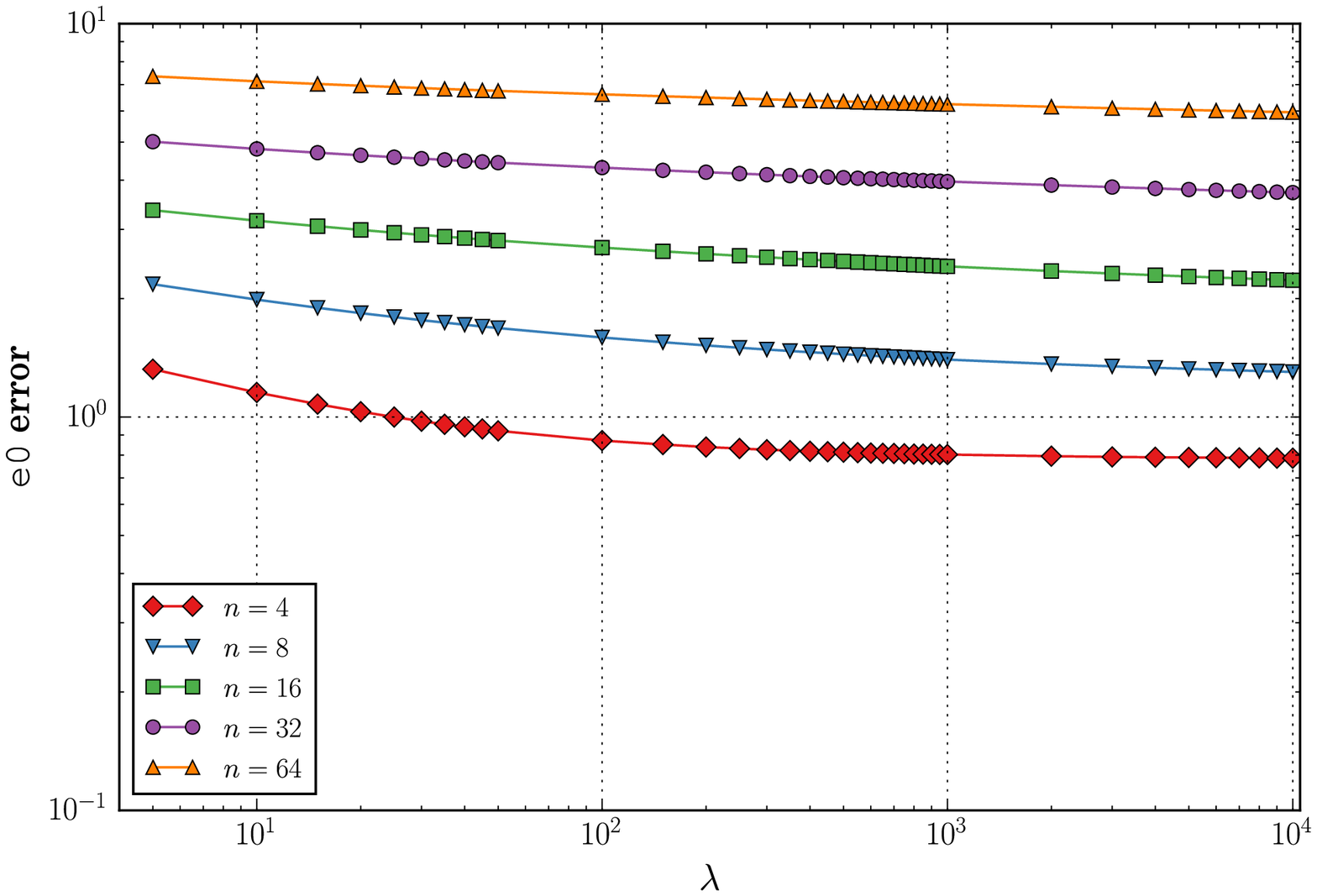, width=0.32\columnwidth} & \epsfig{file=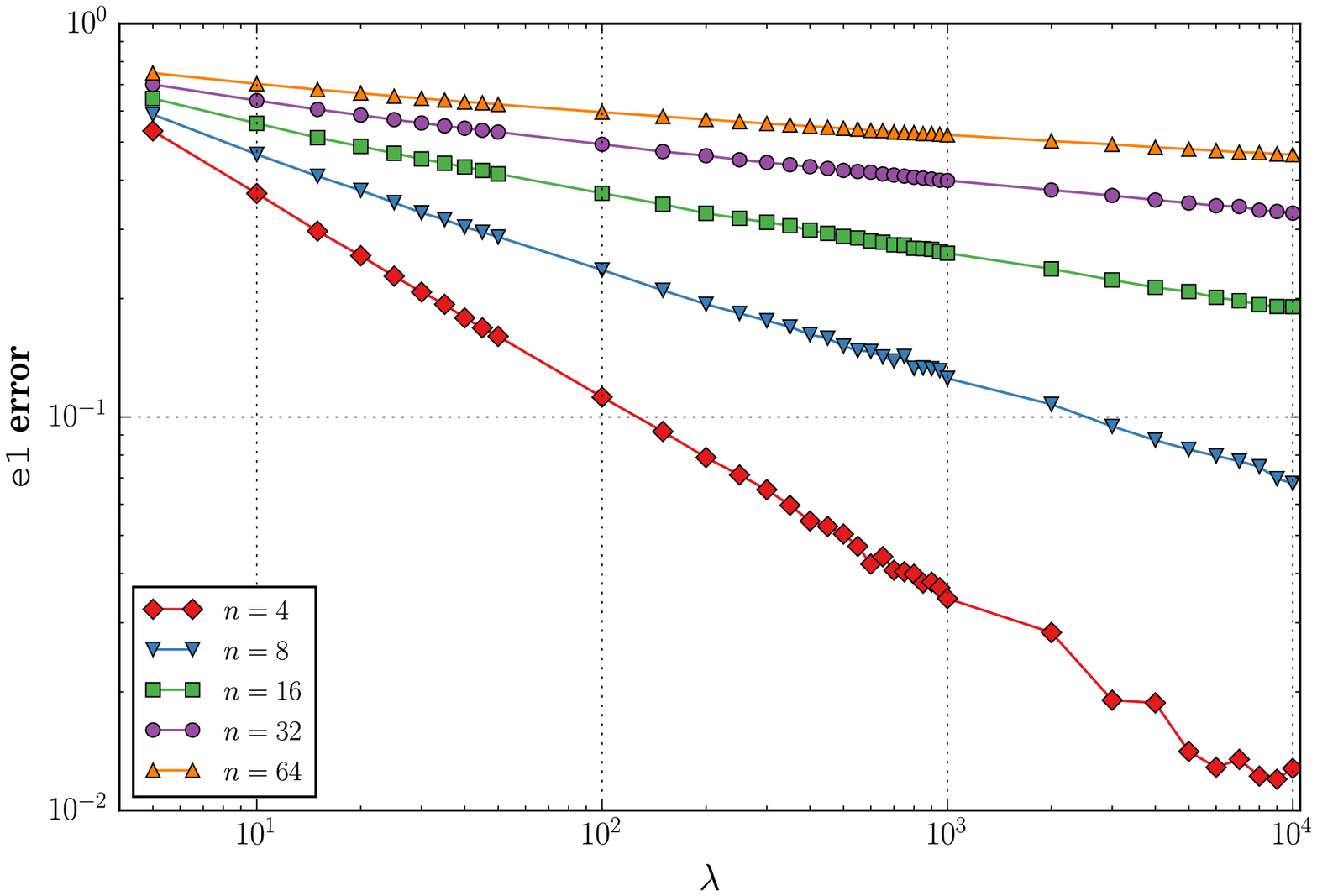, width=0.32\columnwidth} & \epsfig{file=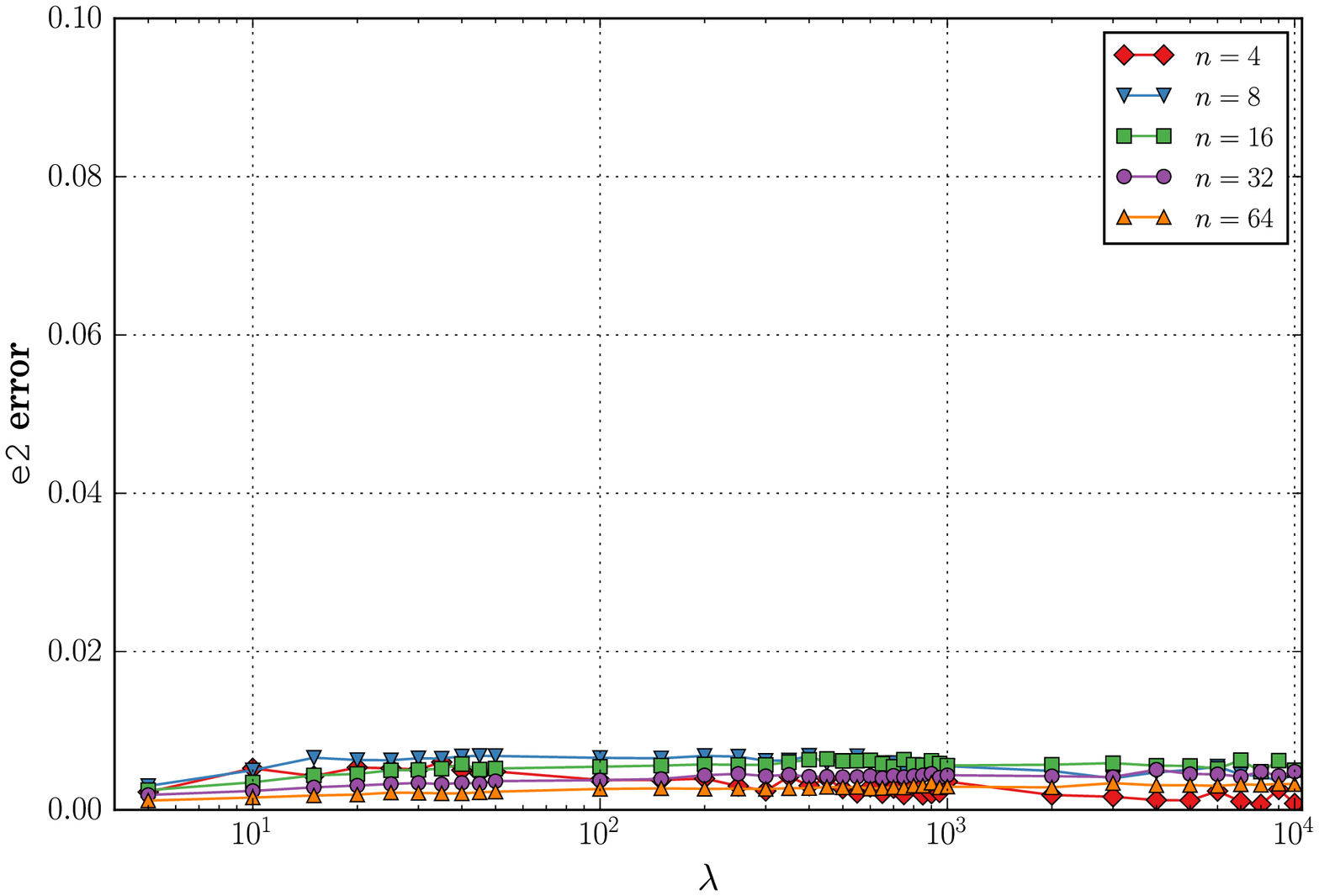, width=0.32\columnwidth} \\
\hline
(H-2) & \\
\epsfig{file=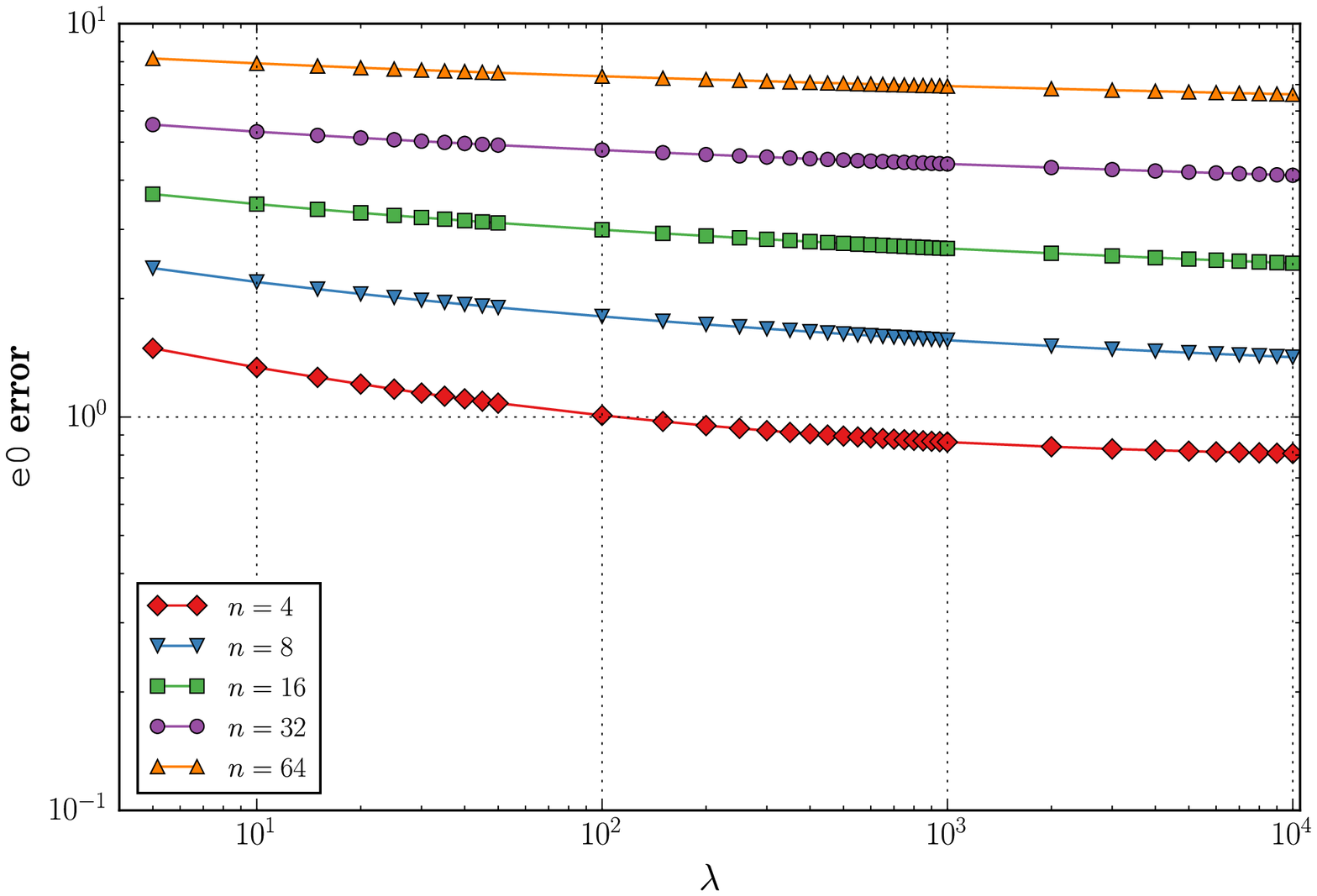, width=0.32\columnwidth} & \epsfig{file=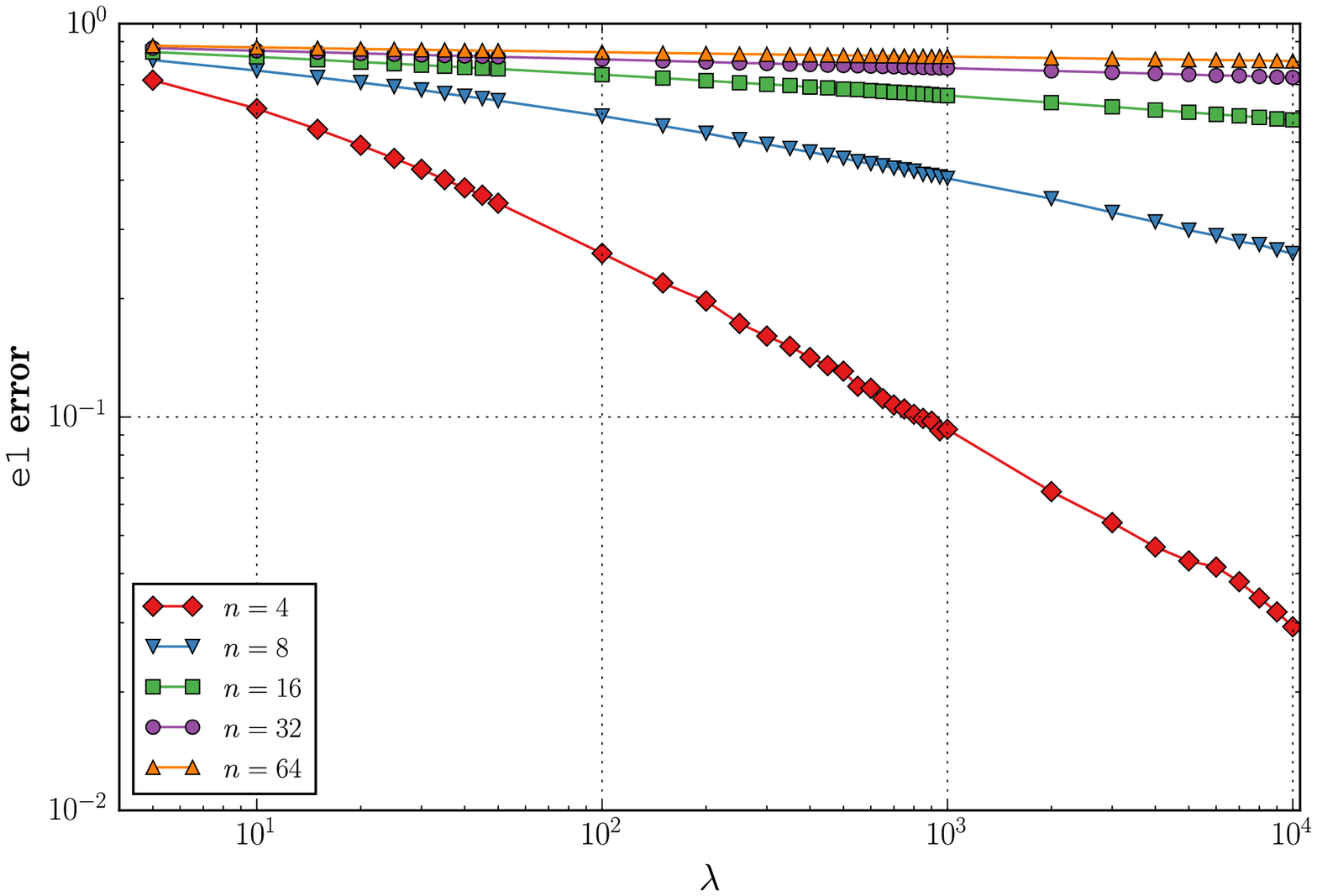, width=0.32\columnwidth} & \epsfig{file=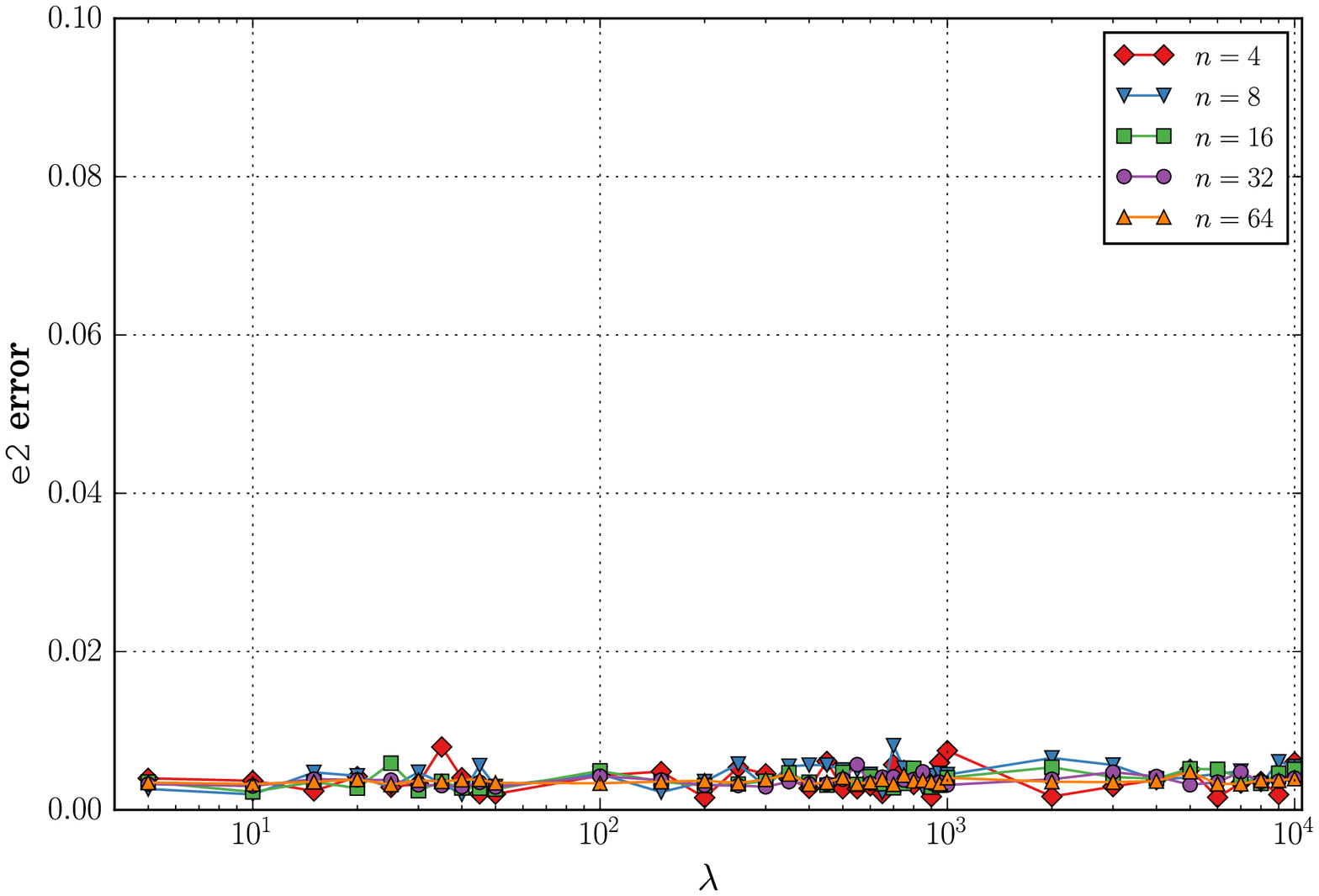, width=0.32\columnwidth} \\
\hline
(H-3) & \\
\epsfig{file=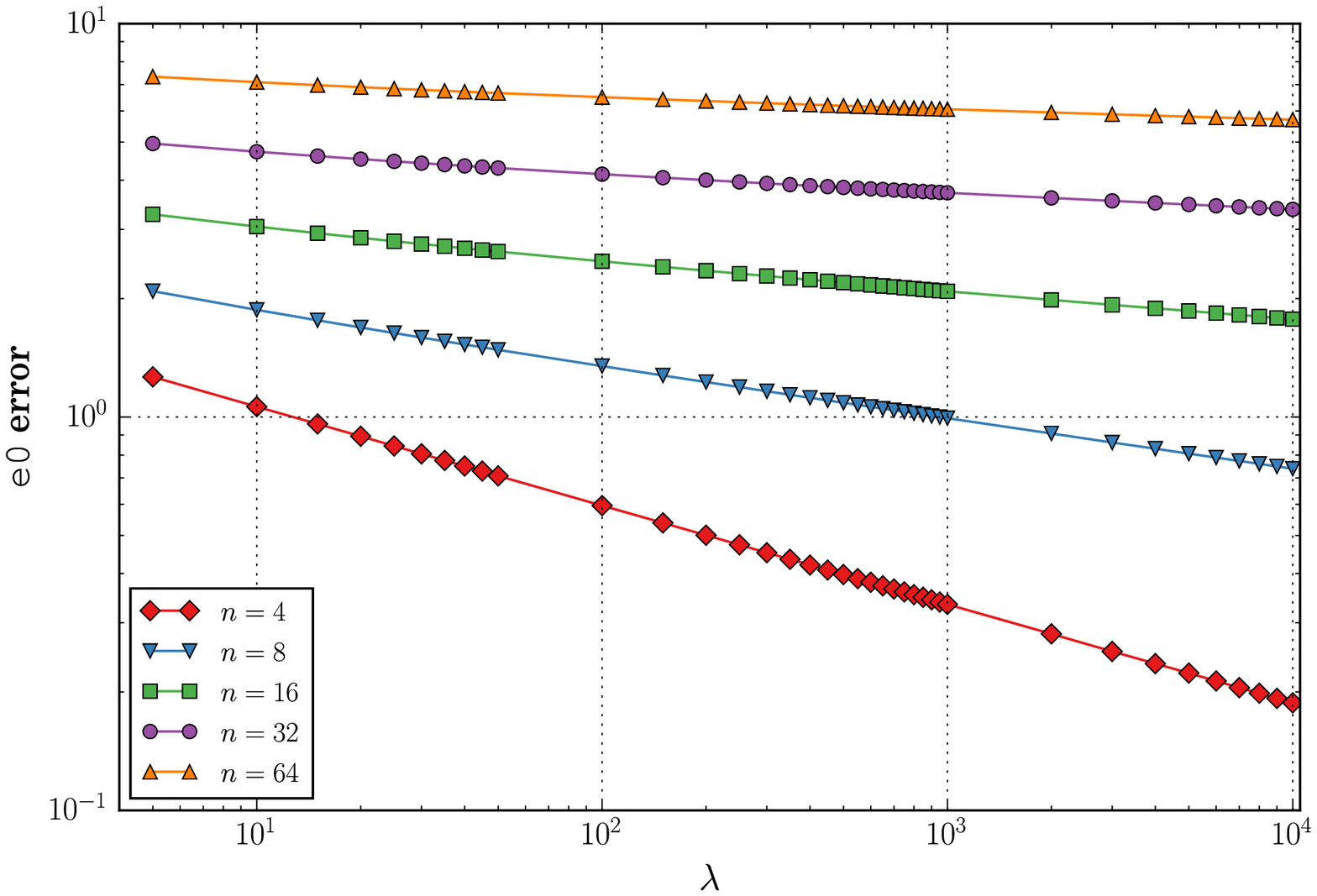, width=0.32\columnwidth} & \epsfig{file=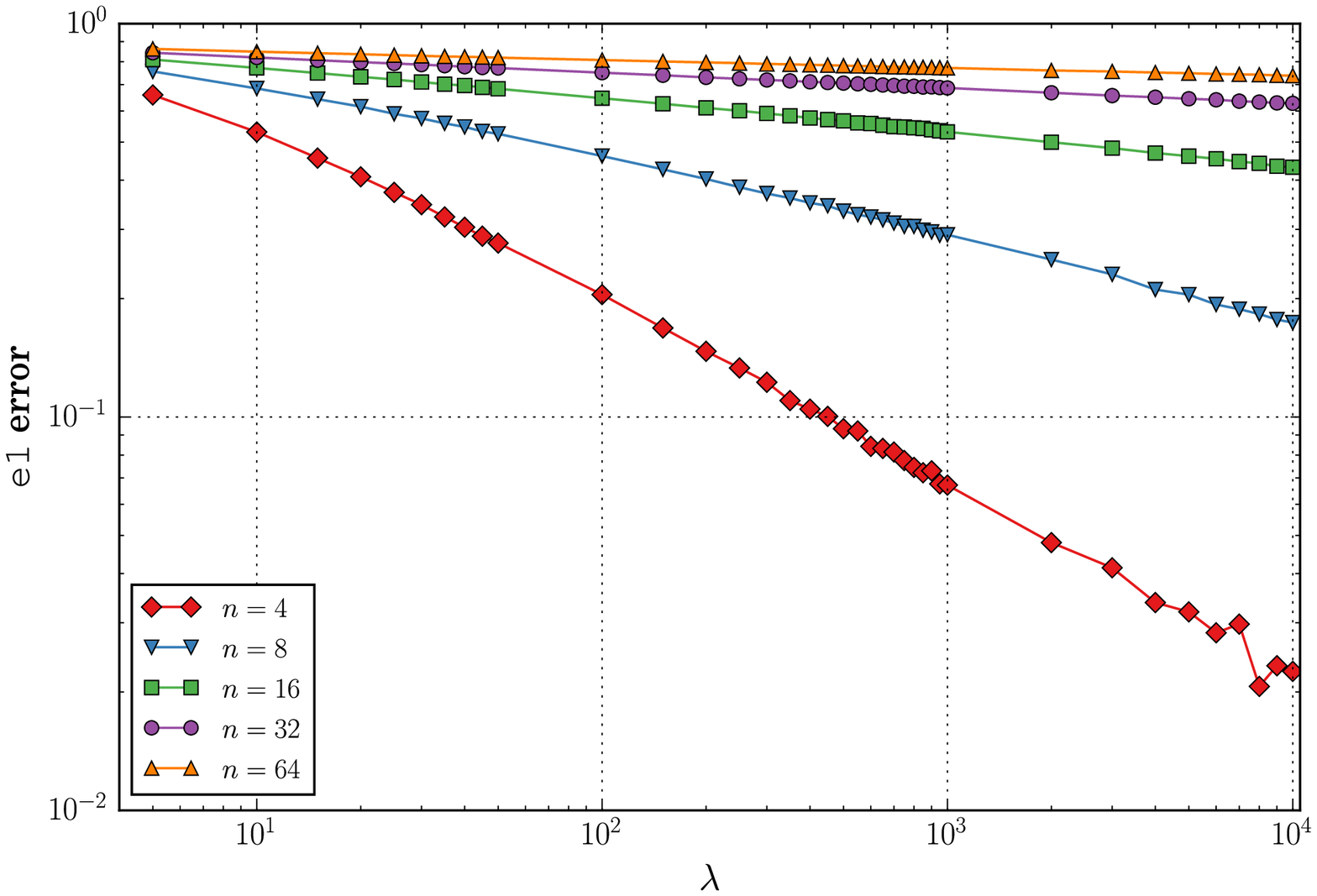, width=0.32\columnwidth} & \epsfig{file=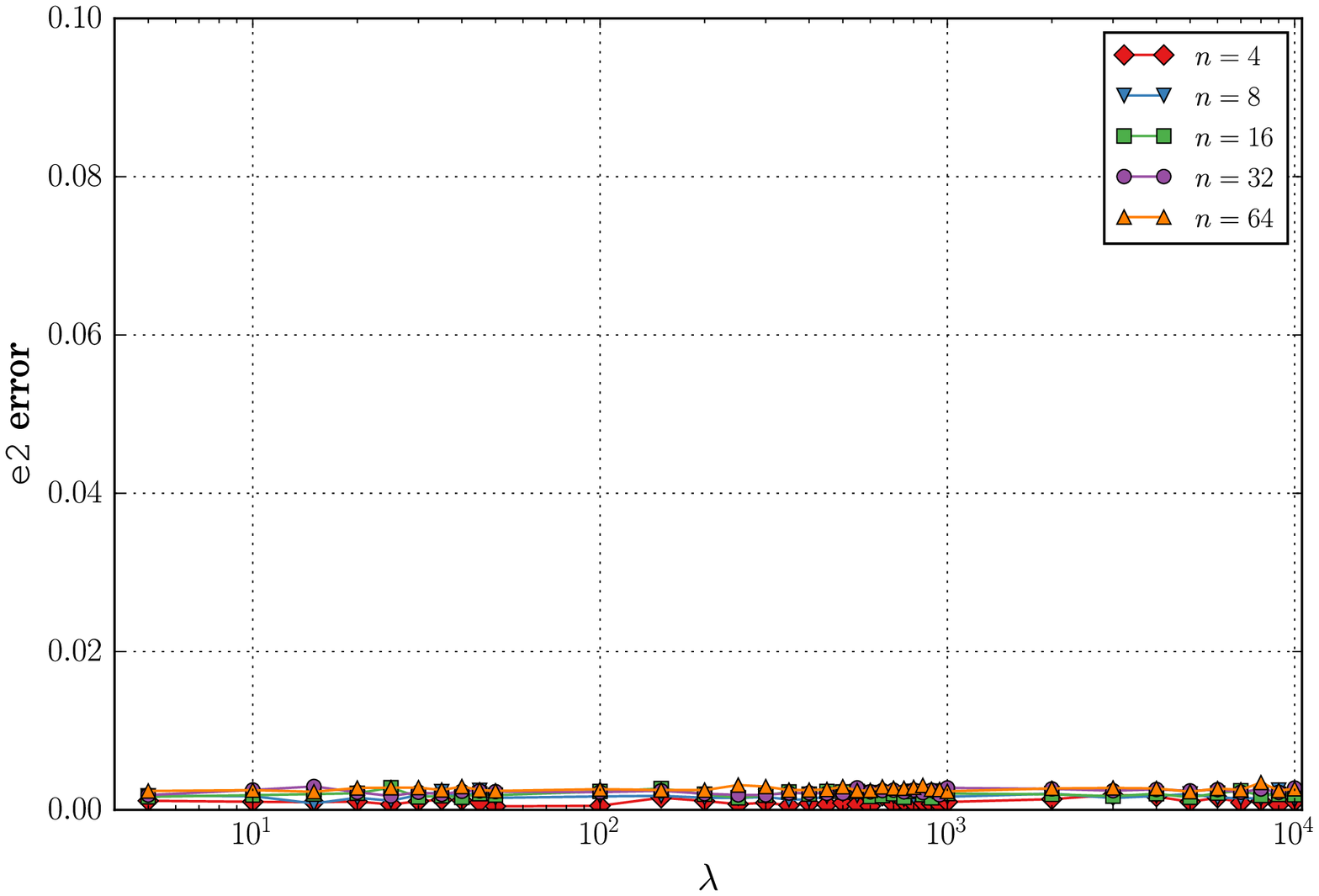, width=0.32\columnwidth} \\
\hline
\end{tabular}
\caption{Numerical corroboration of Theorem~\ref{thm:generalInverseRelation} evaluating the error measures (e-0)-(e-2) as a function of increasing $\lambda$ on (H-1) [TOP], (H-2) [MIDDLE], and (H-3) [BOTTOM] with conditioning $c=10$. The sampling was carried out over $N_{\texttt{iter}}=10^6$ iterations.
All axes are logarithmically scaled except for the $y$-axes in (e-2).
 Both (e-0) and (e-1) exhibit exponential patterns of decay as $\lambda$ increases. The (e-2) measure is practically vanished: it lies within the noise regime due to the separable (non-rotated) nature of all three landscapes. \label{fig:H1H2}}
\end{figure*}
\begin{figure*}
\begin{tabular}{ c  c  c}
{\Large (e-0)} & {\Large (e-1)} & {\Large (e-2)}\\
\hline
(H-4) & & \\
\epsfig{file=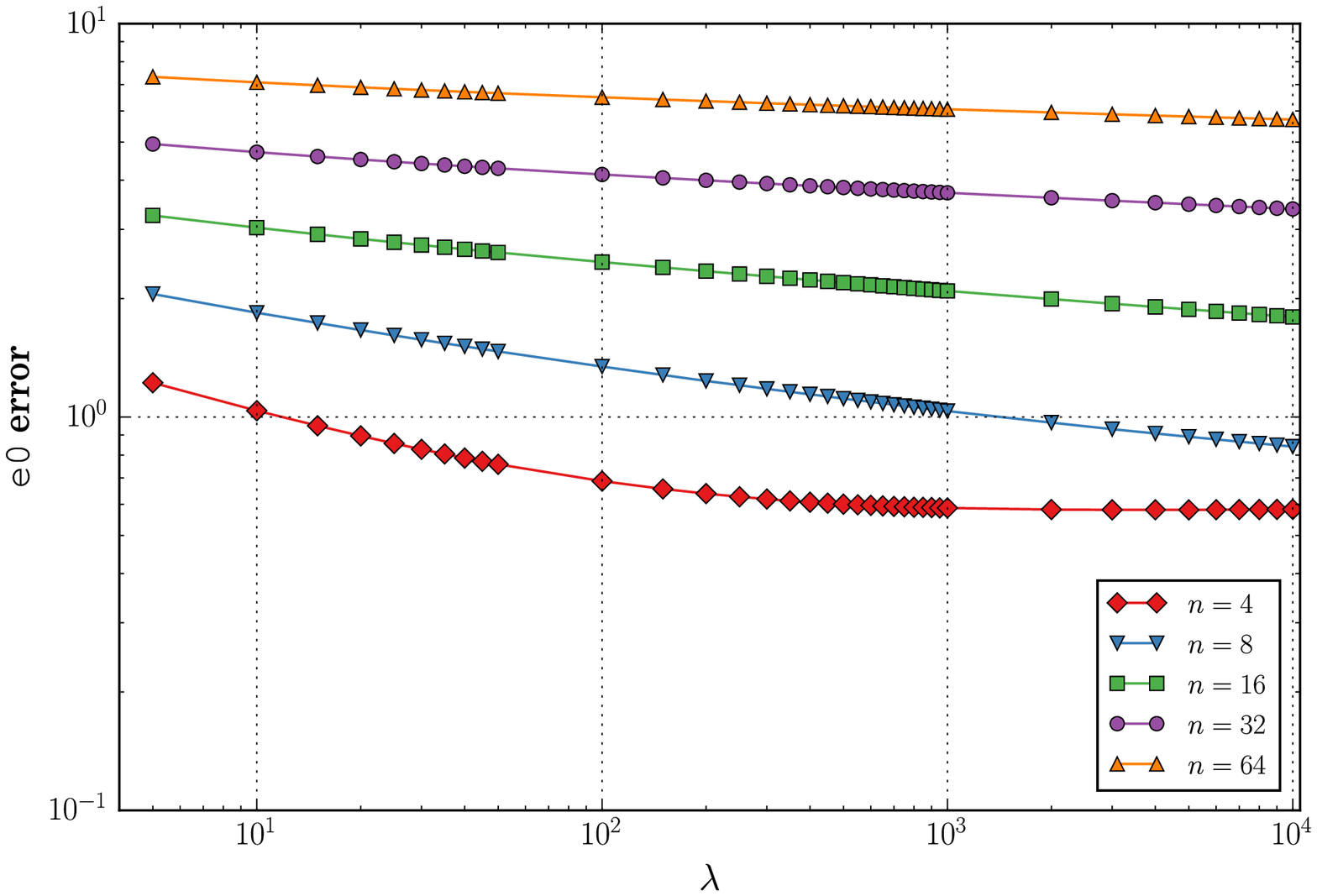, width=0.32\columnwidth} & \epsfig{file=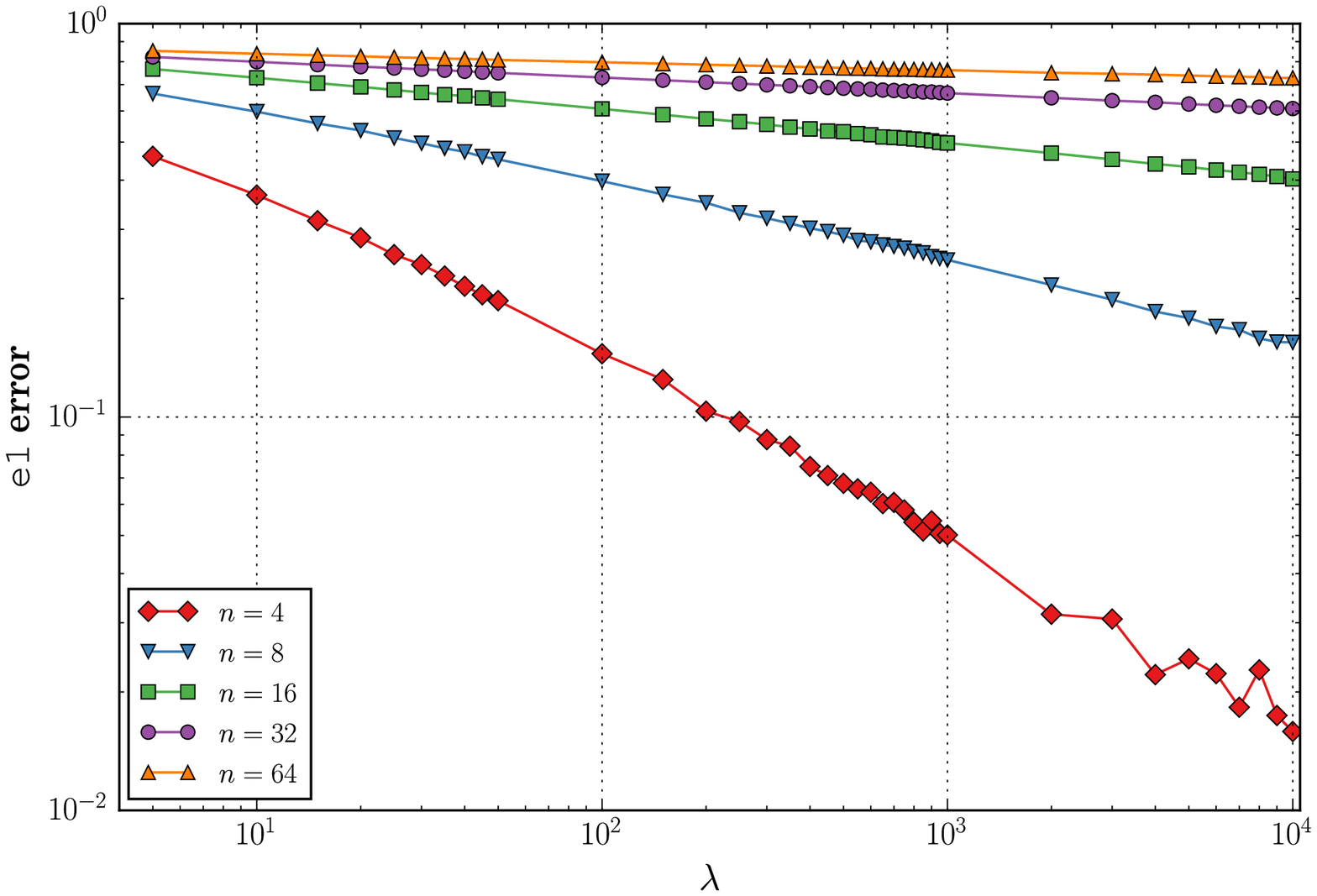, width=0.32\columnwidth} & \epsfig{file=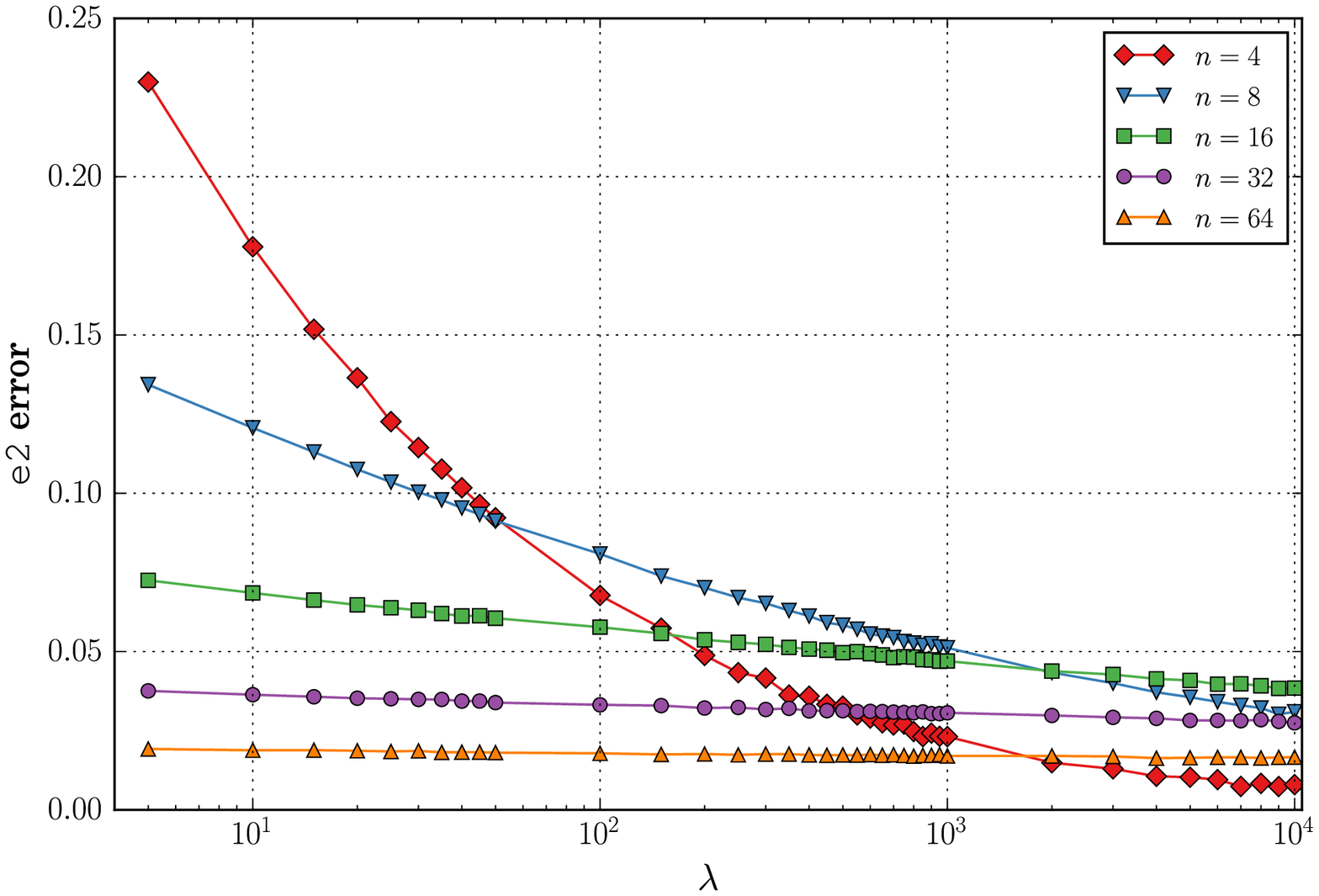, width=0.32\columnwidth} \\
\hline
(H-5) & \\
\epsfig{file=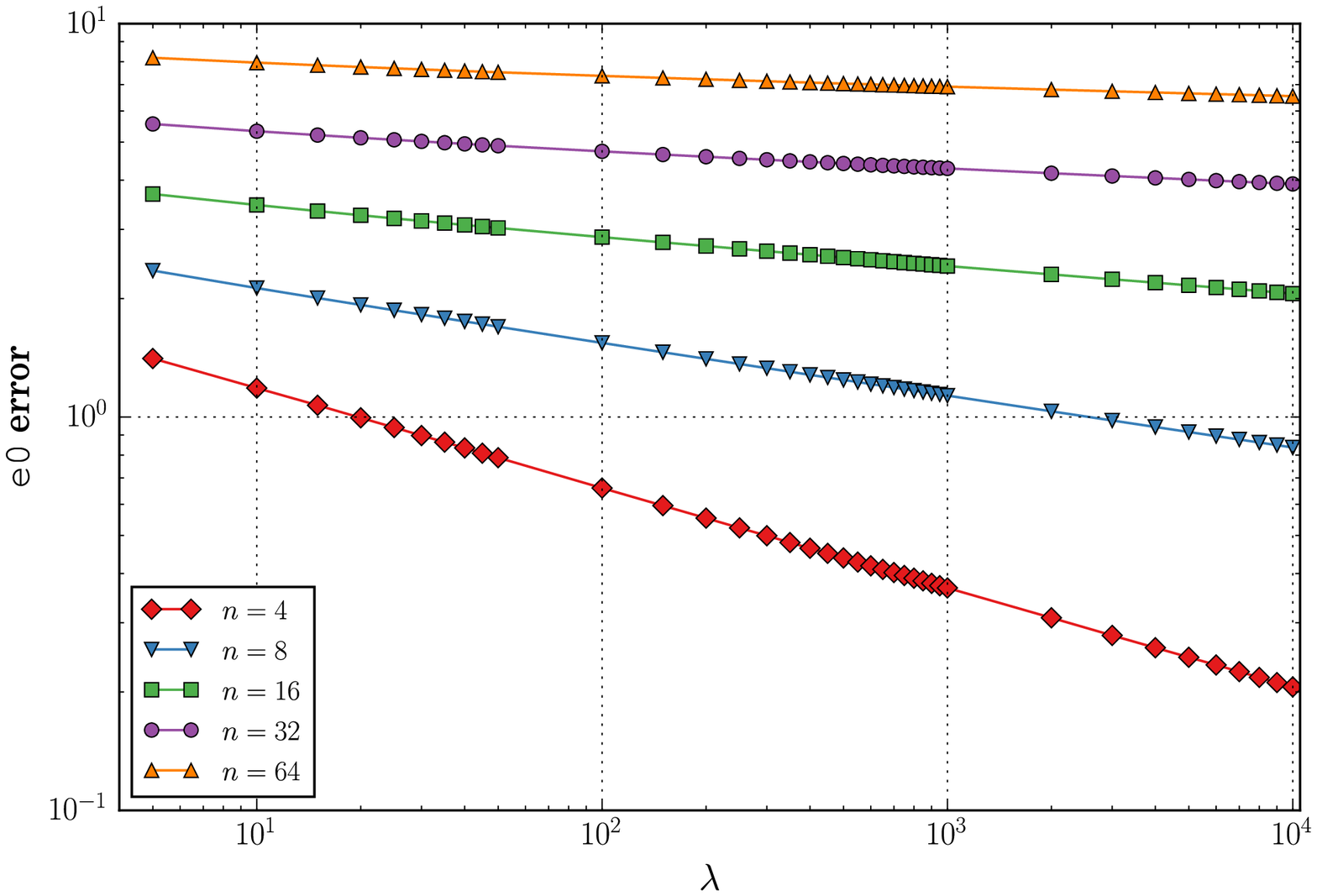, width=0.32\columnwidth} & \epsfig{file=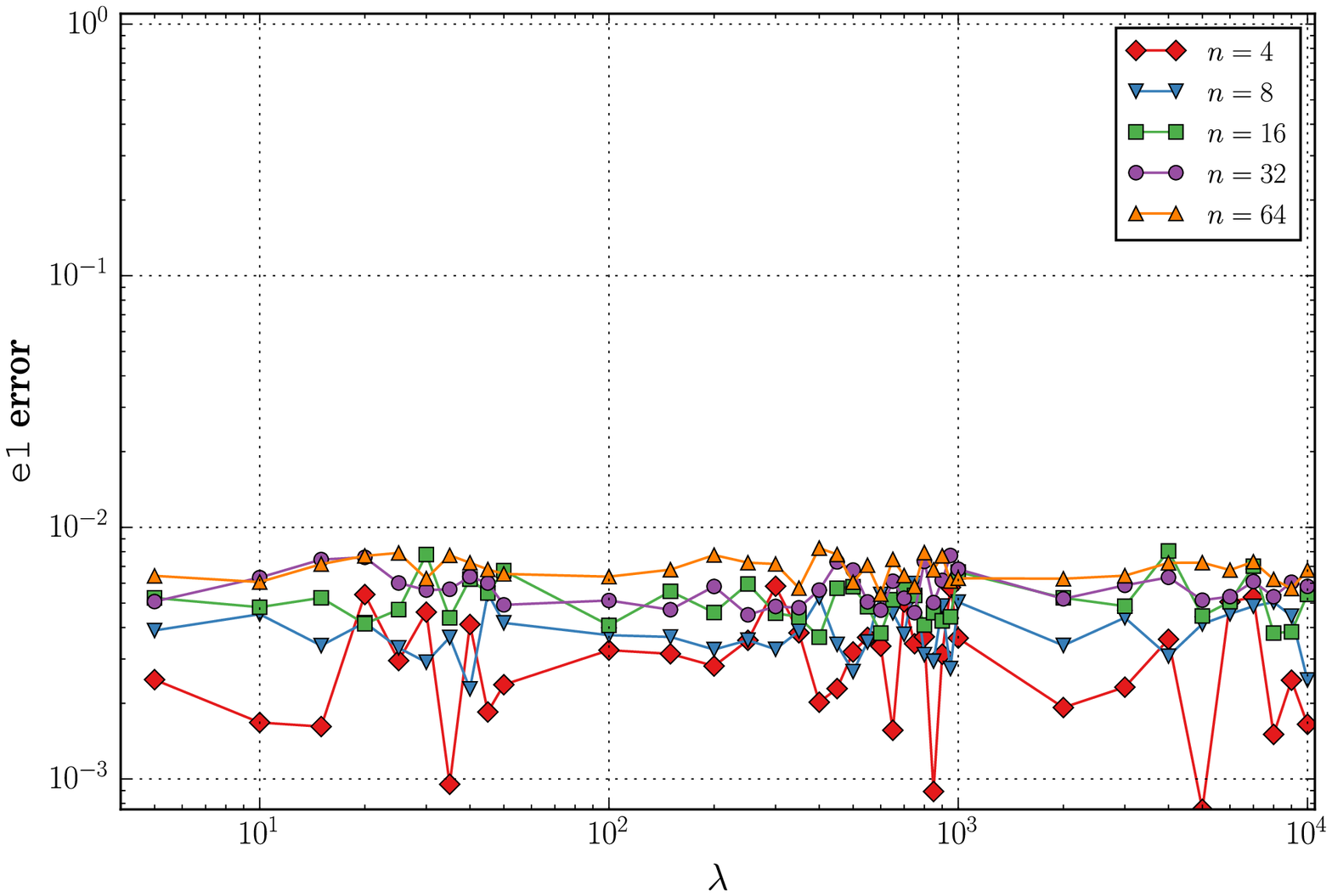, width=0.32\columnwidth} & \epsfig{file=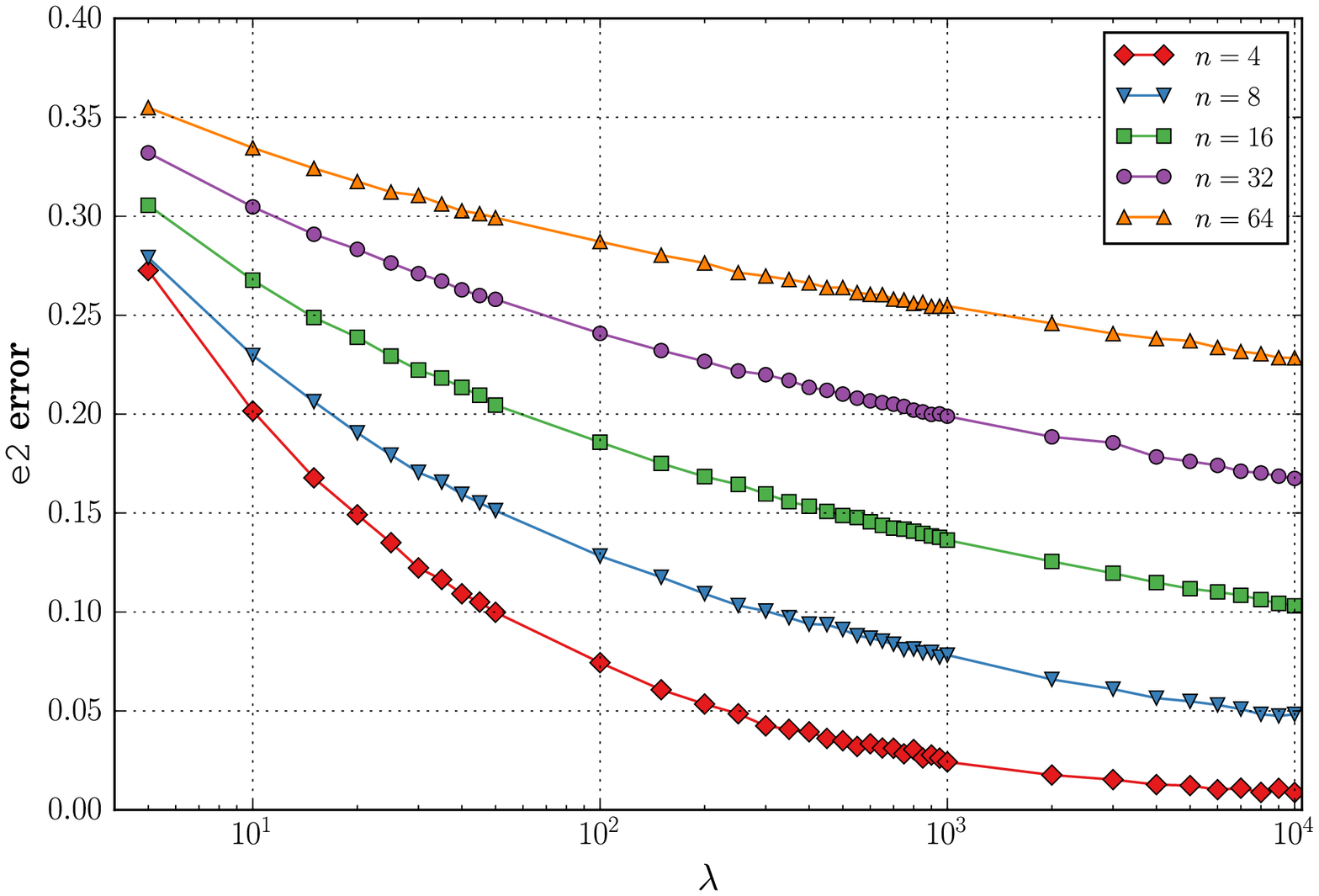, width=0.32\columnwidth} \\
\hline
\end{tabular}
\caption{Numerical corroboration of Theorem~\ref{thm:generalInverseRelation} evaluating the error measures (e-0)-(e-2) as a function of increasing $\lambda$ on (H-4) [TOP] and (H-5) [BOTTOM] with conditioning $c=10$. The sampling was carried out over $N_{\texttt{iter}}=10^6$ iterations. 
All axes are logarithmically scaled except for the $y$-axes in (e-2).
The (e-2) measure exponentially decreases as $\lambda$ increases for both test-case.
For (H-4), both (e-0) and (e-1) exhibit exponential patterns of decay as $\lambda$ increases.
For (H-5), unlike the other test-cases, the (e-1) measure is practically vanishing (it lies within the noise regime), as was noted above.
\label{fig:H3H4}}
\end{figure*}
\begin{figure*}
\begin{tabular}{ c  c  c}
 {\Large $n=4$} & {\Large $n=8$} & {\Large $n=32$} \\
\hline
(e-1) & & \\
\epsfig{file=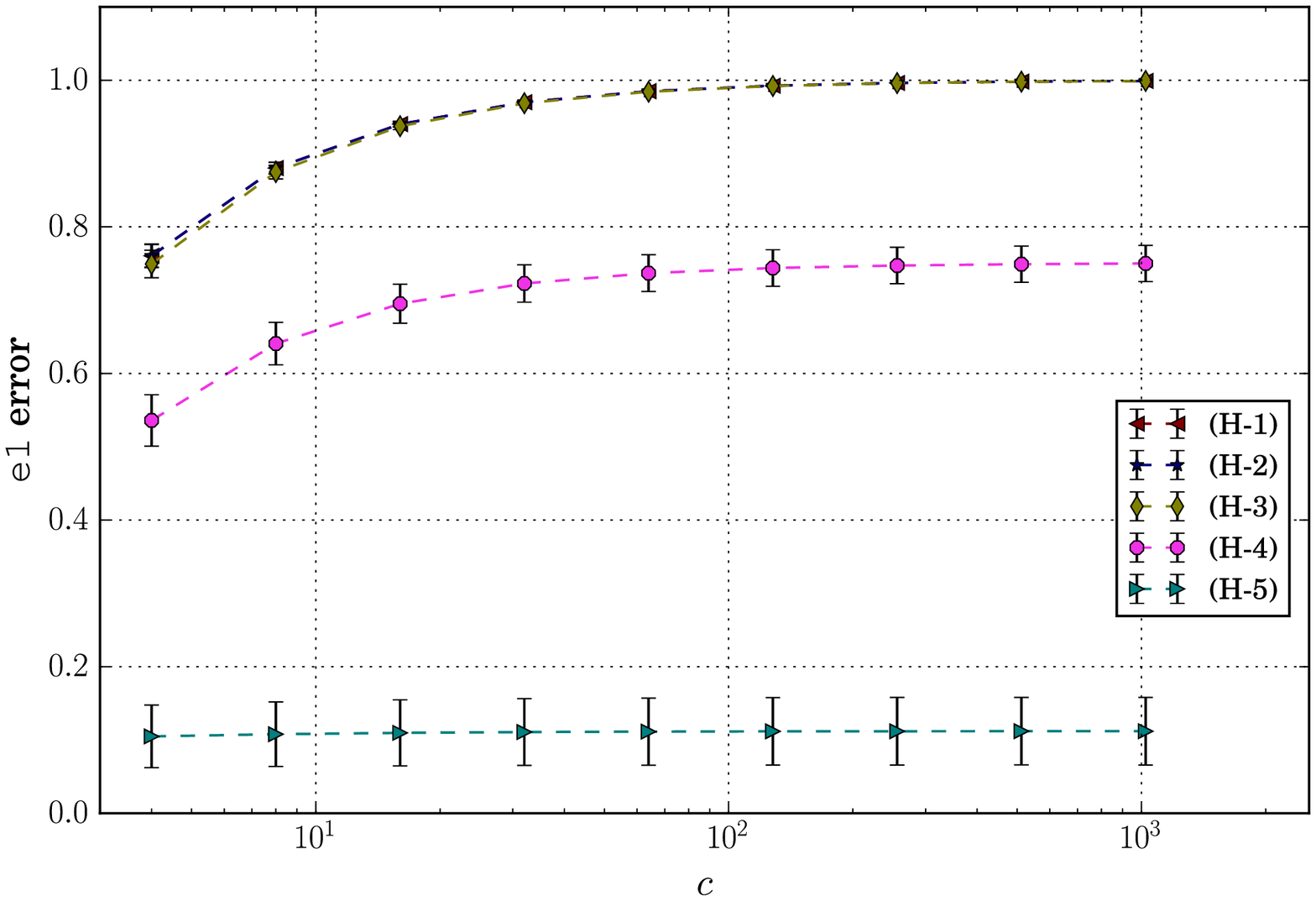, width=0.32\columnwidth} & \epsfig{file=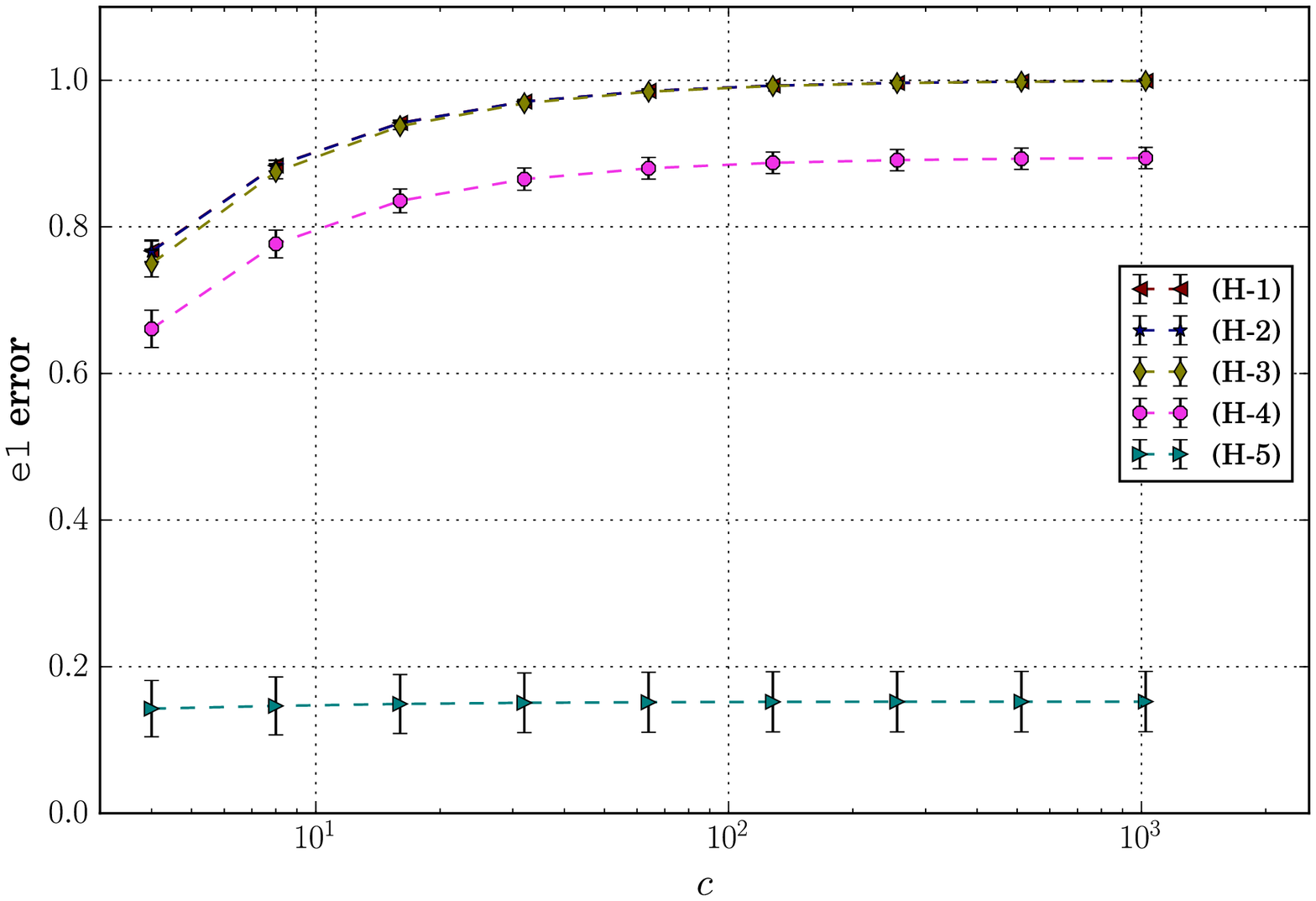, width=0.32\columnwidth} & \epsfig{file=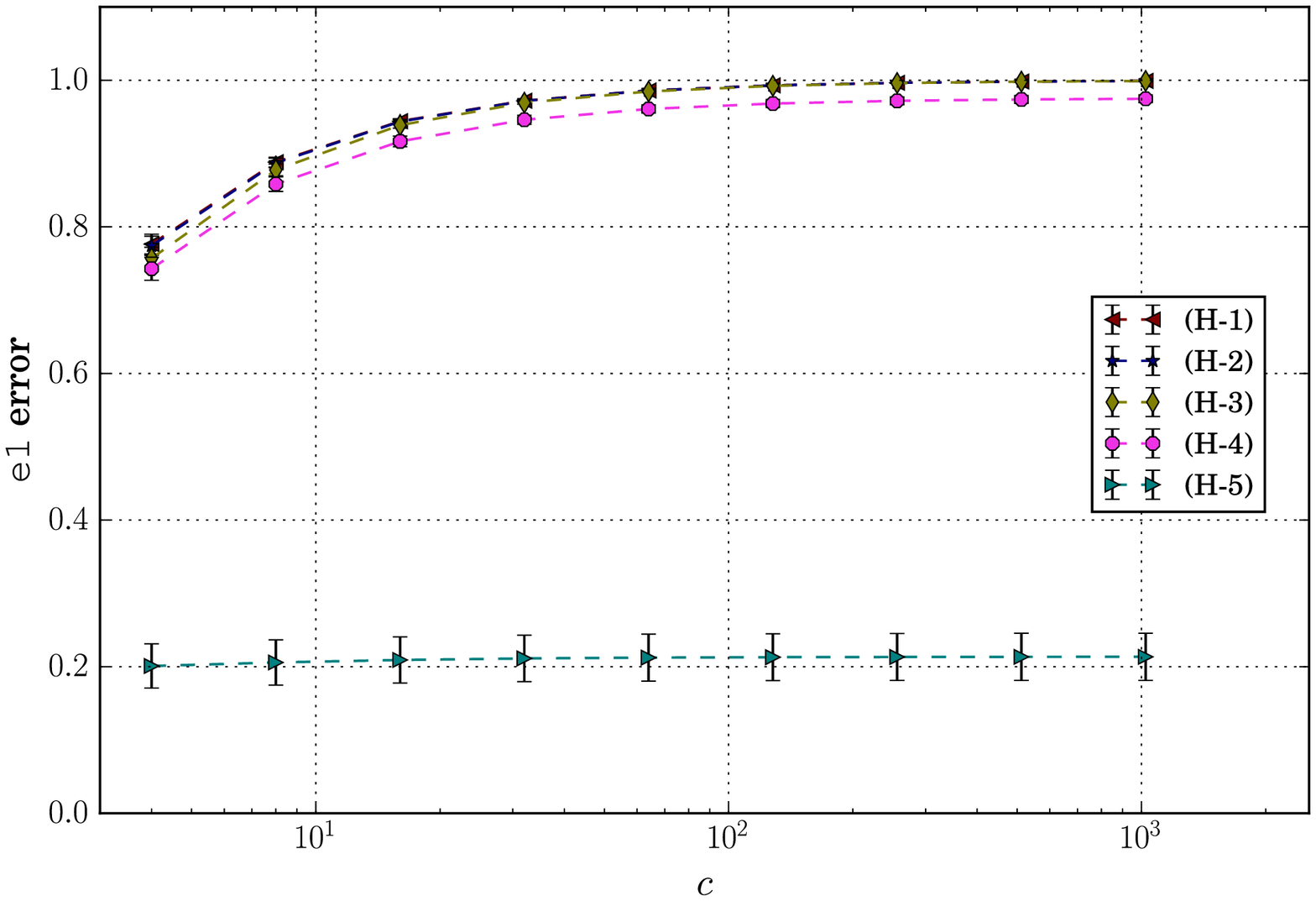, width=0.32\columnwidth} \\
\hline
(e-2) & & \\
\epsfig{file=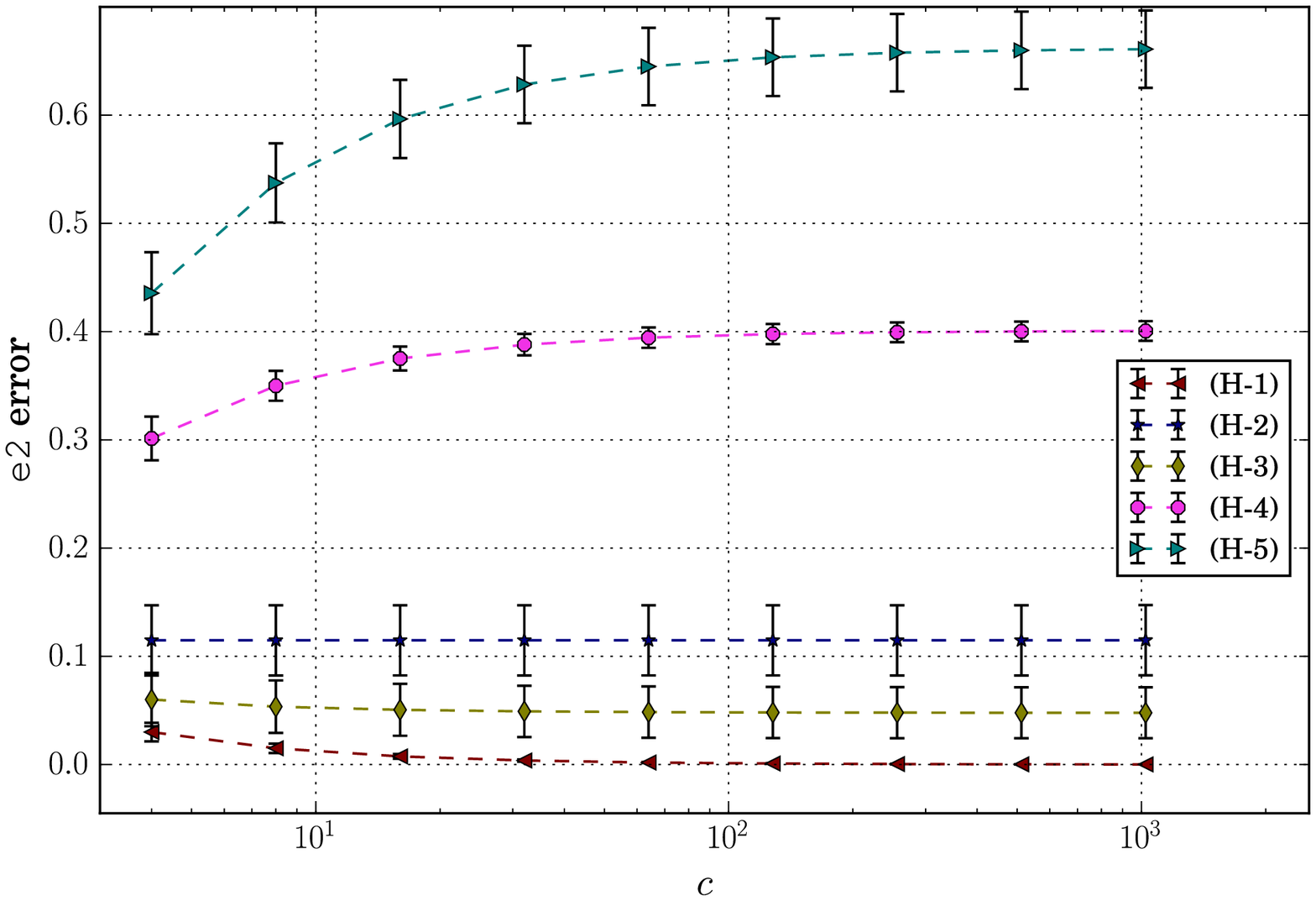, width=0.32\columnwidth} & \epsfig{file=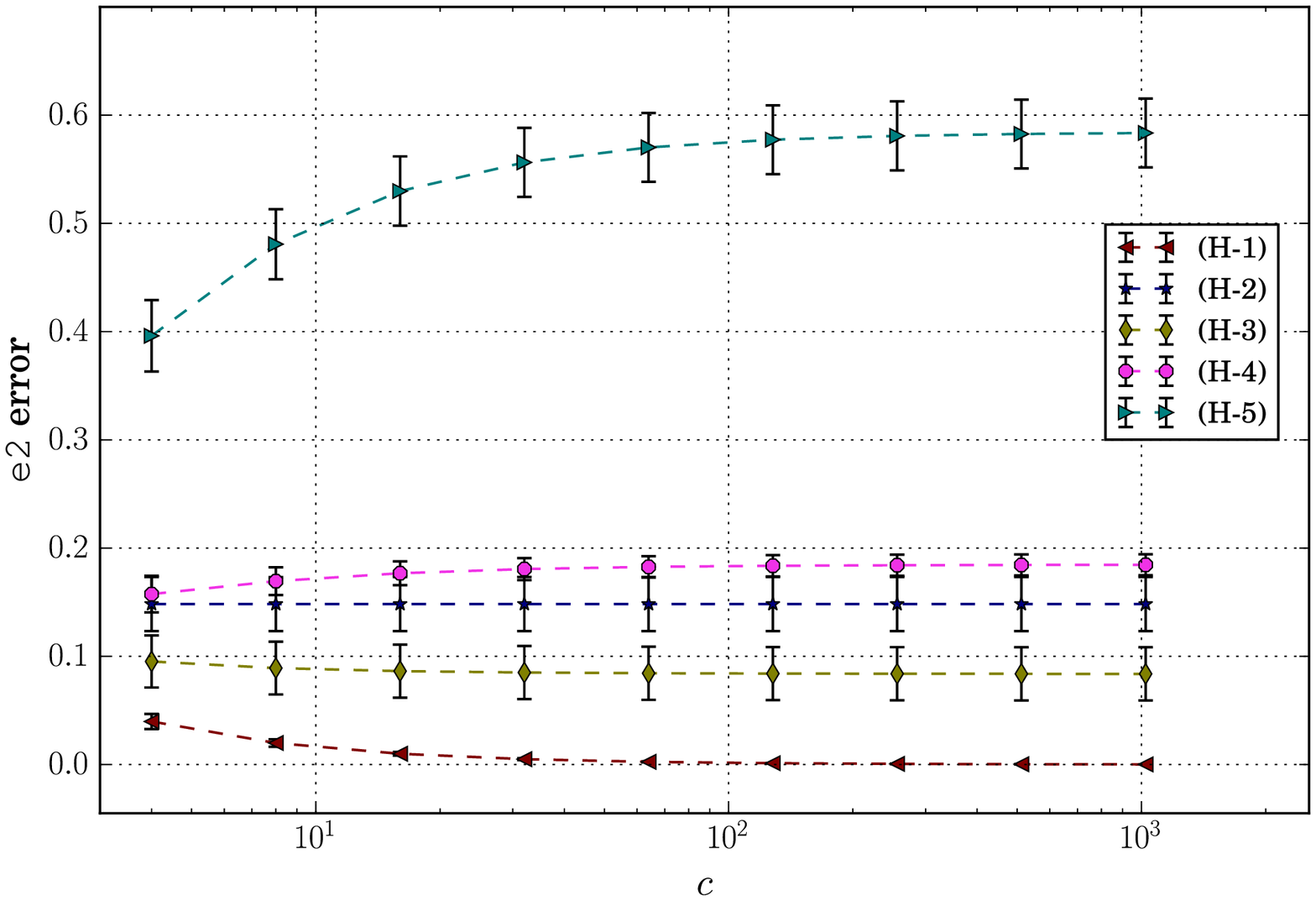, width=0.32\columnwidth} & \epsfig{file=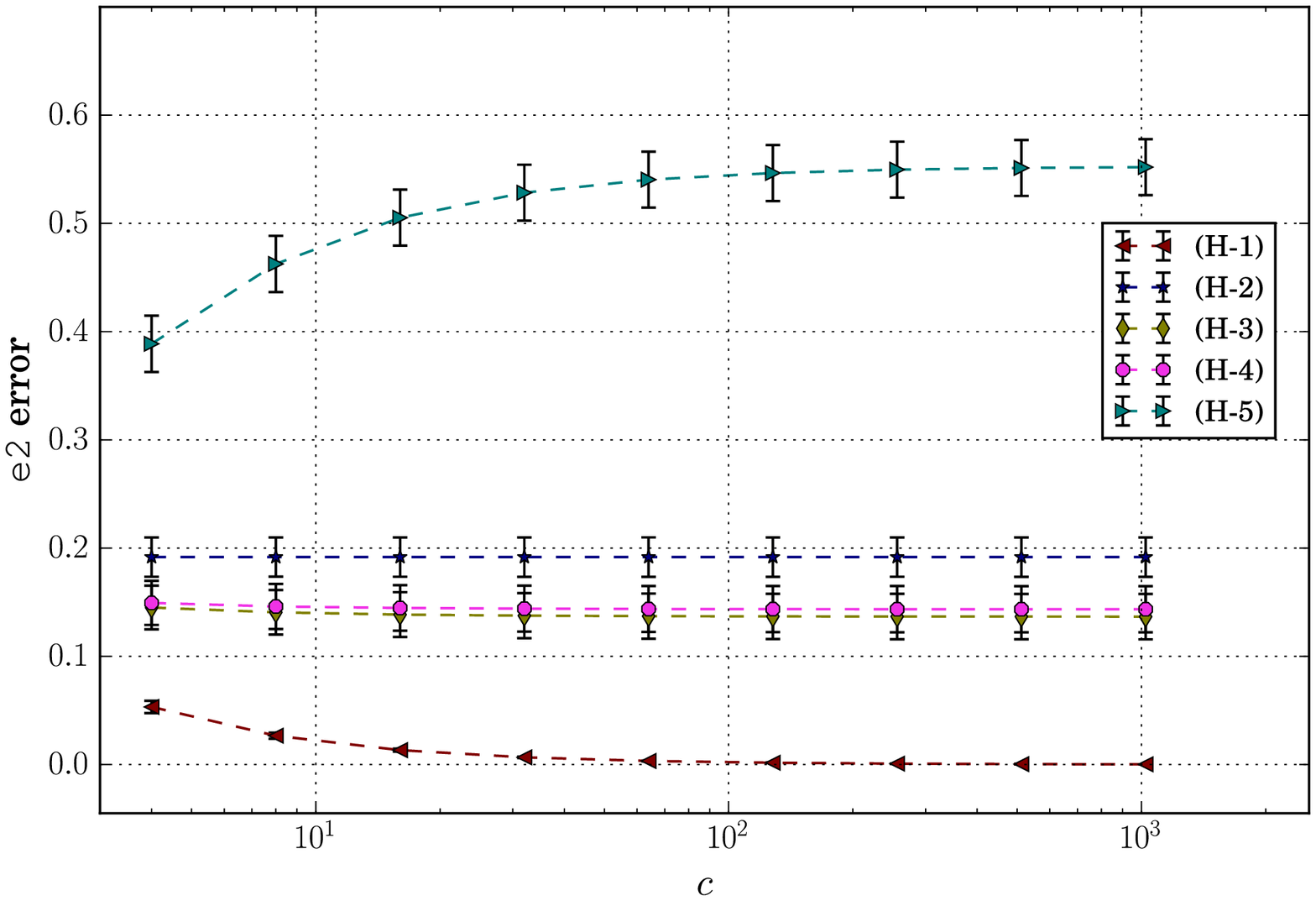, width=0.32\columnwidth} \\
\hline
\end{tabular}
\caption{Assessing the (e-1) [TOP] and (e-2) [BOTTOM] measures for the ``no-cumulation'' reference case.
A statistical test generated $10^6$ matrices of the form $\mathbf{C}^{\texttt{pert}}:=\mathbf{I}+\mathbf{E}$ as described above.
Each random matrix was evaluated on (e-1) and (e-2) against (H-1)-(H-5). The mean values are depicted for $n=4$ [LEFT], $n=8$ [MIDDLE], and $n=32$ [RIGHT] over conditioning in the exponential range $c\in [2^2,2^3,\ldots 2^{10}]$. Vertical error-bars are depicted per test using standard-deviation values. Importantly, (e-1) consistently exhibits low values for (H-5), in all dimensions and over all conditioning, explaining the observations reported in Figure \ref{fig:H3H4}.\label{fig:noLearning}}
\end{figure*}

\subsubsection{The Inverse Relation subject to Translations from the Optimum}\label{sec:shifting}
Since Theorem~\ref{thm:generalInverseRelation} extends a previous result concerning the near-optimum special case \cite{Shir-Theory-foga17}, $\vec{x}_0=\vec{x}^{*}=\vec{0}$, we are also interested in exploring the impact of translating farther away from the optimum.
Additional experiments were run to investigate the impact of such translations on the inverse relation. 
In practice, we account for the effect of increasing the shift vector $\vec{a}$: the vectors $\tilde{\vec{a}}:=\left\{\vec{2},\vec{4},\vec{8} \right\}$ are utilized in \eqref{eqn:QSampling}, introducing a factor of $\left\{2\sqrt{n},~4\sqrt{n},~8\sqrt{n} \right\}$ Euclidean distance away in comparison to the main results reported herein (Figures \ref{fig:H1H2}-\ref{fig:H3H4}).
Numerical observations for (H-4) with conditioning $c=10$ are presented in Figures~\ref{fig:FARTHER}-\ref{fig:FARTHER2}, encompassing in addition the default case $\vec{a}=\vec{1}$, as well as the optimum-based sampling $\vec{a}_0=\vec{0}$. 
Per a given translation vector $\vec{a}$, it is evident that the decrease in both (e-1) and (e-2) is consistent with the previously observed trends. That is, a decay toward zero as $\lambda$ increases, which becomes slower as the dimension $n$ increases.
When comparing amongst the translation degrees per dimension, the decay becomes slower as the sampling location translates farther from the optimum, as expected.
According to Theorem~\ref{thm:generalInverseRelation}, the error rates will necessarily vanish for any translation, yet large translations would necessitate in practice very large $\lambda$ values.

\begin{figure*}
\begin{tabular}{c  c  c}
{\Large $n=4$} & {\Large $n=8$} & {\Large $n=16$}  \\
\hline
(e-1) & & \\
\epsfig{file=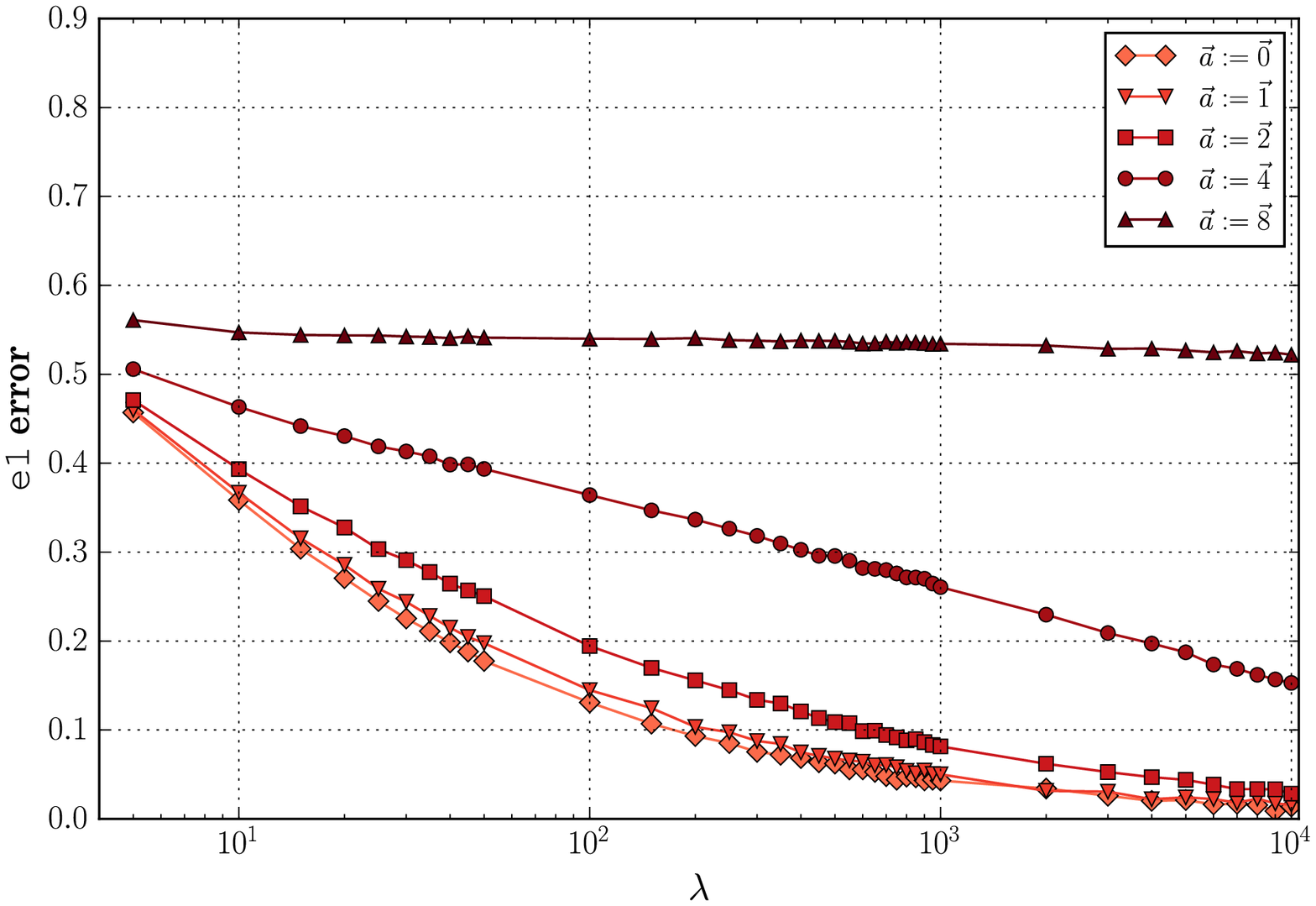, width=0.32\columnwidth} & \epsfig{file=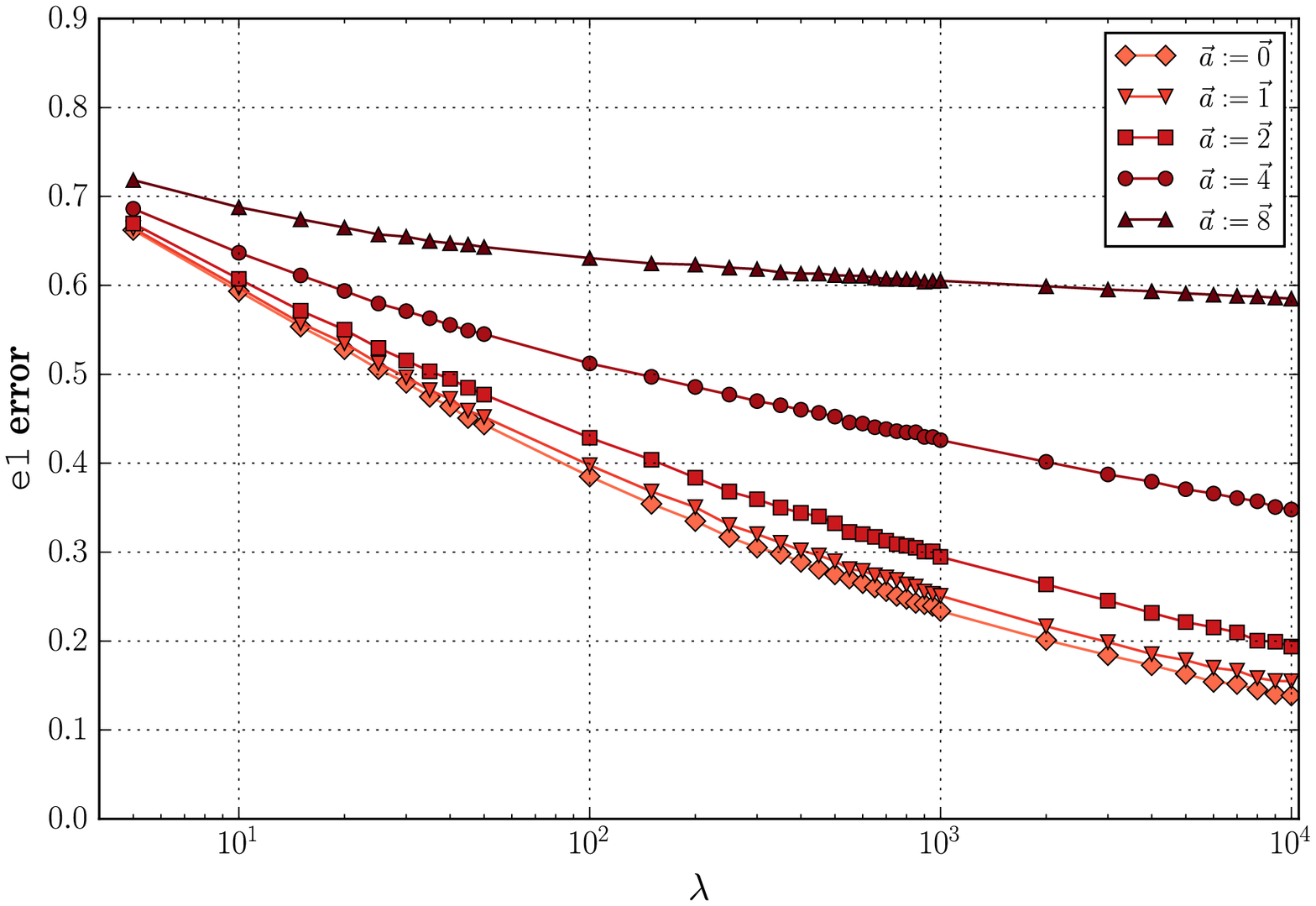, width=0.32\columnwidth} & \epsfig{file=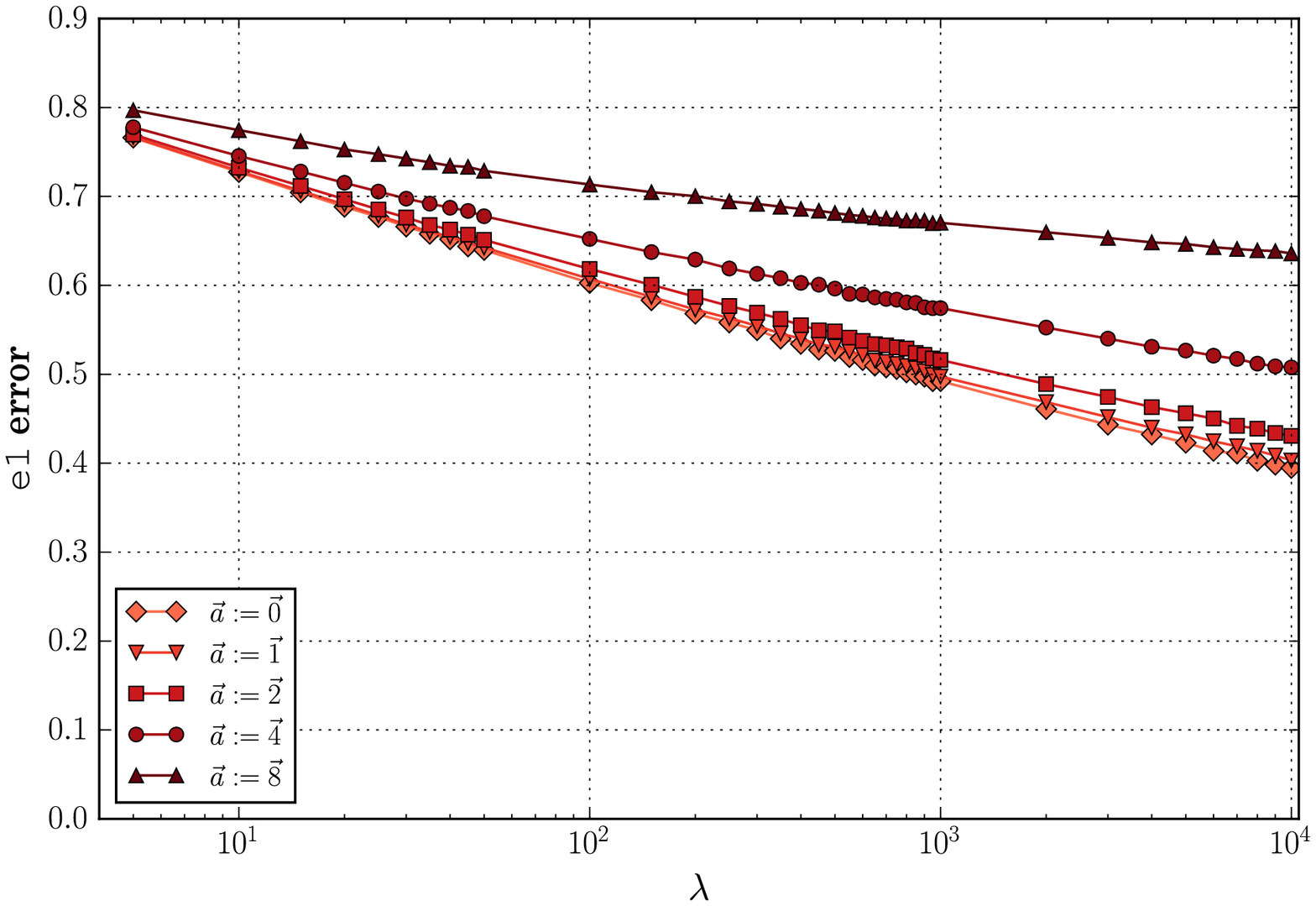, width=0.32\columnwidth}\\
\hline
(e-2) & & \\
\epsfig{file=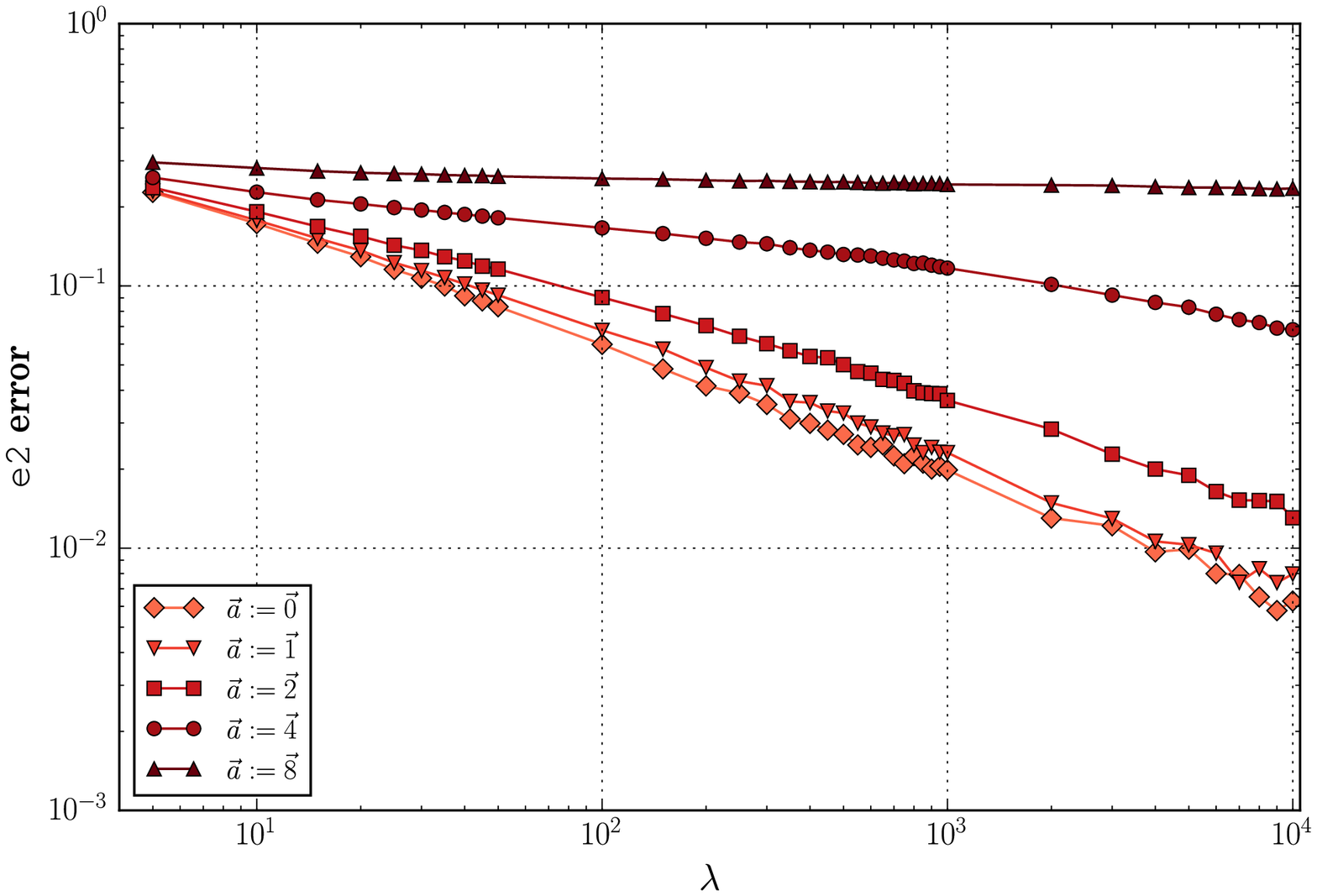, width=0.32\columnwidth} & \epsfig{file=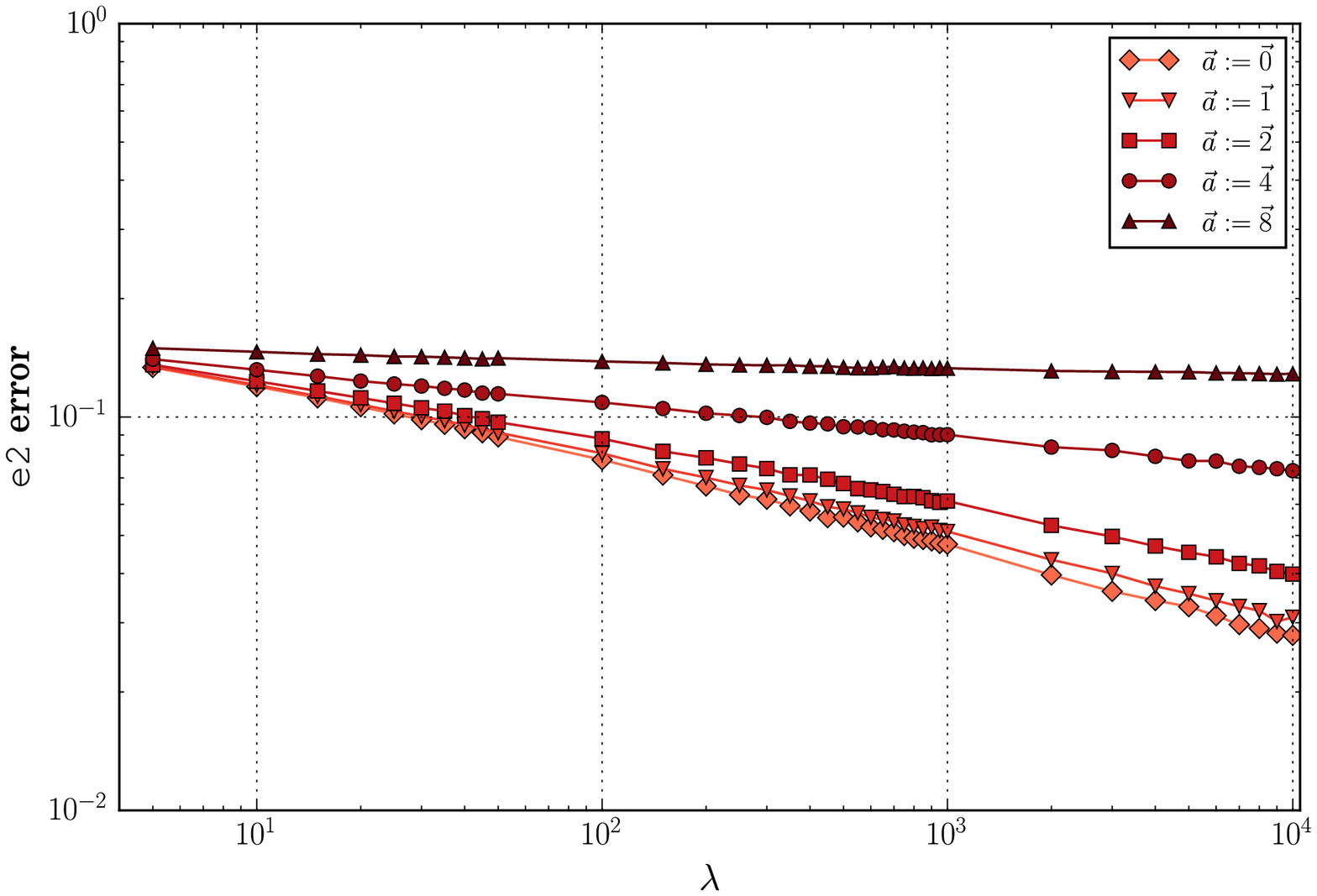, width=0.32\columnwidth} & \epsfig{file=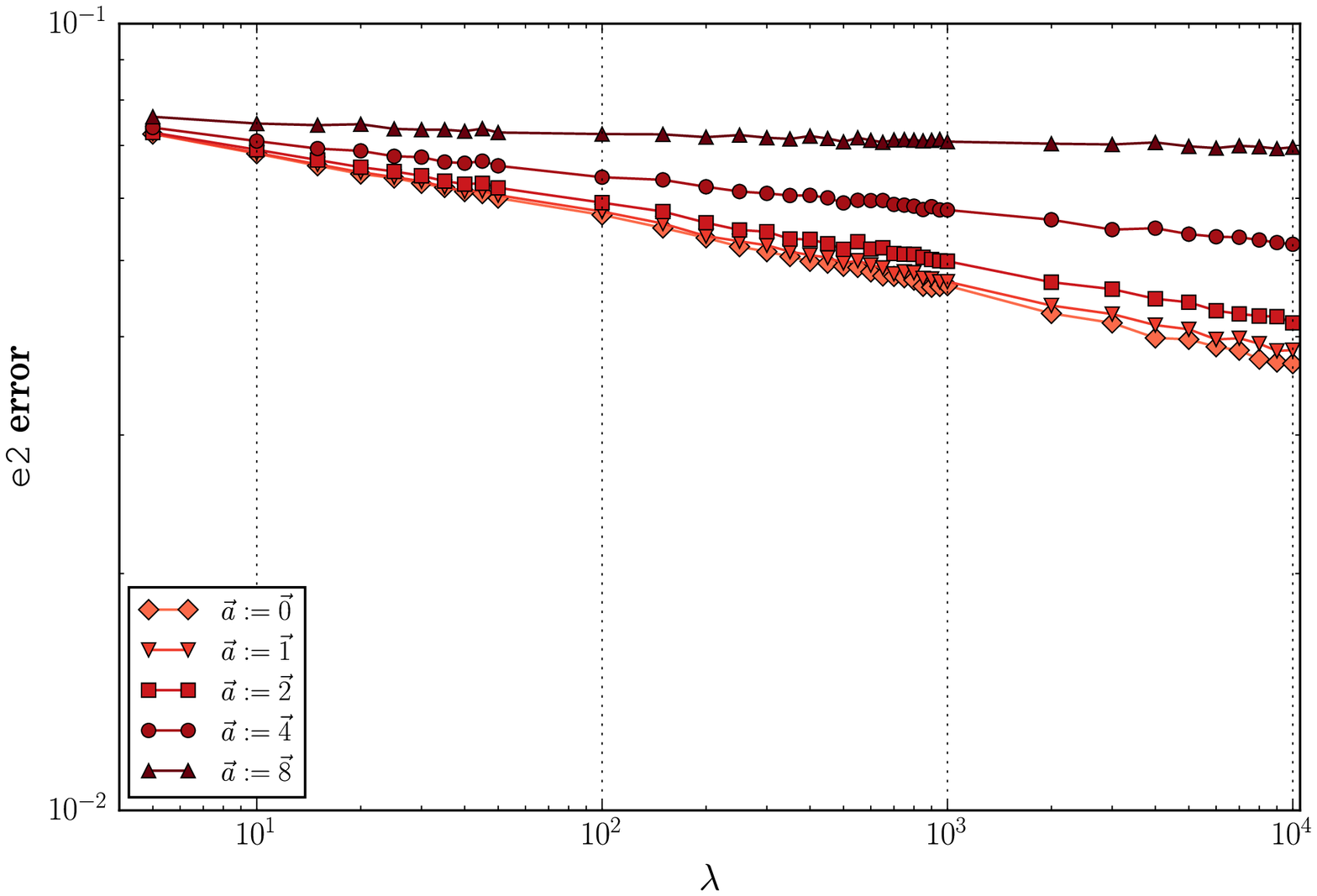, width=0.32\columnwidth}\\
\hline
\end{tabular}
\caption{Investigating the translation of the sampling location on (H-4) with conditioning $c=10$ over $N_{\texttt{iter}}=10^6$ iterations.
The error measures are depicted for $n=4$ [LEFT], $n=8$ [MIDDLE], and $n=16$ [RIGHT] over various instantiations of a shifting vector $\vec{a}$. 
All axes are logarithmically scaled except for the $y$-axes in (e-2).
\label{fig:FARTHER}}
\end{figure*}

\begin{figure*}
\begin{center}
\begin{tabular}{c  c}
{\Large $n=32$} & {\Large $n=64$} \\
\hline
(e-1) & \\
\epsfig{file=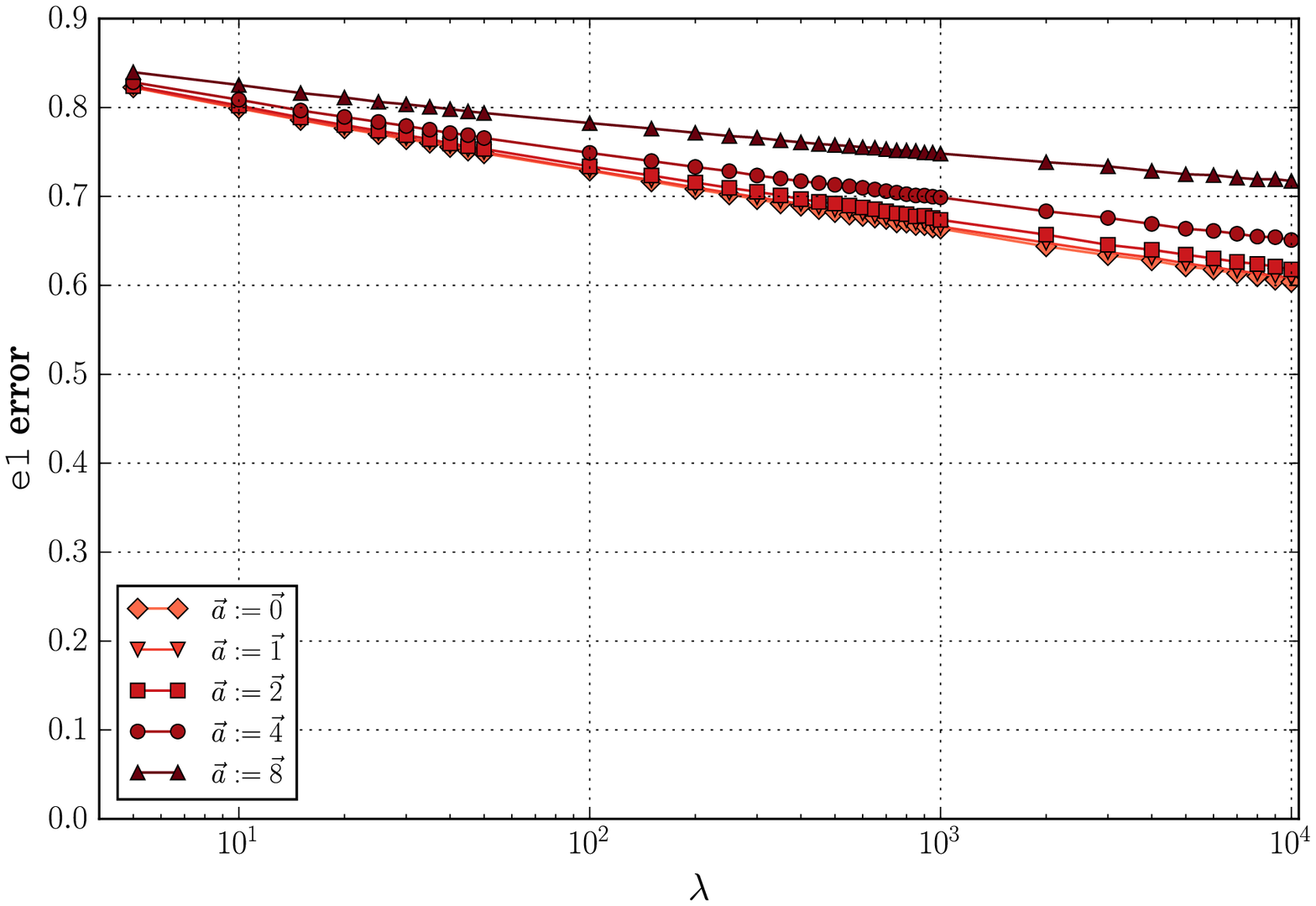, width=0.48\columnwidth} & \epsfig{file=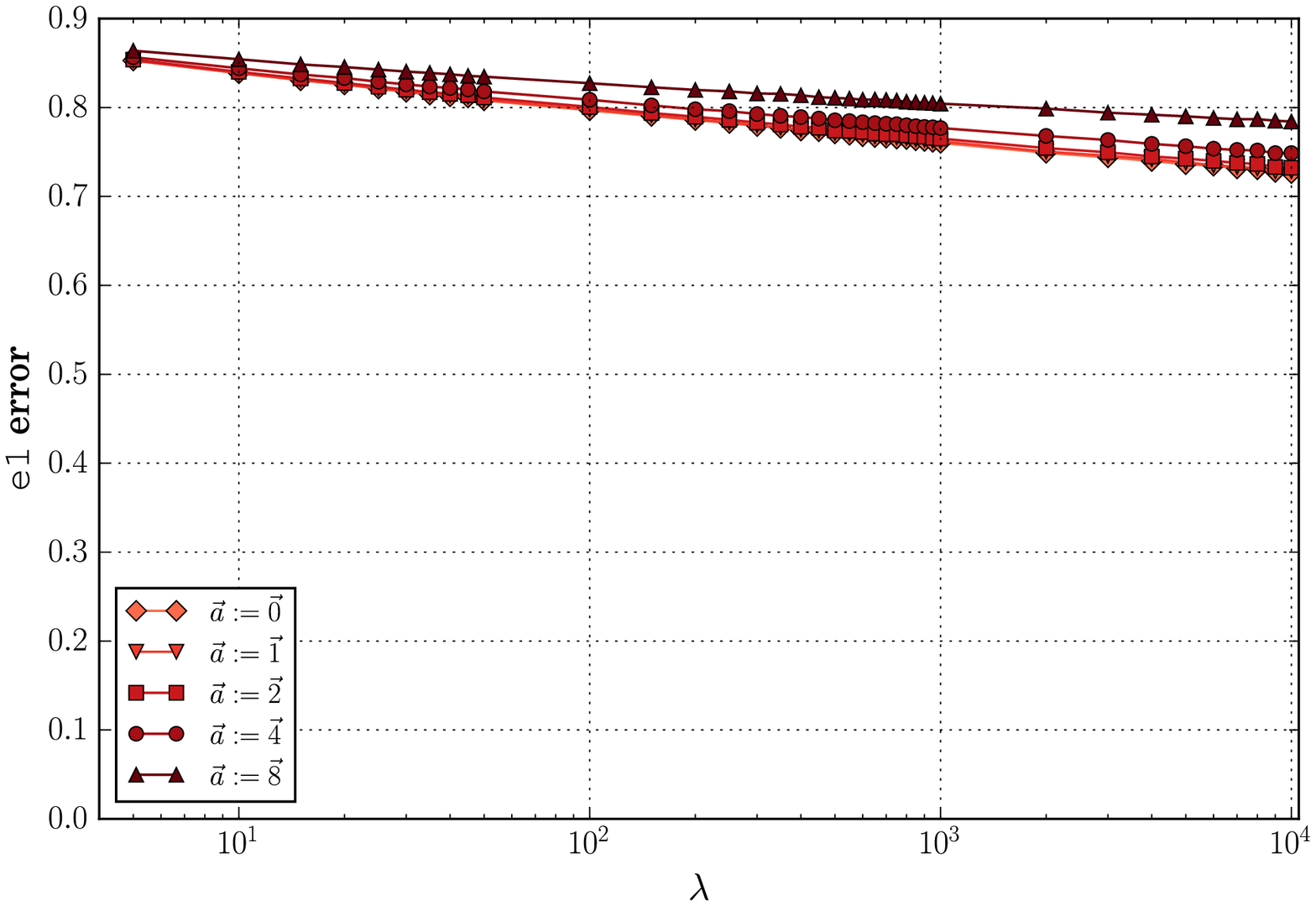, width=0.48\columnwidth} \\
\hline
(e-2) & \\
\epsfig{file=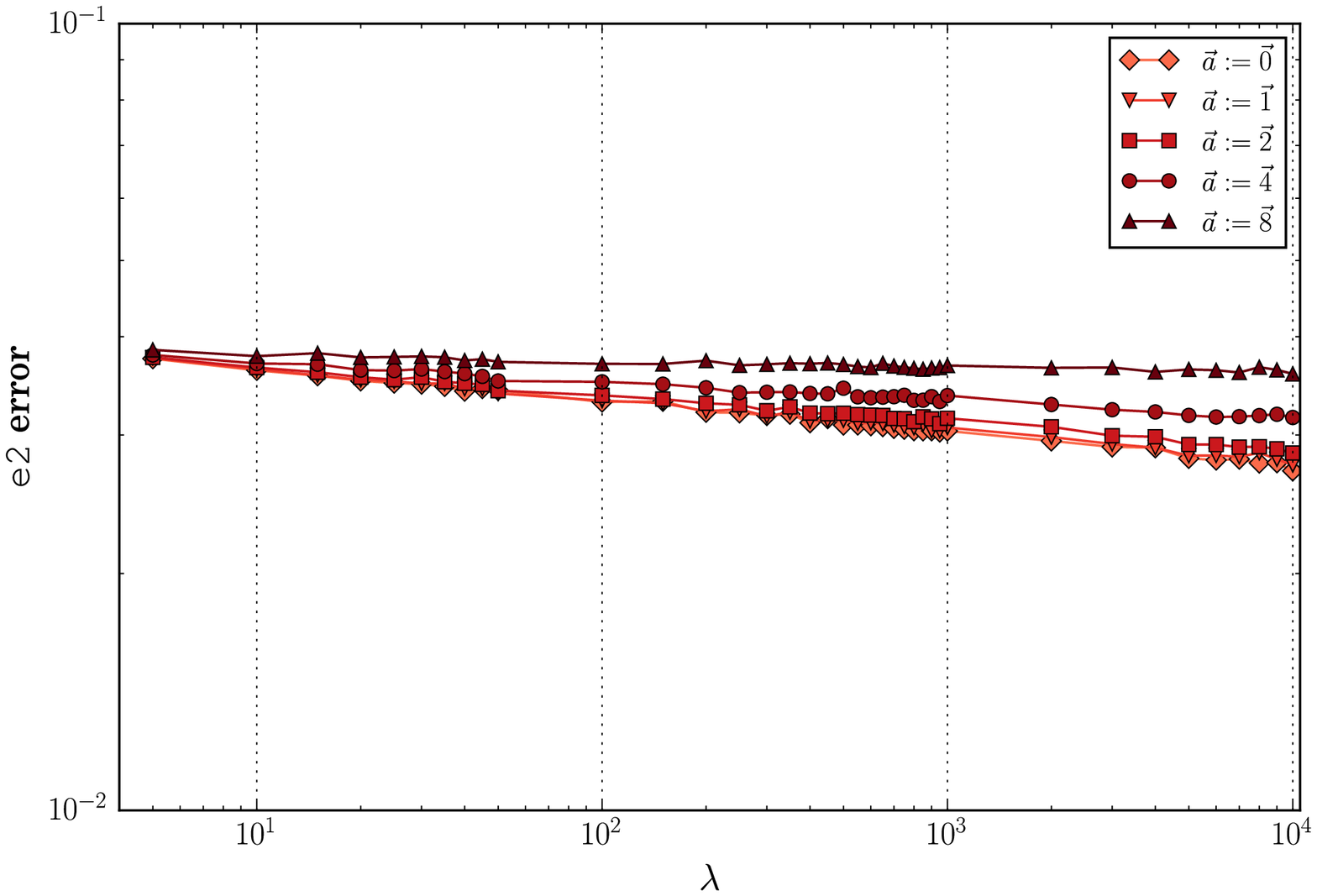, width=0.48\columnwidth} & \epsfig{file=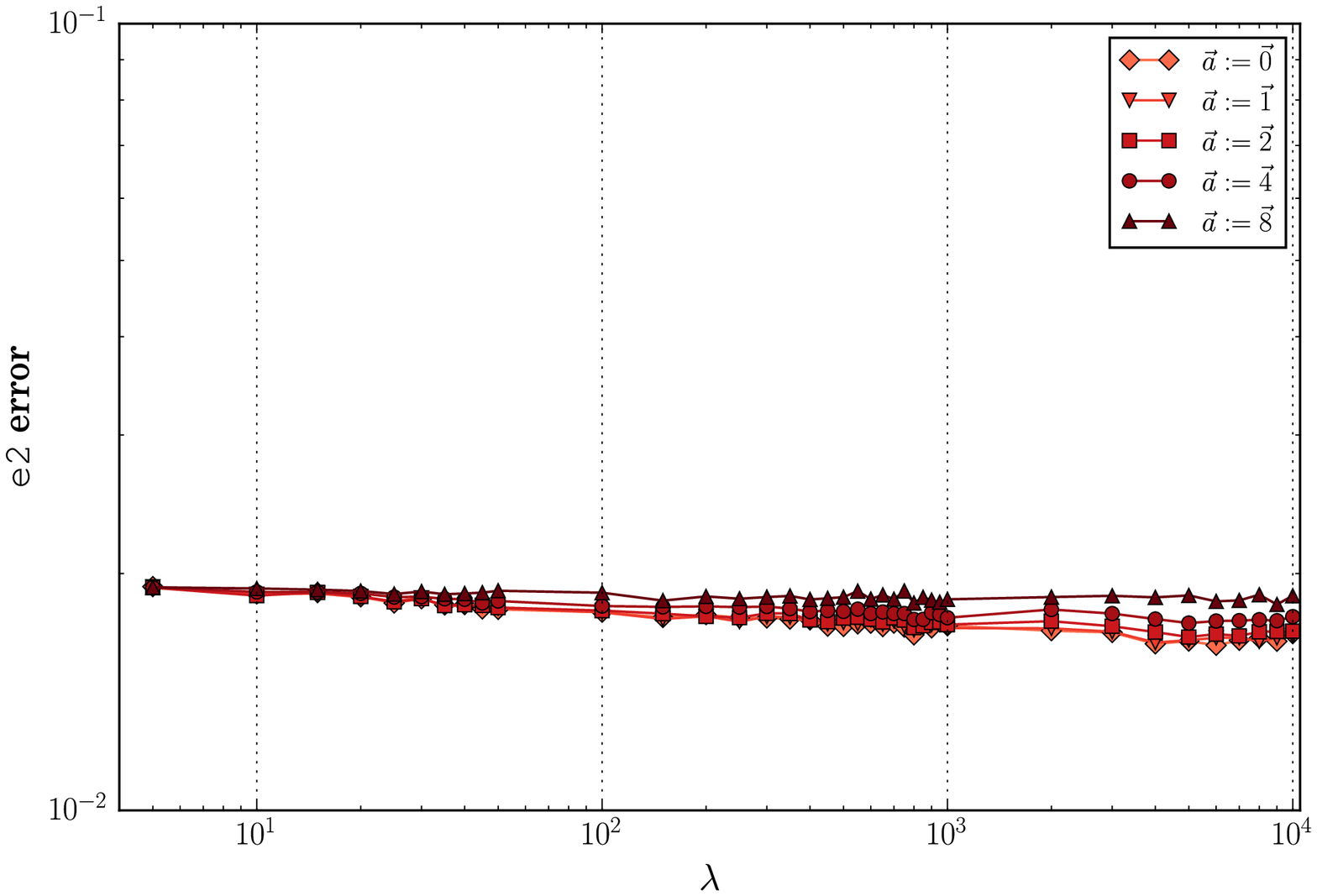, width=0.48\columnwidth} \\
\hline
\end{tabular}
\end{center}
\caption{Investigating the translation of the sampling location on (H-4) with conditioning $c=10$ over $N_{\texttt{iter}}=10^6$ iterations. 
The error measures are depicted for $n=32$ [LEFT] and $n=64$ [RIGHT] over various instantiations of a shifting vector $\vec{a}$.
All axes are logarithmically scaled except for the $y$-axes in (e-2).
\label{fig:FARTHER2}}
\end{figure*}

\subsubsection{The Inverse Relation subject to Increasing Conditioning}\label{sec:condScaling}
As an extension, we conducted a systematic evaluation over condition numbers in the exponential range $c\in [2^2,2^3,\ldots 2^{20}]$, in order to account also for ill-conditioned Hessians.

Figure~\ref{fig:ConditionUp} depicts the evaluation of the (H-4) and (H-5) landscapes using the error measures (e-1) and (e-2), exhibiting a clear trend of error increase as the conditioning grows. 
At the same time, the actual error values are lower as the population-size grows, as expected.
This effect of increasing error rates for increasing conditioning is rather intuitive, and is explained by the increasingly growing problem-complexity and the requirement for larger population-sizes to demonstrate the proved inverse-relation. 
According to Theorem~\ref{thm:generalInverseRelation}, the error rates will necessarily vanish also for extreme Hessian spectra, yet this would require in practice dramatically larger $\lambda$ values.
\begin{figure*}
\begin{center}
\begin{tabular}{c c}
{\Large (e-1)} & {\Large (e-2)} \\
\hline
(H-4) Rotated Ellipse & \\
\epsfig{file=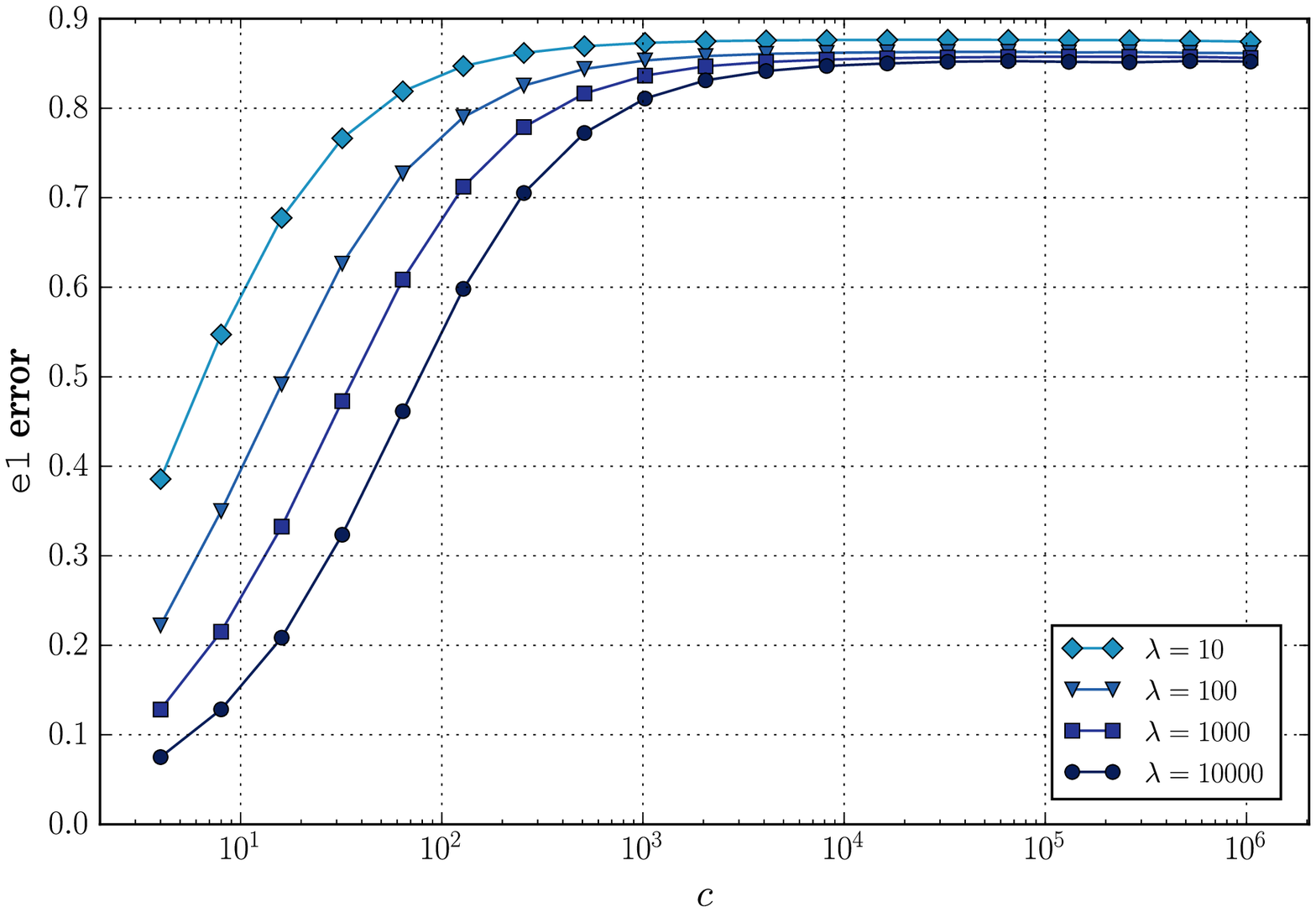, width=0.48\columnwidth} & \epsfig{file=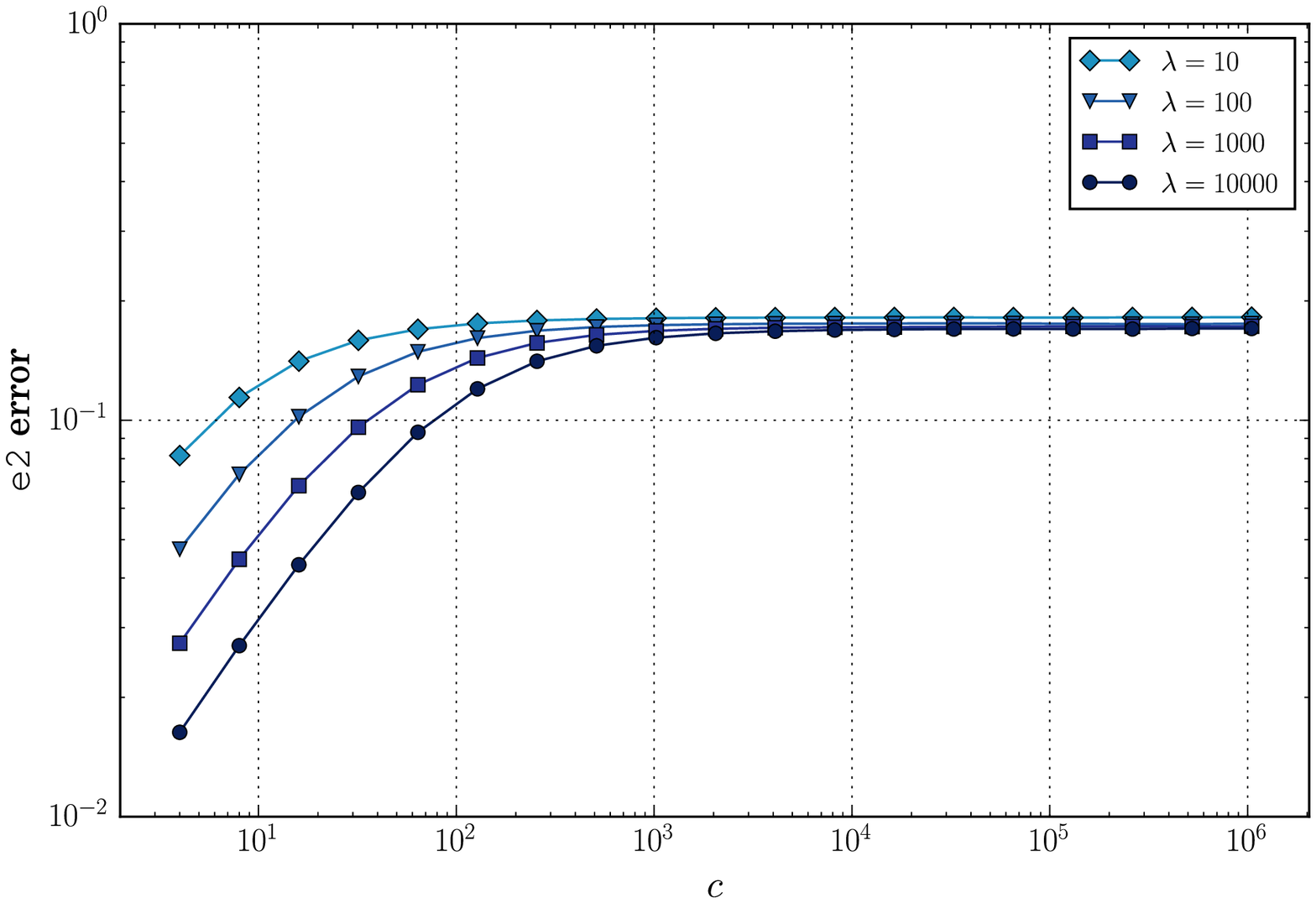, width=0.48\columnwidth} \\
\hline
(H-5) Hadamard Ellipse & \\
\epsfig{file=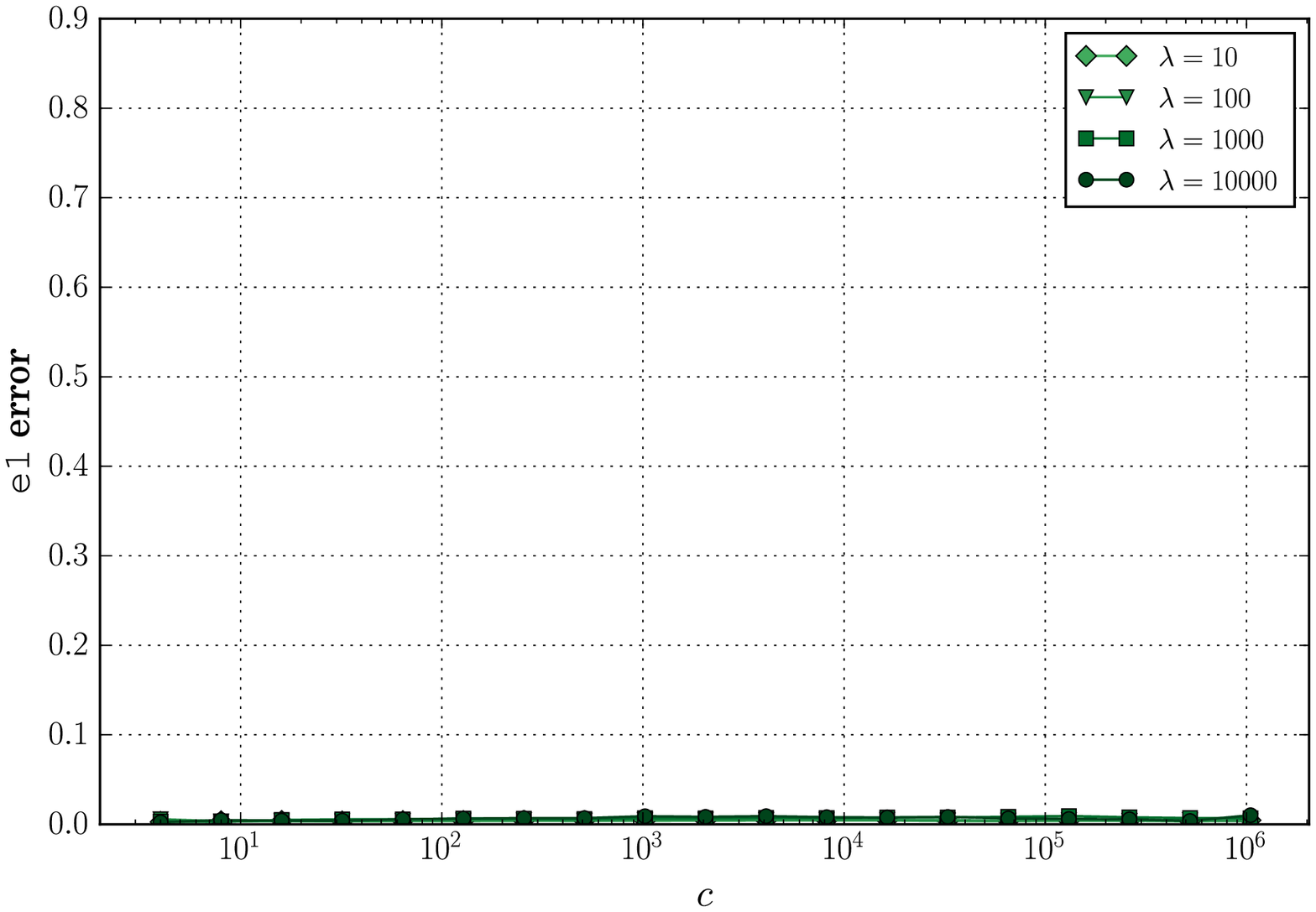, width=0.48\columnwidth} & \epsfig{file=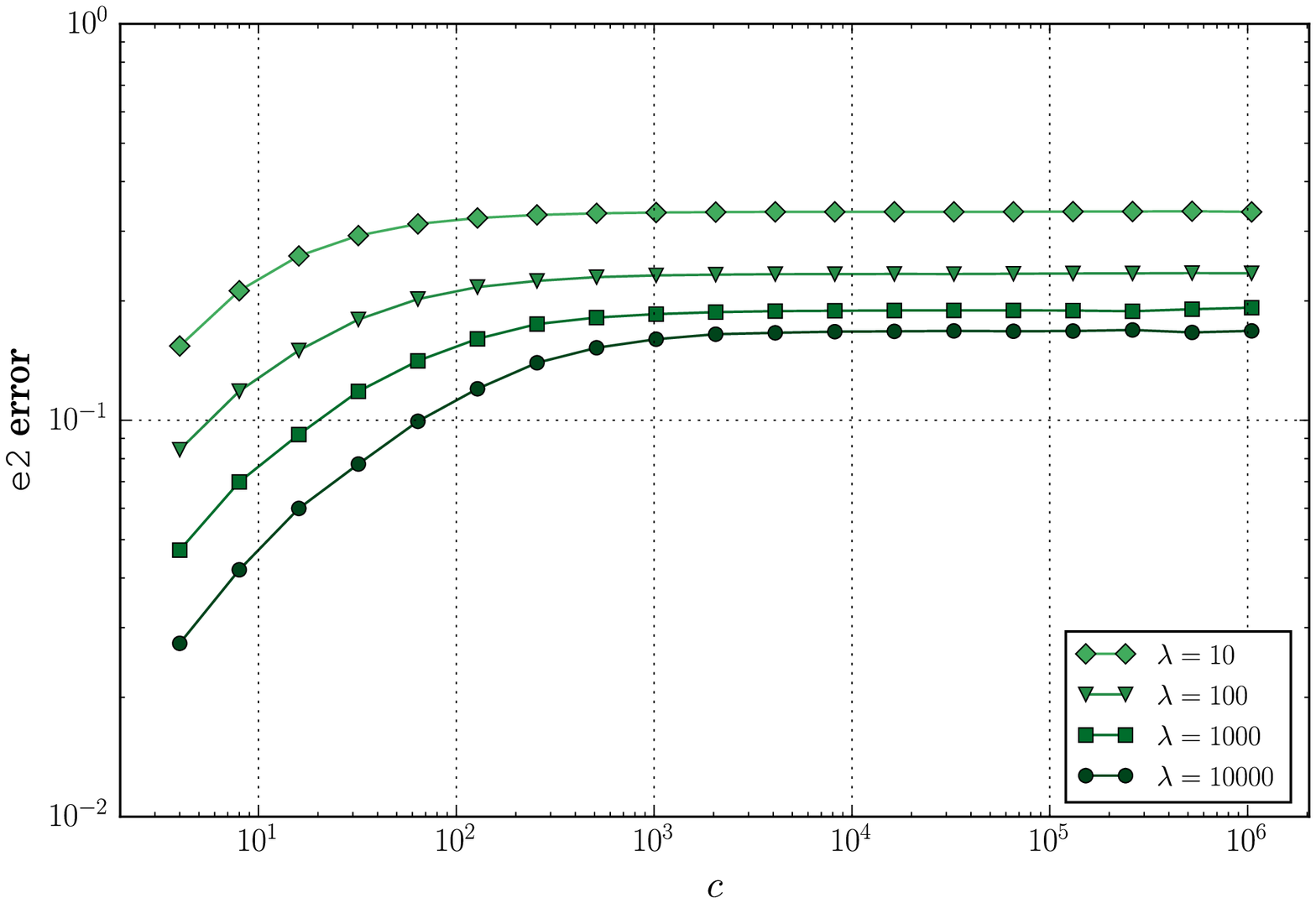, width=0.48\columnwidth} \\
\hline
\end{tabular}
\end{center}
\caption{The inverse-relation subject to increasing conditioning for various population-sizes $\lambda = 10,100,1000,10000$ on (H-4) [TOP] and on (H-5) [BOTTOM] in dimensionality $n=8$. 
The sampling was carried out over $N_{\texttt{iter}}=10^6$ iterations.
(e-1) [LEFT] and (e-2) [RIGHT] are depicted as a function of increasing $c$; all axes are logarithmically scaled, except for the $y$-axis of (e-1). (e-1) vanishes on (H-5), as was previously observed and explained.
\label{fig:ConditionUp}}
\end{figure*}

\section{Discussion}\label{sec:discussion}
This study has proven the inverse relation between the landscape Hessian and the covariance matrix when statistically constructed by ESs operating with $(1,\lambda)$-selection on the class of positive quadratic functions. 
The results presented herein generalized previous analytical work \cite{Shir-Theory-foga17} that was limited to sampling in the vicinity of the minima. As was for the near-optimum special case, this proven learning capability is rooted only at two aspects: (i) isotropic Gaussian mutations, and (ii) rank-based selection. 
This general result confirms the hypothesized capacity of ESs to extract the \textit{sensitive optimization directions} from their covariance matrices.

Notably, learning the landscape is an inherent property of classical ESs; it does not require Derandomization, nor does it require IGO as a proof tool. 
Unlike such prior proofs that were linked to actual search algorithms (e.g., the NGD \cite{Akimoto2012_NGD}), this work modeled \textit{passive evolutionary learning}, adhering to no strategy adaptation, when an empirical covariance matrix was statistically constructed out of winning decision vectors around a fixed sampling point. Yet, a mathematical rigor to a long-standing hypothesis on the behavior of general ESs was finally achieved.

The approximation sampling bounds derived for the near-optimum special case (see Section \ref{sec:previousresults}) hold also for the general case.
Firstly, guaranteeing that $\mathcal{C}^{\texttt{stat}}$ is pointwise $\eps$-close to $\mathcal C$ with confidence $1-\delta$ requires 
a polynomial number of samples in $\lambda,1/\eps,\ln(n)$ and $\ln(1/\delta)$. 
Secondly, guaranteeing that $\mathcal C$ is pointwise $\eps$-close to $\alpha {\mathcal H}^{-1}$ with~$\alpha \left(\lambda,\mathcal{H} \right) >0$, dictates an upper bound on the number of samples depending upon $\eps,\lambda$ and on the Hessian's spectrum.

Following our analytical work, which concluded with Theorem~\ref{thm:generalInverseRelation}, we carried out an extensive simulation study to numerically corroborate this result at multiple levels. Most importantly, we demonstrated the tendency of the normalized $\mathcal{H}_{0} \mathcal{C}^{\texttt{stat}}$ to become the identity when $\lambda$ increases, exactly as Theorem~\ref{thm:generalInverseRelation} predicts, over multiple landscapes, condition numbers, and search-locations.\\

Next, we offer a future direction of work and hypothesize a generalization of this work to $(\mu,\lambda)$-selection.
\subsection*{Future Work: The Inverse Relation in $(\mu,\lambda)$-Selection}
The exact probability and density functions for the $\ell^{th}$-degree winning value, $\omega_{\ell:\lambda}$, read (see, e.g., \cite{OrderStatistic}):
\begin{equation}\label{eq:os_distributions}
\begin{array}{l}
\medskip
\texttt{CDF}_{\omega_{\ell:\lambda}}\left(\psi\right) = \sum_{k=\ell}^{\lambda} {\lambda \choose k} \texttt{CDF}_{\psi}\left(\psi\right)^{k} \left( 1-\texttt{CDF}_{\psi}\left(\psi\right)\right)^{\lambda-k}\\
\texttt{PDF}_{\omega_{\ell:\lambda}}\left(\psi\right)\\
 = \lambda \cdot \texttt{PDF}_{\psi}\left(\psi\right) {\lambda-1 \choose \ell-1} \texttt{CDF}_{\psi}\left(\psi\right)^{\ell-1} \left( 1-\texttt{CDF}_{\psi}\left(\psi\right)\right)^{\lambda-\ell}~.
\end{array}
\end{equation}

Let $\mathcal{C}^{(\ell:\lambda)}$ denote the covariance matrix constructed out of $\ell^{th}$-degree winners only.
We hypothesize that this covariance matrix is also close to being proportional to the inverse Hessian, under a large population and subject to strong selection pressure, i.e., $\lambda\rightarrow \infty,~\ell \ll \lambda$.

In Figure \ref{fig:Ldegree} we present empirical evidence for this hypothesis on the (H-4) test-case considering either $\ell=2$ or $\ell=5$.
It is apparent that the equivalent phenomenon occurs for these covariance matrices of $\ell^{th}$-degree winners. 
\begin{figure*}
\begin{center}
\begin{tabular}{c c}
 {\Large (e-1)} & {\Large (e-2)} \\
\hline
$\ell=2$ & \\
\epsfig{file=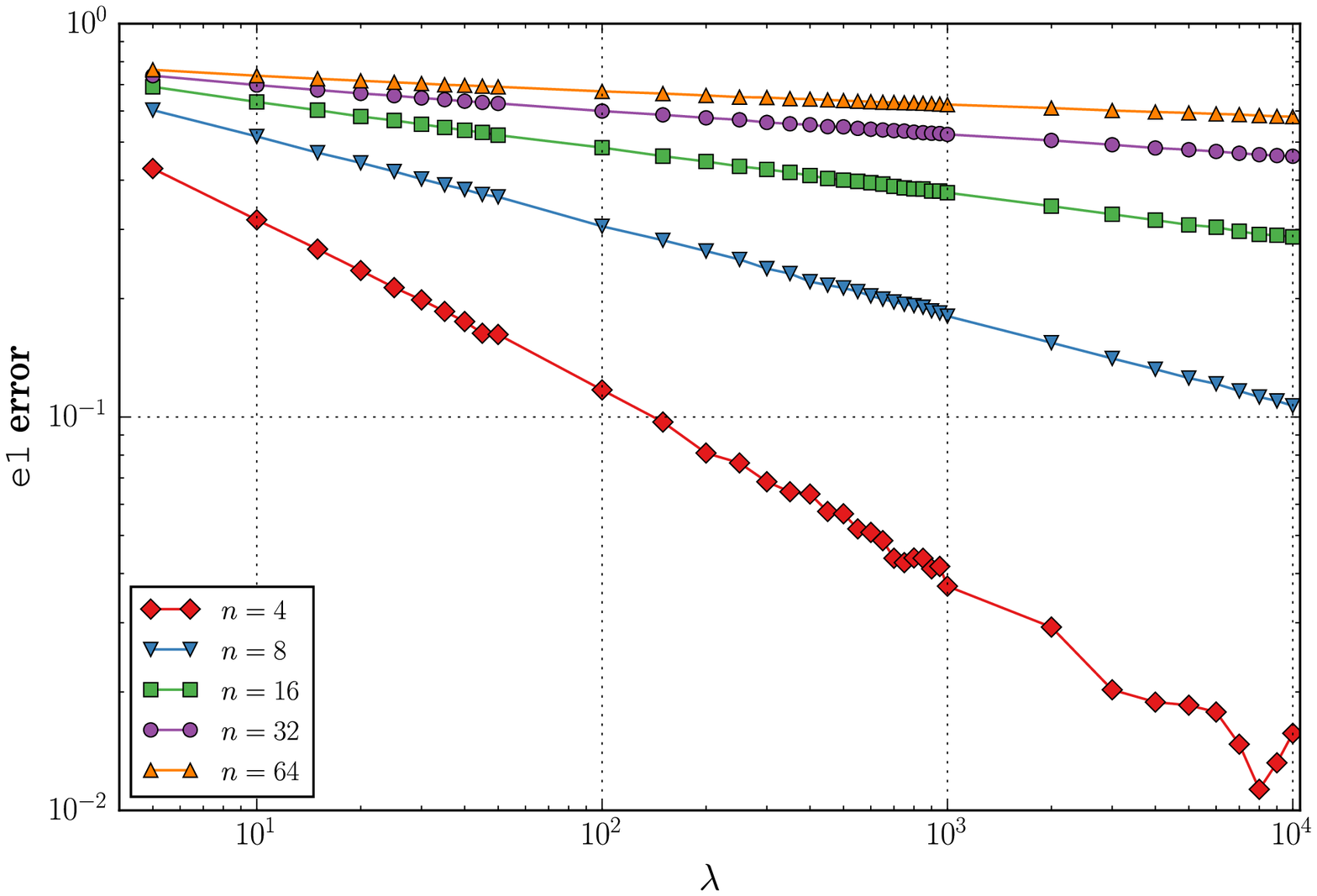, width=0.48\columnwidth} & \epsfig{file=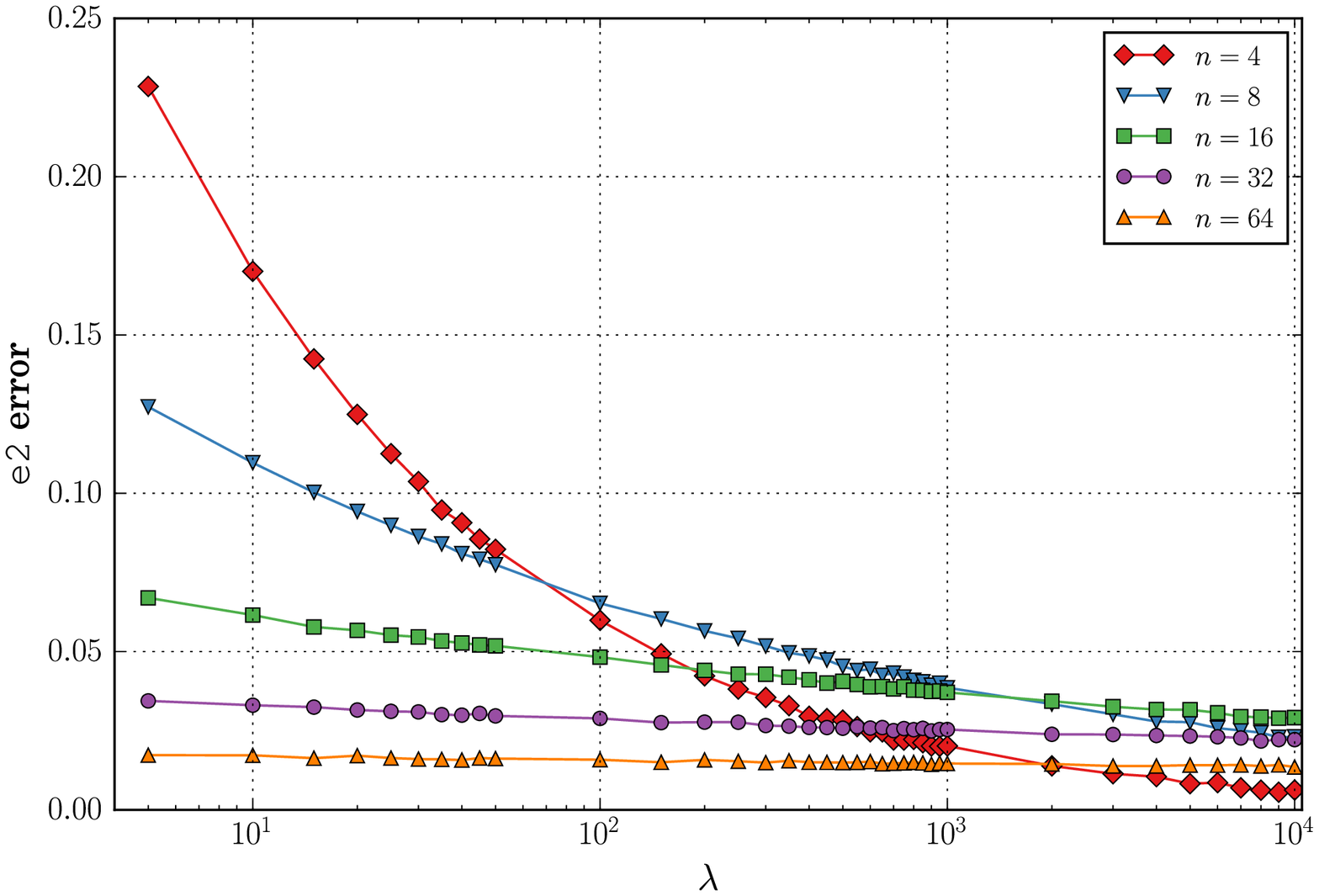, width=0.48\columnwidth} \\
\hline
$\ell=5$ & \\
\epsfig{file=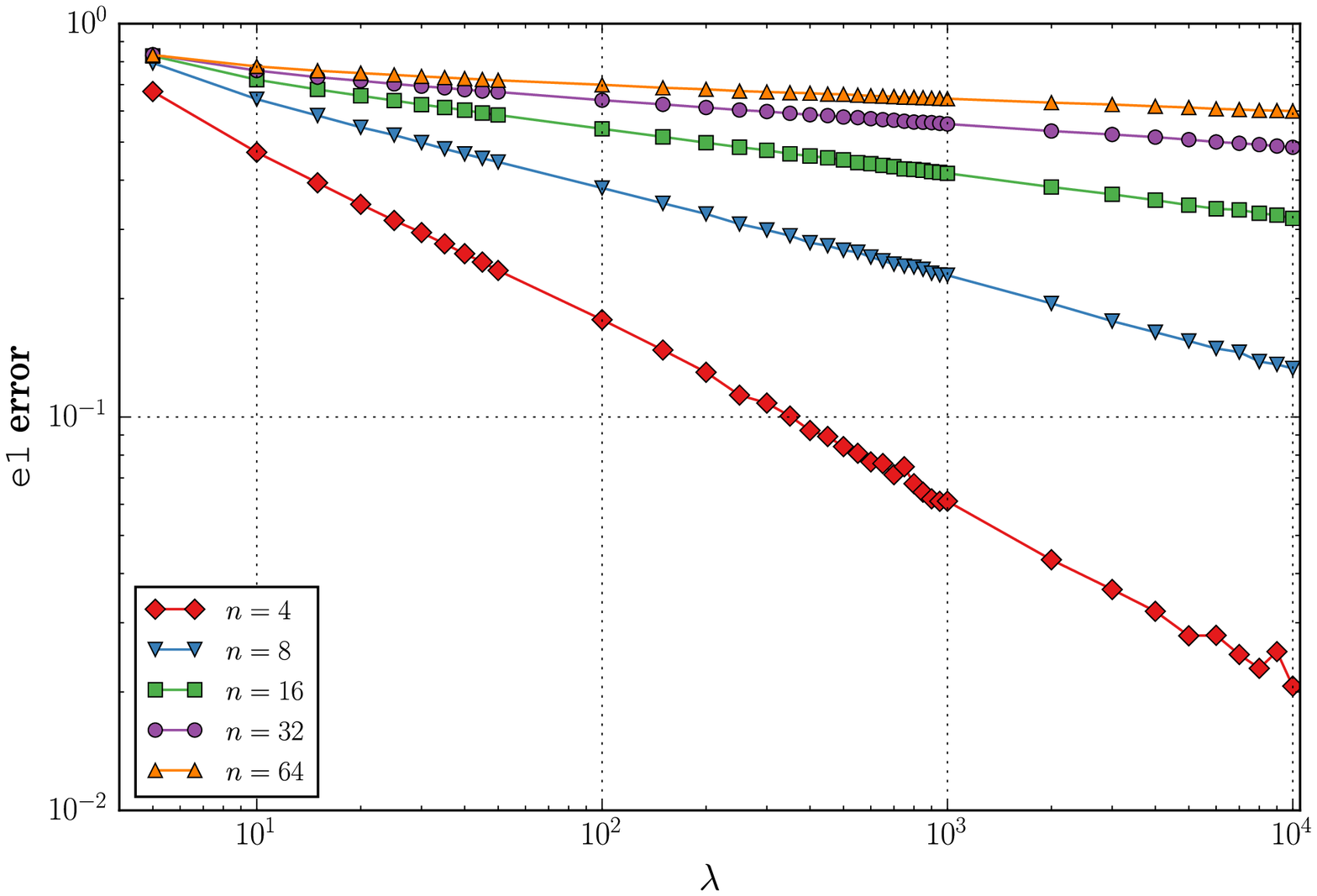, width=0.48\columnwidth} & \epsfig{file=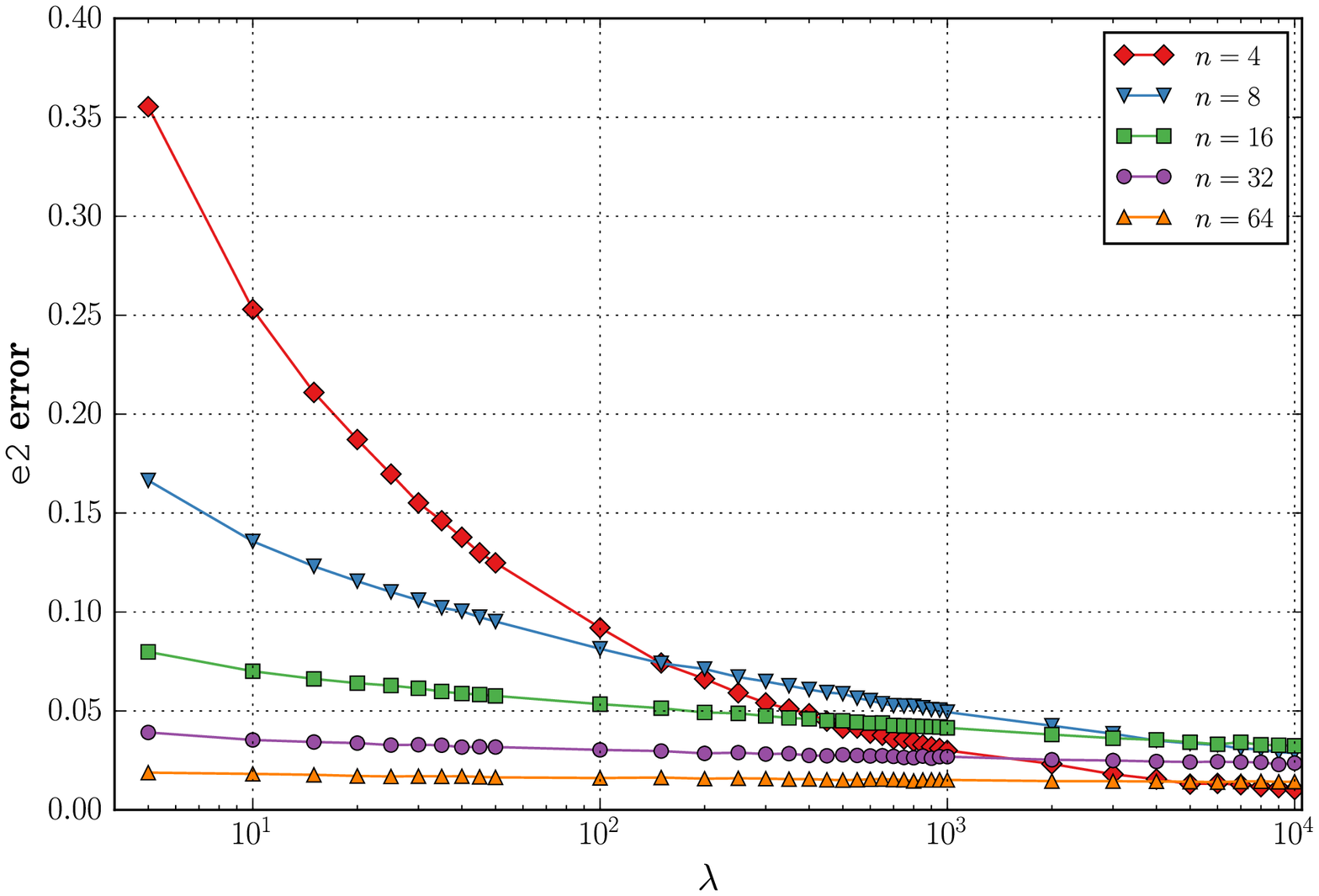, width=0.48\columnwidth} \\
\hline
\end{tabular}
\end{center}
\caption{Empirical evidence for the hypothesis concerning $\ell^{th}$-degree winners by systematic evaluation of the error measures (e-1)-(e-2) as a function of increasing $\lambda$ on (H-4) with conditioning $c=10$ for various dimensions $n = 4,8,16,32,64$. The sampling was carried out over $N_{\texttt{iter}}=10^6$ iterations.
(e-1) [LEFT] and (e-2) [RIGHT] are depicted for $\ell=2$ [TOP] and $\ell=5$ [BOTTOM]; all $x$-axes are logarithmically scaled.
\label{fig:Ldegree}}
\end{figure*}

\section{Acknowledgements}
The authors are indebted to Jonathan Roslund, for his significant contributions that ignited this line of work.
Ofer Shir would like to thank the participants of the Dagstuhl Seminar on ``Theory of Randomized Optimization Heuristics'' (17191) for insightful discussions on theory of evolution strategies and notions of derandomization.

\small


\appendix

\section{The Density Expression}
\label{sec:DE}

\newcommand{\tz}{\tilde z}
\newcommand{\ty}{\tilde y}
\newcommand{\tw}{\tilde \omega}
\newcommand{\tpsi}{\tilde \psi}

Here we explain~\eqref{eq:x_pdf_2}.
Start with random variables taking values in a finite set.
Assume $\tz$ is a random variable
taking values in a finite subset of $\R^n$.
We can sample
$(\tz,\tpsi)$ where $\tpsi = J(\tz)$
by first sampling $\tpsi$ from the appropriate distribution and then sampling
$\tz$ conditioned on the value of $\tpsi$.
Now, sample $\tpsi_1,\ldots,\tpsi_\lambda$ i.i.d.\
as $\tpsi$.
Let $\tw$ be the minimum of all these values.
Let $\ty$ be chosen from the distribution of
$\tz$ conditioned on the event $J(\tz) = \tw$.
The pair $(\ty,\tw)$ is distributed as the winner
among the $\lambda$ samples and its $J$-value.
We see that for each $x$,
\begin{align*}
\Pr[\ty = x]
& = \Pr[\tw = J(x)] \Pr[ \ty = x | J(\ty) = J(x)] \\
& = \Pr[\tw = J(x)] \Pr[ \tz = x | J(\tz) = J(x)] \\
& = \Pr[\tw = J(x)] \frac{\Pr[ \tz = x]}{\Pr[\tpsi = J(x)]} .
\end{align*}
Stated differently,
\begin{align*}
\frac{\Pr[\ty = x]}{\Pr[ \tz = x]}
& =  \frac{\Pr[\tw = J(x)]}{\Pr[\tpsi = J(x)]} .
\end{align*}
Now, choose $\tz$ from denser and denser subsets of $\R^n$
in a way that approaches $\mathcal{N} (\vec{x}_0,\mathbf{I})$.
The distribution of $\ty$ approaches that of $y$.
Since $J$ is smooth,
the distribution of $\tpsi$ approaches that of $\psi$.
The distribution of $\tw$ approaches that of $\omega$.
We see that the l.h.s.\ converges
to $\frac{\texttt{PDF}_y(x)}{\texttt{PDF}_z(x)}$;
for a small ball $B$ around $x$ we have that
$\tfrac{\Pr[\ty = x]}{\Pr[ \tz = x]}$
is approximately 
$\tfrac{\texttt{PDF}_y(x) \cdot |B| }{\texttt{PDF}_z(x) \cdot |B|}$.
Similarly, the r.h.s.\ approaches
$\frac{\texttt{PDF}_\omega(J(x))}{\texttt{PDF}_\psi(J(x))}$.

\end{document}